\documentclass[12pt]{article}
\usepackage[vmargin=1.1in,hmargin=1.2in]{geometry}
\usepackage{setspace}
\usepackage{parskip}
\usepackage{amsmath}
\allowdisplaybreaks  
\usepackage{amssymb}
\usepackage{mathtools}

\usepackage{algorithm}
\usepackage{algpseudocode}

\usepackage{textcomp, gensymb}
\usepackage{gensymb}
\usepackage{enumitem}
\usepackage{graphicx}
\usepackage{scrextend}
\usepackage{blindtext}
\usepackage{caption}
\usepackage{url}
\usepackage[numbers]{natbib}
\usepackage{subcaption}
\usepackage{circuitikz}
\usepackage{listings}
\usepackage{color}
 \usepackage{microtype}
\usepackage{graphicx}
\usepackage{multirow}
\usepackage{booktabs} 

\usepackage{amssymb}
\usepackage{pifont}
\usepackage{hyperref}

\newcommand{\specialcell}[2][c]{%
\begin{tabular}[#1]{@{}c@{}}#2\end{tabular}}
\usepackage{amsmath}
\DeclareMathOperator*{\argmax}{arg\,max}
\usepackage{authblk}
\usepackage{amsmath}
\usepackage{amssymb}
\usepackage{mathtools}
\DeclarePairedDelimiter{\ceil}{\lceil}{\rceil}
\usepackage{amsthm}
\usepackage[capitalize,noabbrev]{cleveref}
\usepackage{comment}
\usepackage{dsfont}
\theoremstyle{plain}
\newtheorem{theorem}{Theorem}
\newtheorem{proposition}{Proposition}
\newtheorem{lemma}{Lemma}
\newtheorem{corollary}{Corollary}
\theoremstyle{definition}
\newtheorem{definition}{Definition}
\newtheorem{assumption}{Assumption}
\theoremstyle{remark}
\newtheorem{remark}{Remark}
\definecolor{codegreen}{rgb}{0,0.6,0}
\definecolor{codegray}{rgb}{0.5,0.5,0.5}
\definecolor{codepurple}{rgb}{0.58,0,0.82}
\definecolor{backcolour}{rgb}{0.95,0.95,0.92}
 
\lstdefinestyle{mystyle}{
    backgroundcolor=\color{backcolour},   
    commentstyle=\color{codegreen},
    keywordstyle=\color{magenta},
    numberstyle=\tiny\color{codegray},
    stringstyle=\color{codepurple},
    basicstyle=\footnotesize,
    breakatwhitespace=false,         
    breaklines=true,                 
    captionpos=b,                    
    keepspaces=true,                 
    numbers=left,                    
    numbersep=5pt,                  
    showspaces=false,                
    showstringspaces=false,
    showtabs=false,                  
    tabsize=2
}
 
\date{\today}
\usepackage{bbm}

\renewcommand{\P}{\mathbb{P}}
\newcommand{\N}{\mathbb{N}}
\usepackage{color}

\usepackage{cleveref}
\crefname{assumption}{assumption}{assumptions}
\Crefname{assumption}{Assumption}{Assumptions}

\crefname{theorem}{theorem}{theorems}
\Crefname{theorem}{Theorem}{Theorems}

\crefname{figure}{fig.}{}
\Crefname{figure}{Fig.}{}

\title{\vspace{-0.6cm}
\bfseries Post-detection inference for\\ sequential changepoint localization}

\author[1]{Aytijhya Saha}
\author[2]{Aaditya Ramdas}

\affil[1]{Massachusetts Institute of Technology.  \texttt{aytijhya@mit.edu}}
\affil[2]{Carnegie Mellon University. 
\texttt{aramdas@cmu.edu}}
\begin{document}

\date{}
\maketitle

\begin{abstract}
  This paper addresses a fundamental but largely unexplored challenge in sequential changepoint analysis: conducting inference following a detected change. 
   We develop a very general framework to construct confidence sets for the unknown changepoint using only the data observed up to a data-dependent stopping time at which an arbitrary sequential detection algorithm declares a change. 
   Our framework is nonparametric, making no assumption on the composite post-change class, the observation space, or the sequential detection procedure used, and is nonasymptotically valid. We also extend it to handle composite pre-change classes under a suitable assumption, and also derive confidence sets for the change magnitude in parametric settings. We provide theoretical guarantees on the width of the confidence intervals. Extensive simulations demonstrate that the produced sets have reasonable size, and slightly conservative coverage. In summary, we present the first general method for sequential changepoint localization, which is theoretically sound and broadly applicable in practice.
\end{abstract}

\section{Introduction}

Consider the following general problem of sequential change analysis in an arbitrary space $\mathcal{X}$ (which could be the Euclidean space $\mathbb R^d$, or images, or anything else). A sequence of $\mathcal{X}$-valued observations 
 ${\bf X} = X_1,X_2,\cdots$ arrives sequentially, such that for some unknown time $T\in\N\cup\{\infty\}$ (the changepoint),
\begin{equation}\label{eq:data-setup}
X_1,X_2,\cdots,X_{T-1}\stackrel{}{\sim}F_0 ~\text{ and }~X_{T},X_{T+1},\cdots\stackrel{}{\sim}F_1,
\end{equation}
for some $F_0$ in a pre-change class $\mathcal{P}_0$ and $F_1$ in a  post-change class $\mathcal{P}_1$, where $\mathcal{P}_0$ and $\mathcal{P}_1$ denote some non-intersecting classes of probability distributions. Initially, we shall assume all observations to be independent, although our methodology later extends to handling dependent data as well.

Suppose the user has already decided to use a sequential changepoint detection algorithm $\mathcal A$ for this task, which is a mapping from the data sequence $\bf X$ to a stopping time\footnote{While sequential change detection methods are commonly described solely by stopping times, we describe them by algorithms $\mathcal A$ because our method will feed in simulated (synthetically generated) data sequences $\bf X'$ into $\mathcal A$; we find it easier to describe our methodology when making $\mathcal A$ explicit.} $\tau = \mathcal A({\bf X})$ at which it raises an alarm that a change may have previously occurred. When $\mathcal A$ raises an alarm at time $\tau$, we do not know if it is a false alarm, and if it is not, then we only know that there has been some possible change before $\tau$, but not when that change occurred. 

Our aim here is to construct confidence sets for the unknown changepoint $T$ after stopping at $\tau$, assuming only black-box access to $\mathcal A$ (and ideally no further assumptions or restrictions). By ``black-box'', we mean that we can run $\mathcal A$ on any input data sequences we like, but we do not require any further insight into its inner workings. The fact that we do not require knowledge of internal details of $\mathcal A$ is a particularly important feature; it implies that our changepoint localization method can wrap around any existing workflow for changepoint \emph{detection} --- it could be employed both with simple classical statistical algorithms (like CUSUM), and also with complex modern pipelines in the information technology industry that may involve heuristic machine learning algorithms (to deal with complex data types). 

{ 
For instance, rapid shifts in public sentiment often occur on online platforms: a product failure may trigger a sudden surge of negative reviews, or a major event may abruptly alter the tone of social-media discussions. Detecting when such a shift occurs is vital for understanding emerging issues, tracing public reactions to events, or monitoring product reputation in real time. To illustrate the relevance of our methodology in such settings, we conduct an experiment using the Stanford Sentiment Treebank (SST-2) dataset \citep{socher2013recursive}, simulating a stream of text reviews whose sentiment abruptly transitions from predominantly positive to predominantly negative. 
Accurately localizing the changepoint is important: it can help investigate why the change may have occurred. 
}

As an initial goal, one may hope to obtain a set $\mathcal C \subseteq \{1,\dots,\tau\}$ such that
\begin{equation}\label{eq:marginal-coverage}
\P_{F_0,T,F_1} (T \in \mathcal C) \geq 1-\alpha,
\end{equation}
where $\P_{F_0,T,F_1}$ is the data distribution in~\eqref{eq:data-setup}, i.e.\ having the true changepoint at $T$. But this goal is harder to achieve than may initally appear; we elaborate below.

There are two classes of algorithms $\mathcal A$ in the literature: 
(i) algorithms with probability of false alarm (PFA) bounded by $\delta < 1$; we call these ``bounded PFA'' for short, and
(ii) algorithms with PFA equal to one, but have a bounded average run length to false alarm (ARL2FA, or just ARL for short); we call these ``bounded ARL''.
Since algorithms with bounded ARL have PFA equal to one, false detections occur with a nontrivial frequency. But the flipside is that bounded ARL algorithms typically have smaller detection delay than bounded PFA algorithms. Typically, detection delays of bounded ARL algorithms can be bounded by quantities independent of the changepoint $T$, but this guarantee is impossible to achieve with bounded PFA algorithms. In this paper, we remain agnostic to the choice of $\mathcal{A}$.

 For algorithms with a bounded ARL, we prove that no set $\mathcal C$ can satisfy  guarantee~\eqref{eq:marginal-coverage} (for algorithms with bounded PFA, we show how to achieve~\eqref{eq:marginal-coverage} in the Supplementary Section F). This impossibility result (which we prove formally later) intuitively holds because it is meaningless to ask for coverage if the algorithm stopped before the changepoint; indeed, on the event $\{\tau < T\}$, we have no information to rule out the possibility that $T=\infty$, and thus the only way to cover $T$ is to output the set of all natural numbers $\mathbb N$.
Thus, an arguably more sensible goal is to have 
\begin{equation}
\label{eq:cond-coverage-intro}
\P_{F_0,T,F_1} (T \in \mathcal C \mid \tau \geq T) \geq 1-\alpha,
\end{equation}
meaning that $\mathcal C$ only needs to cover if algorithm $\mathcal A$ only stops after the true changepoint $T$. 
This definition has been considered by, for example,  \cite{ding2003lower} and \cite{WU20063625,wu2007inference}, who achieved it (in a certain asymptotic sense) in much more restricted settings than the current investigation. 

The aforesaid impossibility result no longer applies to the new goal~\eqref{eq:cond-coverage-intro}.
It is important to note that we do not prescribe which $\mathcal A$ to use in which situation, but simply enable post-hoc inference on $T$ for any user-chosen $\mathcal A$. One can view our techniques as providing a wrapper to convert $\mathcal A$ from producing only a stopping time $\tau$ to also producing a confidence set for $T$ 
As we survey below, there is a rich literature on designing change \emph{detection} algorithms $\mathcal A$, but a substantial paucity of inference methodology for \emph{localizing} the change after detection; our focus is on the latter goal, assuming the user has chosen a suitable algorithm $\mathcal A$ for detection.

While inference for $T$ is intrinsically interesting and important, a side benefit of confidence sets for $T$ is that one can use these to derive post-hoc confidence sets for unknown parameters of $F_0$ and $F_1$. To clarify, suppose $\mathcal P_i = \{F_\theta\}_{\theta \in \Theta_i}; i=0,1$. Our method that delivers~\eqref{eq:cond-coverage-intro} also delivers confidence sets $\mathcal C' \subseteq \Theta_0 \times \Theta_1$ such that
$\mathbb P_{\theta_0,T,\theta_1} ((\theta_0,\theta_1) \in \mathcal C' \mid \tau \geq T) \geq 1-\alpha, \forall \theta_0 \in \Theta_0, \theta_1 \in \Theta_1,$
where $\mathbb P_{\theta_0,T,\theta_1}$ is the distribution of the sequence $X_1,\dots,X_{T-1}\stackrel{}{\sim} F_{\theta_0}$, $X_T,X_{T+1},\dots\stackrel{}{\sim} F_{\theta_1}$. 
We stress that our confidence sets do not rely on asymptotics and are valid in finite samples.

\subsection{Related work}

\textbf{Prior work on post-detection confidence sets for $T$.} 
The problem of constructing confidence sets for the changepoint after stopping has been largely overlooked in the literature. This gap is significant because in most real-world applications, data arrives sequentially, and it can be essential to make timely inferences about the changepoint immediately upon detecting a change. 

We are aware of only two directly related works: \cite{ding2003lower} and \cite[Chapter 3]{wu2007inference}. They also chose to aim for the guarantee~\eqref{eq:cond-coverage-intro} like us. 
 Both  \cite{ding2003lower,wu2007inference} provide asymptotically valid inference after the CUSUM procedure \citep{page1954continuous} declares a change for known $F_0$ and $F_1$ in an exponential family; the former only constructs a lower confidence bound for $T$ while the latter also provides a confidence set for $T$.
Their approaches have notable methodological and theoretical limitations: (i) they are tailored specifically to the CUSUM procedure, (ii) they assume that the pre- and post-change distributions are both known and they belong to an exponential family,  (iii) they only provide asymptotic coverage guarantees (valid when both $T\to\infty$ and ARL $\to\infty$). These constraints limit practical applicability, and their confidence set exhibits undercoverage in our experiments.

In sharp contrast, our proposed confidence sets for $T$ are very general: We make no assumptions on the detection procedure $\mathcal A$, and our framework acts as a versatile wrapper that can be layered on top of any sequential change detection algorithm. It is not restricted to exponential families --- it works for any known $F_0$ and any nonparametric post-change class $\mathcal{P}_1$ on any general space (e.g., $\mathbb R^d$) and offers non-asymptotic coverage guarantees. It also extends to composite pre-change classes under some assumptions. Thus, our framework is the \textit{first nonparametric, theoretically sound, and widely applicable practical tool} for sequential change localization tasks. We summarize the comparison in \Cref{tab:method-comparison}.

\begin{table}[ht]
\centering
\caption{Comparison of existing methods and our proposed approach}
\label{tab:method-comparison}
\resizebox{\linewidth}{!}{
\begin{tabular}{@{}lcc|cc@{}}
\toprule
{Pre-, Post-} 
& \multicolumn{2}{c|}{Existing Methods \cite{ding2003lower,wu2007inference}} 
& \multicolumn{2}{c}{Our Proposed Method} \\
\cmidrule(lr){2-3} \cmidrule(l){4-5}
change distributions & Assumptions & Guarantee & Assumptions & Guarantee \\
\midrule

{Known, Known} 
& $\mathcal A=$ CUSUM. $F_0,F_1$ & Asymptotic 
& No assumption 
& Nonasymptotic   \\
($F_0,F_1$)& $\in$ exponential family & ($T$, ARL $\to\infty$)  & (any $\mathcal A, F_0,F_1$) & (\Cref{thm:exact-cond-coverage}) \\

\midrule
{Known, Unknown}
& \multicolumn{2}{c|}{No existing work} 
& No assumption 
& Nonasymptotic \\
($F_0,\mathcal{P}_1$)& & & (any $\mathcal A, F_0,\mathcal{P}_1$) & (\Cref{thm:coverage-nonpara-known-pre})\\

\midrule

{Unknown, Unknown}
& \multicolumn{2}{c|}{No existing work} 
& \Cref{assmp-1}
& Nonasymptotic \\
($\mathcal{P}_0,\mathcal{P}_1$)& & & on $\mathcal A,\mathcal{P}_0$ (any $\mathcal{P}_1$) & (\Cref{thm:coverage-nonpara-comp-pre})\\
\bottomrule
\end{tabular}}
\end{table}

\textbf{Literature on sequential change \emph{detection}.}
In sequential change detection, the goal is to \emph{detect that a change has occurred} as soon as possible, but not to provide inferential tools for \emph{when the change occurred}. This methodology has proved crucial in many applications, such as quality control, cybersecurity, medical diagnostics, and so on. 
Classical methods such as Shewhart control charts \citep{shewhart1925application}, the CUSUM procedure \citep{page1954continuous}, and the Shiryaev-Roberts approach \citep{shiryaev1963optimum,roberts1966comparison} assume known pre- and post-change distributions. To address scenarios with unknown post-change distributions, generalized likelihood ratio (GLR) procedures \citep{lorden1971procedures,siegmund1995using} were developed. Comprehensive reviews of these methods are provided by \cite{tartakovsky2014sequential}. More recently, kernel-based techniques \citep{harchaoui2008kernel}, e-detectors \citep{shin2022detectors}, etc., have emerged to handle nonparametric settings. All these methods focus on the quickest change detection, with average run length (ARL) control. In contrast, another strand of the literature emphasizes controlling the false alarm rate instead of ARL, often at the cost of a longer detection delay \citep{chu1996monitoring,kirch2022sequential,aue2024state}. These approaches typically rely on an initial historical period of data to calibrate the detection thresholds.

\textbf{Offline vs. online change \emph{localization}.}
As far as we are aware, most existing approaches to changepoint localization focus on offline settings, where a fixed amount of data is available at once, and they do not generalize to analysis at data-dependent stopping times.
Classical works, e.g.,  \cite{hinkley1970inferencce} study the asymptotic behavior of the MLE for the changepoint and construct likelihood-ratio-based confidence sets. \cite{darkhovskh1976a} proposed some nonparametric estimators. \cite{worsley1986confidence} developed likelihood ratio-based confidence sets for exponential family distributions, while \cite{dumbgen1991asymptotic} used bootstrap-based tests for each candidate time to construct nonparametric confidence sets with asymptotic coverage.
Recently, \cite{verzelen2023optimal} studies optimal rates for both changepoint detection and localization.
\cite{jang2024fast} constructs confidence sets by evaluating a local test statistic. To re-emphasize, all these methods are for the offline setting. 
In the sequential setting, \cite{Srivastava01011999,Gombay10012003,wu2004bias,Brodsky30042010} develop point estimators for $T$; but very few works construct confidence sets \citep{ding2003lower,wu2007inference}, whose drawbacks were discussed earlier.

\textbf{Inference on $\theta_0,\theta_1$ after stopping.} 
As with the problem of post-detection inference on the changepoint $T$, there is limited literature on valid estimation or inference for the pre- and post-change parameters $\theta_0$ and $\theta_1$ post-stopping.
 \cite{WU20063625} constructs asymptotic bias-corrected point estimates and confidence intervals for the post-change mean after a mean change is detected by a CUSUM procedure. \cite{wu2005inference} generalizes these results for the case when there is a possible change in the variance as well. These methods are also discussed in  \cite{wu2007inference}. However, these approaches share several limitations with the existing post-detection changepoint inference methods, as mentioned earlier: they rely on asymptotics, assume specific parametric forms, and apply only to the CUSUM procedure.

\subsection{Contributions and paper outline}

 \begin{itemize}
     \item  We propose a general framework for constructing a confidence set for the changepoint based on inverting tests obtained using a universal threshold, assuming all observations to be independent. We show that it provides $1-\alpha$ conditional coverage for known pre-change and arbitrary (nonparametric) post-change classes under \emph{no additional assumption} in \Cref{thm:coverage-nonpara-known-pre}. We also extend it for composite pre-change under a suitable assumption in \Cref{thm:coverage-nonpara-comp-pre}.

     \item  We propose an alternative method of constructing the tests using adaptive thresholds (which is often tighter than the universal threshold), assuming  pre- and post-change distributions to be known, in \Cref{algo:1-exact}. We extend this framework to handle parametric pre- and/or post-change classes in \Cref{algo:comp-post,algo:comp}. These adaptive methods allow the observations to be dependent. We establish corresponding coverage guarantees in \Cref{thm:exact-cond-coverage,thm:coverage-comp-post,thm:coverage-comp}. 
     \item As a byproduct of our methods, we derive confidence sets for the unknown pre- and post-change parameters and show their coverage guarantees in \Cref{thm:coverage-theta0-theta1}.
 \end{itemize}


In short, we present the first nonparametric, theoretically grounded, and widely applicable post-detection confidence sets that are agnostic to the detection method.

The rest of the paper is organized as follows. 
In \Cref{sec:fundamental}, we establish some fundamental results showing the impossibility of unconditional coverage and linking conditional confidence sets to hypothesis tests. \Cref{sec:general} develops a general framework for constructing a confidence set for $T$ based on inverting tests obtained using a universal threshold.
In \cref{sec:parametric}, we introduce a simulation-based approach for constructing confidence sets using an adaptive threshold in the parametric setting and study their conditional coverage probabilities. In \Cref{sec:ci-pre-post-para}, we derive confidence sets for the changepoint, as well as the pre- and post-change parameters. \Cref{sec:expt} presents an extensive simulation studies that validate our theoretical findings and demonstrate the practical efficacy of our approach \footnote{Codes are available at \href{https://github.com/Aytijhya/Seq-changepoint-localization}{https://github.com/Aytijhya/Seq-changepoint-localization}}. This
article is concluded in \cref{sec:conc}. All the mathematical proofs, implementation details of some algorithms in the paper, and additional experiments are presented in the Supplementary Material.

\section{Some Fundamental Results}
\label{sec:fundamental}
In this section, we highlight some key observations that shape our methodology of post-detection inference for the change-point $T$. 
We begin by clarifying some notational conventions that will be used throughout the paper.

\textbf{Notation.}
For $t \in \N$, let $\P_{F_0,t,F_1}$ denote joint distribution of the sequence $\{X_n\}_{n}$ when  $X_1,\dots,X_{t-1}\stackrel{}{\sim}F_0$ and $X_t,X_{t+1},\dots\stackrel{}{\sim}F_1$. \footnote{
To understand the key ideas in the paper, it will suffice for the reader to focus on the case when all $X_i$'s are all independent, and $F_0$ and $F_1$ denote the pre- and post-change marginal distributions. However, our framework in Section~\ref{sec:parametric} can actually handle dependent data. Under dependence, $F_0$ and $F_1$ denotes the joint distributions of $X_1,\dots,X_{t-1}$ and $X_t,X_{t+1},\dots$ respectively, and is in fact ambivalent to whether $X_t$ is drawn independently of $X_1,\dots,X_{t-1}$ or conditional on it, so we do not emphasize the distinction in our notation.} 
As a corner case, let $\P_{F_0,\infty}$ be the joint distribution when $X_1,X_2,\dots \stackrel{}{\sim}F_0$. 
In parametric settings, we write $\P_{\theta_0,t,\theta_1}$ and $\P_{\theta_0,\infty}$ as shorthand for $\P_{F_{\theta_0},t,F_{\theta_1}}$ and $\P_{F_{\theta_0},\infty}$ respectively.

It is impossible for any set $\mathcal{C}^*\subseteq \{1,\cdots,\tau\}$, to produce an unconditional coverage guarantee at a predetermined level $1-\alpha$, if the detection algorithm has finite ARL (i.e., $\tau$ is finite almost surely even if there is no change). The intuitive reason is that if  $\tau$ is finite almost surely, it is possibly a false detection, i.e., $\tau<T$, in which case the set $\{1,\cdots,\tau\}$ itself does not cover $T$. The next proposition formalizes this fact. 
\begin{proposition}
\label{prop:impossibility}
    If, for some $\alpha\in(0,1)$, a confidence set $\mathcal{C}^*$ that is a subset of $\{1,\cdots,\tau\}$ satisfies $\P_{F_0,T,F_1}(T\in\mathcal{C}^*)\geq1-\alpha$,  $\forall T\in\N$, then $\P_{F_0,\infty}(\tau =\infty) \geq 1-\alpha$ (meaning that the PFA of the algorithm is at most $\alpha$ and thus the ARL is infinite). 
\end{proposition}
The proposition says that for any algorithm with finite ARL, it is impossible to construct a $1-\alpha$ confidence set that is only a subset of $\{1,\dots,\tau\}$.
Thus, for post-detection inference, the usual notion of coverage is not an appropriate one.
 Instead, we focus on conditional coverage \eqref{eq:cond-coverage-intro}, a more reasonable criterion, which was also considered in prior work \citep{ding2003lower,wu2007inference}. 
The key idea behind our methodology is that
for each candidate $t\in\{1,\cdots,\tau\}$, we test whether $t$ is the true changepoint from some pre-change class $\mathcal{P}_0$ to some post-change class $\mathcal{P}_1$:
\begin{equation}
    \label{h_0t}
 H_{0,t}: \; X_1,\dots,X_{t-1} \sim F_0  ~\text{ and }~  X_t,X_{t+1},\dots \sim F_1, \text{ for some } F_i\in \mathcal{P}_i; i=0,1.
\end{equation}
Let $\phi_{\alpha^*}^{(t)}=\mathds{1}( H_{0,t} \text{ is rejected})$ denote some test for the null $H_{0,t}$ such that  
$ \P_{F_0,t,F_1}(\phi_{\alpha^*}^{(t)}(X_1,\cdots,X_\tau)=1)\leq \alpha^{*}, \forall F_0\in\mathcal{P}_0, F_1\in\mathcal{P}_1$
and form the confidence set by inverting these tests:
\begin{equation}
    \mathcal{C}:=\{t\in\N: t\leq \tau, H_{0,t} \text{ is not rejected}\}.
\end{equation}
In classical statistics, inverting level-$\alpha$ tests yields $1-\alpha$ confidence sets. However, this relationship does not directly hold for us since we test the hypotheses up to a data-dependent time $\tau$, and  we are considering \emph{conditional} coverage.
\begin{proposition}
\label{prop:dual}
Let $\tau$ be a stopping time and $\mathcal{C} \subseteq \{1, \dots, \tau\}$ be a confidence set satisfying \eqref{eq:cond-cov}.
Then, the test defined by $\phi^{(t)}(X_1,\cdots,X_\tau)=\mathds{1}(t \notin \mathcal{C},t\leq\tau)$ satisfies
\begin{equation}
\label{eq:invert}
    \P_{F_0,t,F_1}(\phi^{(t)}(X_1,\cdots,X_\tau)=1)\leq \alpha \cdot \P_{F_0,\infty}(\tau \ge t), \forall t\in\N, F_0\in\mathcal{P}_0, F_1\in\mathcal{P}_1.
\end{equation}
Conversely, given any test $\phi^{(t)}$ for $H_{0,t}$ using data up to $\tau$ satisfying \eqref{eq:invert}, the inverted confidence
set $\mathcal{C}=\{t\in\N, t\leq \tau:\phi^{(t)}(X_1,\cdots,X_\tau)=0\}$ satisfies
\begin{equation}
\label{eq:cond-cov}
   \P_{F_0,T,F_1}(T \in \mathcal{C} \mid \tau \ge T) \ge 1 - \alpha, ~\forall T\in\N, F_0\in\mathcal{P}_0, F_1\in\mathcal{P}_1.
\end{equation}
\end{proposition}
The above result establishes a duality between tests and a post-detection confidence set.
Next, we  construct the latter by inverting tests based on the above.

\section{A general framework for nonparametric post-change}
\label{sec:general}
 Let $\mathcal{P}_0$ and $\mathcal{P}_1$ denote the pre- and post-change classes respectively, $\mathcal{P}_0\cap\mathcal{P}_1=\emptyset$.
 Our goal is to construct post-detection confidence sets for the changepoint $T$ by inverting tests for $H_{0,t}$, as formalized in \Cref{prop:dual}.

To motivate the general construction, we begin with the simplest case when both pre- and post-change distributions $F_0$ and $F_1$ (with densities $f_0$ and $f_1$) are known.
Then, the likelihood of the data $X_1,\cdots,X_{T-1}\stackrel{i.i.d}{\sim}F_0$, $X_T,\cdots,X_{n}\stackrel{i.i.d}{\sim}F_1$  is given by$L_T^n:=\prod_{i=1}^{T-1} f_0(X_i)\prod_{i=T}^n f_1(X_i)$, and the likelihood ratio under $H_{0,t}$ against $H_{0,n}$ can be written as $\prod_{j=i}^{t-1} \frac{f_1(X_i)}{f_0(X_i)}.$
 Now, a prototypical likelihood-ratio-based test statistic is the following:
  \begin{equation}
    \label{eq:test-stat-simp}
        M_t=\begin{cases}
           \displaystyle\max\left\{\max_{1\leq j\leq t-1} \prod_{i=j}^{t-1} \frac{f_1(X_i)}{f_0(X_i)},\max_{t\leq j\leq \tau} \prod_{i=t}^{j} \frac{f_0(X_i)}{f_1(X_i)}\right\}
             ,& \text{ if } t\leq \tau<\infty\\
             -\infty, &\text{ if } t>\tau\text{ or }\tau=\infty.
        \end{cases}
    \end{equation} 
One natural way to extend this test statistic for the general nonparametric setting is to  replace the likelihood ratio by e-processes \citep{ramdas2022game,grunwald2024safe},
which have emerged as a central tool in anytime-valid inference. While e-processes have been used for sequential testing \citep{wasserman2020universal,ramdas2022testing,saha2024testing}, estimation \citep{waudby2023estimating}, and change detection \citep{shin2022detectors}, we employ them for the first time for changepoint localization.
For $t \in \mathbb{N}$, we use two different types of e-processes:
\begin{itemize}
    \item \textbf{Forward $t$-delay e-process under $\mathcal{P}_1$ (against $\mathcal{P}_0$).} This is a stochastic process $\{R^{(t)}_n\}_{n \ge t}$ adapted to the forward filtration $\{\mathcal{F}^{(t)}_n\}_{n = t,t+1,\cdots}$, where
$\mathcal{F}^{(t)}_n := \sigma(X_t, X_{t+1}, \dots, X_n),$
    such that for all stopping times $\kappa$ with respect to $\{\mathcal{F}^{(t)}_n\}_{n \geq t}$,
 $ \sup_{F_1 \in \mathcal{P}_1} \mathbb{E}_{F_1}[R^{(t)}_\kappa] \le 1.$
    Ideally, this process should have power against $\mathcal{P}_0$, i.e., $R^{(t)}_n$ should grow to infinity (as $n\to\infty$) under $\mathcal{P}_1$.
    \item \textbf{Backward $t$-delay e-process under $\mathcal{P}_0$ (against $\mathcal{P}_1$).} This is a stochastic process $\{S^{(t)}_n\}_{n < t}$ adapted to the backward filtration $\{\mathcal{G}^{(t)}_n\}_{n =t-1,t-2,\cdots}$, where
  $\mathcal{G}^{(t)}_n := \sigma(X_n,X_{n+1},\cdots, X_{t-1}),$
    such that for all stopping times $\kappa$ with respect to $\{\mathcal{G}^{(t)}_n\}_{n <t}$,
  $ \sup_{F_0 \in \mathcal{P}_0}   \mathbb{E}_{F_0}[S^{(t)}_\kappa] \le 1.$
  Ideally, this process should have power against $\mathcal{P}_1$, i.e., $S^{(t)}_n$ should grow to infinity (as $n\to\infty$) under $\mathcal{P}_0$.
\end{itemize}
Such powered e-processes for $\mathcal{P}_1$ against $\mathcal{P}_0$ (and vice versa) exist if the total variation distance between the two convex hulls of $\mathcal{P}_0$ and $\mathcal{P}_1$ is positive (see e.g., Section 3.4 of \cite{ramdas2024hypothesis}). 
{ 
The likelihood ratio process is the simplest example of an e-process.
There is a substantial literature on constructing e-processes in both parametric (e.g., \cite{wasserman2020universal}) and nonparametric settings (e.g., \cite{waudby2023estimating}). For readers less familiar with these developments, we refer to Chapters 6 and 7 of the recent textbook \cite{ramdas2024hypothesis}.}


With the above ingredients in place, we construct a test statistic for $H_{0,t}$ by comparing it against a competing hypothesis $H_{0,T}$. 
Suppose $t < T$. Then under $H_{0,t}$ (which is false), the segment $X_t, \dots, X_{T-1}$ follows a post-change distribution from $\mathcal{P}_1$ (because under $H_{0,t}$, the change occurred at $t$), whereas under $H_{0,T}$ (which is true), it follows a pre-change distribution from $\mathcal{P}_0$. Outside this segment $X_t, \dots, X_{T-1}$, both hypotheses agree on the distribution of the observations, and hence those are non-informative for testing between $H_{0,t}$ and $H_{0,T}$. Therefore, we use a forward e-process $R^{(t)}_{T-1}$ that accumulates evidence in favor of $\mathcal{P}_1$ against $\mathcal{P}_0$ over this segment. Similarly, if $t > T$, we use a backward e-process $S^{(t)}_T$ to accumulate evidence for $\mathcal{P}_0$ against $\mathcal{P}_1$, using the segment $X_T, \dots, X_{t-1}$ . Since the true $T$ is unknown, we consider the maximum over all valid choices of $T$ for both the e-processes, i.e., $\displaystyle\max_{t\leq n\leq\tau} R^{(t)}_{n}$ and $\displaystyle\max_{1\leq n\leq t-1}S^{(t)}_{n}$. We define the final test statistic as the maximum of these two quantities:
 \begin{equation}
    \label{eq:test-stat-nonpara}
        M_t=\begin{cases}
            \displaystyle\max\left\{ \max_{t\leq n\leq\tau} R^{(t)}_{n}, \max_{1\leq n\leq t-1}S^{(t)}_{n}\right\},& \text{ if } t\leq \tau<\infty\\
             -\infty, &\text{ if } t>\tau\text{ or }\tau=\infty.
        \end{cases}
    \end{equation} 

Although, $M_t$ is not an e-variable under $H_{0,t}$ (nor are its constant multiples), thresholding $\frac{M_t}{2}$ at $\frac{1}{\alpha}$ yields a level-$\alpha$ test for $H_{0,t}$ --- a nontrivial fact shown below.
\begin{theorem}
\label{lem:univ-threshold}
  $\P_{F_0,t,F_1}[M_t\geq2/\alpha] \leq \alpha,$ for any $t\in\N, F_0\in\mathcal{P}_0, $ and $F_1\in\mathcal{P}_1.$
\end{theorem}
 While this bound may be conservative, it is remarkably universal: it holds without any assumptions on the distribution classes $\mathcal{P}_0$, $\mathcal{P}_1$, or on the detection method $\mathcal{A}$.
For constructing a valid confidence set,  \Cref{prop:dual} implies that we require tests of level $\alpha \cdot \P_{F_0,\infty}(\tau \geq t)$, whose value is of course not known to us.
To circumvent this issue, we first assume that the pre-change distribution $F_0$ is known. This is a common assumption in sequential changepoint literature.
 
 \subsection{Confidence set for \texorpdfstring{$T$}{Lg} with known pre-change distribution}
 \label{subsec:ci-simple}
Let $r_t$ be an unbiased estimator of $\mathbb P_{F_0,\infty}[\tau \geq t]$, that is independent of the data. Then, 
 define the confidence set as
\begin{equation}
\label{eq:ci-nonpara-known-pre}
\mathcal{C}=\left\{t\in\mathbb{N}: t\leq\tau, M_t<\frac{2}{\alpha r_t}\right\}.
\end{equation}
 One natural approach to estimate $\mathbb P_{F_0,\infty}[\tau \geq t]$ is via Monte Carlo simulation.
Fix $N \in \mathbb{N}$ and for each $j = 1, \dots, N$, simulate a sequence of i.i.d.\ data points $Z_1^{(j)}, Z_2^{(j)}, \dots \stackrel{i.i.d.}{\sim} F_0$ until either the detection algorithm triggers or the sample size reaches $\tau$—whichever occurs first. Let $\tau_j := \min\{\tau, \mathcal{A}(\{Z_n^{(j)}\}_n)\}$ denote the stopping time for the $j$-th simulated sequence.
Then $r_t = \frac{1}{N}\sum_{j=1}^N \mathds{1}(\tau_j \ge t)$ is a natural unbiased estimator of $\mathbb P_{F_0,\infty}[\tau \geq t]$ 
,but it can equal zero with positive probability. To avoid this, we fix some $r\geq 2$ and for each $t$, one might define $N_t$ to be a random index until $\tau_j \ge t$ happens exactly $r$ many times. Then $r_t =\frac{r-1}{N_t-1}$ is an unbiased estimator, which is strictly positive. Alternatively, one may use an asymptotically unbiased estimator with $N$, such as, $r_t = \frac{1}{N+1}\left(1+\sum_{j=1}^N \mathds{1}(\tau_j \ge t)\right)$, which guarantees strict positivity but only provides asymptotic $1-\alpha$ coverage guarantee as $N\to\infty$.
The next theorem proves the validity of the above confidence set.
\begin{theorem}
\label{thm:coverage-nonpara-known-pre}
Suppose that the pre-change distribution is known, i.e., $\mathcal{P}_0=F_0$. Then, for \textit{any} $\alpha\in(0,1), T,N\in \N$, post-change class $\mathcal{P}_1$, and change detection method $\mathcal{A}$ with stopping time $\tau$ satisfying $\mathbb \P_{F_0,\infty}(\tau\geq T)\neq 0$, the confidence set $\mathcal{C}$ defined in \eqref{eq:ci-nonpara-known-pre}, constructed using any data-independent unbiased estimator $r_t$ of $\mathbb P_{F_0,\infty}[\tau \geq t]$, satisfies $\P_{F_0,T,F_1}(T\in \mathcal{C}\mid \tau\geq T)\geq 1-\alpha, \forall  F_1\in\mathcal{P}_1.$ 
    \end{theorem}
{ Now we analyze the size of the 
 confidence set \eqref{eq:ci-nonpara-known-pre} for known pre-change $F_0$ and post-change $F_1$, with $r_t = \frac{1+\sum_{j=1}^N \mathds{1}(\tau_j\ge t)}{N+1}$, $M_t$ as defined in \eqref{eq:test-stat-simp}. We provide a bound on the expected size of the confidence set as $N \to \infty$, where $N$ is the number of simulations used to calculate $r_t$. Define, $p_t=\P_{F_0^*,\infty}[\tau\geq t]$. Note that $r_t \to p_t$ as $N \to \infty$.
Define
$\rho_0(s) =\mathbb E_{F_0}\left(\frac{f_1 (X_i)}{f_0 (X_i)}\right)^s$ and $\rho_1(s) =\mathbb E_{F_1}\left(\frac{f_0 (X_i)}{f_1 (X_i)}\right)^s$. These two quantities measure closeness between $F_0$ and $F_1$, and thus the problem hardness. Note that for $s=1/2,$ this coincides with Bhattacharyya distance. At $s=0,1$, $\rho_0$ and  $\rho_1$ take value $1.$
Since for $s\in(0,1)$, $h(x)=x^s$ is a strictly concave function and since $F_1 \neq F_0$, $\frac{f_1(X)}{f_0(X)}$ is not almost everywhere constant, we get a strict inequality using Jensen's:
$\rho_0(s) =\mathbb{E}_{F_0} \left[ \left( \frac{f_1(X)}{f_0(X)} \right)^s \right] < \left( \mathbb{E}_{F_0} \left[ \frac{f_1(X)}{f_0(X)} \right] \right)^s=1.$
Similarly, for $s\in(0,1), \rho_1(s)<1$. We now define $s_0 = \arg\min_{s \in [0,1]} \rho_0(s)$ and $s_1= \arg\min_{s \in [0,1]} \rho_1(s)$, and note that the preceding facts imply that $0 < s_0 < 1$ and $0 < s_1 < 1$.
Our length bound will depend on the expected detection delay, $\Delta_{F_0,T,F_1}:=\mathbb E_{F_0,T,F_1}(|\tau-T|\mid \tau\geq T).$
For typical detection algorithms, this quantity often scales like the logarithm of the average run length divided by the KL divergence between $F_1,F_0$.
Finally, we call the stopping time $\tau$ as \emph{sensitive} if 
\begin{equation}~\label{def:sensitive}
\P_{F_0,T,F_1}[\tau\leq t]\geq \P_{F_0,\infty}[\tau\leq t],
\end{equation}
for all $t$ and $T$. 
Equivalently, 
$\tau$ under a change is stochastically smaller than $\tau$ under no change. We can expect all practically used detection algorithms to be sensitive, as they should stop later when a change is absent than when it is present.
\begin{theorem}
\label{thm:len}
In the setting described above, for any change detection method with stopping time $\tau$ satisfying $p_T\neq 0$,
    the size of ${\mathcal{C}}$ 
 can be bounded as:
\begin{align}
\label{eq:len}
 \limsup_{N\to\infty} \mathbb E_{F_0,T,F_1} (|{\mathcal{C}}|\mid \tau\geq T)\leq  \frac{2^{s_0}}{\alpha^{s_0} p_{{T}}^{s_0+1}}\cdot\frac{\rho_0(s_0)-\rho_0(s_0)^{T-1}}{1-\rho_0(s_0)}+1+\Delta_{F_0,T,F_1}.
\end{align}
    Additionally, if $\tau$ is sensitive in the sense of~\eqref{def:sensitive}, the last term in~\eqref{eq:len} can be improved to $\Psi_{F_0,T,F_1}:=\min\left\{\Delta_{F_0,T,F_1}, \frac{1}{p_T}\left(  \left(\frac{2}{\alpha}\right)^{s_1}\times\frac{\rho_1(s_1)^{s_1}}{1-\rho_1(s_1)^{s_1}}+\frac{\rho_1(s_1)}{1-\rho_1(s_1)}\right)\right\}$.
\end{theorem}

\begin{remark}
Suppose  $F_1=N(\mu_1,\sigma^2)$ vs $F_0=N(\mu_0,\sigma^2)$ and we are using a CUSUM detector with the likelihood ratio (which satisfies the assumption) with average run length $A$.
Consider the asymptotic regime when $A \to\infty$, and $T=C A$,  for some constant $C>0$.
Using the property of the CUSUM detector that $\frac{\tau}{\mathbb E_{F_0,\infty}(\tau)}\stackrel{d}{\to}\exp(1)$ as  $A\to\infty$, we have $p_T=\mathbb \P_{F_0,\infty}(\tau\geq T)=\mathbb \P_{F_0,\infty}(\tau\geq CA_T)\to \exp(-C)$. Also, for the Gaussian distribution, it is easy to verify that $\rho_0(s) =\rho_1(s)= \exp\left( -s(1-s)K \right)$, where $K=\frac{(\mu_1-\mu_0)^2}{2\sigma^2}$ and $s_0=s_1=1/2$. Substituting, we get 
\begin{align*}
\label{eq:len-asymp}
 \limsup_{T,A\to\infty,T=CA} \limsup_{N \to \infty}  \mathbb E_{F_0,T,F_1} \left(|\mathcal{C}|\mid \tau\geq T\right)\leq&  \sqrt{\frac{2}{\alpha}}\left[\frac{\exp(\frac{3C}{2}-\frac{K}{4})}{1-\exp\left(-\frac{K}{4}\right)}+\frac{\exp\left(C-\frac{K}{8}\right)}{1-\exp\left(-\frac{K}{8}\right)}\right]\\
 &+1+\frac{\exp(C-\frac{K}{4})}{1-\exp\left(-\frac{K}{4}\right)}.
    \end{align*}
To simplify further, consider the difficult regime of small $K$, where $\left(\exp\left(\frac{K}{4}\right)-1\right) \approx K/4$, and $\left(\exp\left(\frac{K}{8}\right)-1\right) \approx K/8$, and take $C =\frac{1}{3}\log(\frac{2}{\alpha})$, so that the above expression is bounded by, ignoring constants, $K^{-1}[\frac{1}{\alpha}+ \frac{1}{\alpha^{5/6}}++ \frac{1}{\alpha^{1/3}}]$. The fact that the bound grows with $C=T/A$ is empirically verified in Supplementary Section G.1. This behavior is not a limitation of our methodology; rather, it reflects a fundamental difficulty of the problem: $p_t$ decays as $C$ increases, and hence, by our duality result (\Cref{prop:dual}), the tests corresponding to \emph{any} confidence set will have smaller levels $\alpha p_t$, leading to reduced power and, consequently, wider confidence sets.
\end{remark}
}
\subsection{An extension to composite pre-change}
To extend our method to composite pre-change classes $\mathcal{P}_0$,  we require the following assumption on the change detection method and $\mathcal{P}_0$:
\begin{assumption}
\label{assmp-1}
    $\exists P_0^*\in\mathcal{P}_0$ such that $\inf_{P_0\in\mathcal{P}_0}\P_{P_0,\infty}(\tau\geq t)\geq \P_{P_0^*,\infty}(\tau\geq t)$,  $\forall t\in\N$. 
\end{assumption}

{ 
We justify this assumption by appealing to Huber's  notion of ``least favorable distribution pair" \citep{huber1965robust} in terms of risk, where the risk $R$ for a hypothesis test $\phi$ between $\mathcal{P}_0$ vs. $\mathcal{P}_1$ is defined as
\begin{equation}
\label{risk}
      R(P_0, \phi) = C_0 P_{0}(\phi \text{ accepts } \mathcal{P}_1),        ~~
        R(P_1, \phi) = C_1 P_{1}(\phi \text{ accepts } \mathcal{P}_0).
\end{equation}
The formal definition of a least favorable distribution pair is given below.
\begin{definition}
    A pair $(P_0^*, P_1^*) \in \mathcal{P}_0 \times \mathcal{P}_1$ is called \textit{least favorable distribution (LFD) pair} in terms of risk $R$ for testing $\mathcal{P}_0$ vs. $\mathcal{P}_1$, if for every likelihood ratio test $\phi$ between $P_0^*$ and $P_1^*$, $R(P_j^*,\phi) \geq R(P_j,\phi), ~\text{ for all } P_j \in \mathcal{P}_j,~ j=0,1.$
\end{definition}

The following proposition establishes that the existence of an LFD pair is sufficient for our assumption to hold.
\begin{proposition}
\label{prop:mlr}
$(P_0^*,P_1^*)$ is a least favourable distribution pair in terms of risk for testing $\mathcal{P}_0$ vs. $\mathcal{P}_1$, then  \Cref{assmp-1} holds with $F^*=P^*$, and $\tau$ being the stopping time of a CUSUM/SR change detector with the likelihood ratio of $P_1^*$ to $P_0^*$. 
\end{proposition}

LFD pairs are known to exist for various parametric and nonparametric models, e.g., the monotone likelihood ratio (MLR) families, and Huber's robust models (including $\epsilon$ neighborhood model, total variation model, and more generally with 2-alternating capacities \cite{huber1973minimax}).
}

Suppose we have estimators $r_t^*$ of $\mathbb P_{F_0^*,\infty}[\tau \geq t]$, for each $t$ (which can be constructed using simulations from $F^*_0$ as described in \Cref{subsec:ci-simple}). Now, we analogously
define the confidence set as
\begin{equation}
\label{eq:ci-nonpara}
            \mathcal{C}=\left\{t\in\N: t\leq\tau, M_t<\frac{2}{\alpha r_t^*}\right\}.
        \end{equation}
The next theorem proves the validity of the above confidence set.
\begin{theorem}
\label{thm:coverage-nonpara-comp-pre}
Suppose that the pre-change class $\mathcal{P}_0$ and change detection method $\mathcal{A}$ with stopping time $\tau$ satisfy \Cref{assmp-1} and $\mathbb \P_{F_0,\infty}(\tau\geq T)\neq 0, \forall  F_0\in\mathcal{P}_0$. Then, for \textit{any} $\alpha\in(0,1),$ $T,N\in \N$, post-change class $\mathcal{P}_1$, the confidence set  $\mathcal{C}$ defined in \eqref{eq:ci-nonpara}, constructed using any data-independent unbiased estimator $r_t^*$ of $\mathbb P_{F_0^*,\infty}[\tau \geq t]$,
satisfies 
$\P_{F_0,T,F_1}(T\in \mathcal{C}\mid \tau\geq T)\geq 1-\alpha, \forall  F_0\in\mathcal{P}_0, F_1\in\mathcal{P}_1.$ 
\end{theorem}
{ 
We now analyze the size of the confidence set, generalizing \Cref{thm:len} to the case where an LFD pair $(F_0^*,F_1^*)$ exists for $\mathcal{P}_0$ vs $\mathcal{P}_1$. We consider the confidence set in \eqref{eq:ci-nonpara} with $r_t^* = \frac{1+\sum_{j=1}^N \mathds{1}(\tau_j^* \ge t)}{N+1}$, $M_t$ as defined in \eqref{eq:test-stat-nonpara} using $R_n^{(t)}=\prod_{i=t}^{n} \frac{f_{0}^*(X_i)}{f_{1}^*(X_i)},$ and $S_n^{(t)}=\prod_{i=n}^{t-1} \frac{f_{1}^*(X_i)}{f_{0}^*(X_i)}$.  We redefine $p_t,\rho_0$ and $\rho_1$ in terms of the LFD pair: $p_t=\P_{F_0^*,\infty}(\tau\geq t)$, $\rho_0(s) =\mathbb E_{F_0^*}\left(\frac{f_1^* (X_i)}{f_0^* (X_i)}\right)^s$ and $s_0 = \arg\min_{s \in (0,1)} \rho_0(s)$ and $\rho_1(s) =\mathbb E_{F_1}\left(\frac{f_0^* (X_i)}{f_1^* (X_i)}\right)^s$ and $s_1 = \arg\min_{s \in (0,1)} \rho_1(s)$.
Finally, we call the stopping time $\tau$ \emph{well-behaved} if it satisfies
\begin{subequations}
\label{eq-group} 
\begin{align}
\P_{F_0,\infty}(\tau\geq t)&\geq \P_{F_0^*,\infty}(\tau\geq t), \forall t\in\N, \forall F_0\in\mathcal{P}_0;\label{eq-a}\\
    \P_{F^*_0,T,F^*_1}[\tau\leq t]&\geq \P_{F^*_0,\infty}[\tau\leq t], \forall t, T\in\N ~\text{ and }\label{eq-b}  \\ 
    \P_{F_0,T,F_1}[\tau\leq t]&\geq \P_{F^*_0,T,F^*_1}[\tau\leq t], \forall t, T\in\N, \text{ and }\forall(F_0,F_1)\in\mathcal{P}_0\times\mathcal{P}_1.\label{eq-c}    
\end{align}
\end{subequations}
Note that \eqref{eq-b} is analogous to \eqref{def:sensitive}, while \eqref{eq-c} requires that the stopping time $\tau$ under change from $F_0^*$ to $F_1^*$ is stochastically larger than $\tau$ under change from $F_0$ to $F_1$, for any $(F_0,F_1)\in\mathcal{P}_0\times\mathcal{P}_1.$ To justify that \eqref{eq-c} is sensible, we note that the LFD pair $(F_0^*,F_1^*)$ is the closest pair within $\mathcal{P}_0\times\mathcal{P}_1$ (according to various metrics or divergences such as Kullback-Leibler), i.e., it renders the change detection problem most difficult. Therefore, it is reasonable to assume that $\tau$ under change from $F_0^*$ to $F_1^*$ is stochastically largest. \eqref{eq-a} can be justified similarly because $F_0^*$ is the ``closest'' element of the pre-change class to the post-change class.
\begin{theorem}
\label{thm:len:comp}
In the setting described above, for any change detection method with stop time $\tau$ that is well-behaved in the sense of \eqref{eq-a} and $p_T\neq 0$, 
\begin{align}
\label{eq:len-comp}
 \limsup_{N\to\infty} \sup_{\substack{F_0\in\mathcal{P}_0\\ F_1\in\mathcal{P}_1}}\mathbb E_{F_0,T,F_1} (|{\mathcal{C}}|\mid \tau\geq T)\leq  \frac{2^{s_0}}{\alpha^{s_0} p_{{T}}^{s_0+1}}\frac{1-\rho_0(s_0)^{T-1}}{1-\rho_0(s_0)}+1+\Delta_{\mathcal P_0,T,\mathcal P_1},
\end{align}
where $\Delta_{\mathcal P_0,T,\mathcal P_1}:=\displaystyle\sup_{\substack{F_0\in\mathcal{P}_0,F_1\in\mathcal{P}_1}}\Delta_{F_0,T,F_1}$.
Additionally, if the stopping time $\tau$ is well-behaved in the sense of \eqref{eq-b} and \eqref{eq-c} as well, the last term in \eqref{eq:len-comp} can be improved to
$$\min\left\{\frac{1}{p_T}\left(  \left(\frac{2}{\alpha}\right)^{s_1}\times\frac{\rho_1(s_1)^{s_1}}{1-\rho_1(s_1)^{s_1}}+\frac{\rho_1(s_1)}{1-\rho_1(s_1)}\right), \Delta_{\mathcal P_0,T,\mathcal P_1}\right\}.$$
\end{theorem}
}
{ 
}

 Next, we explore how, in parametric settings, one can obtain adaptive thresholds via simulation. Although such adaptive thresholds can sometimes be tighter, our simulations indicate that the resulting confidence sets from the universal threshold are not significantly wider, and are useful for their simplicity and generality.

\section{A simulation-based approach for parametric pre- and post-change settings} 
\label{sec:parametric}
In this section, we introduce an alternative approach that, while restricted to the parametric setting, allows observations to be dependent, and often yields tighter confidence sets than the method proposed in the previous section, as shown by our simulations later. We begin with the most basic case, where the pre- and post-change distributions are known, and then extend it to more general parametric settings.


\subsection{Known pre- and post-change distributions}

We develop a simulation-based scheme to approximate the exact $\alpha$-th quantile of the distribution of the test statistic $M_t$. Here, $M_t$ need not be of the form \eqref{eq:test-stat-nonpara}, it can be any test statistic based on $X_1,\cdots,X_\tau$. Assuming known pre- and post-change distributions, we draw $B$ many independent streams of data under $H_{0,t}$. For the $j$-th sequence, denoted as $\{X_n^j\}_n$, we simulate data points until a change is detected at  $\tau_j^\prime=\mathcal{A}(\{X_n^j\}_n)$ or until a fixed time $L$, whichever happens first. Then, we compute a test statistic $M_{t,L}^j$, based on $X_1^j,\cdots,X_{\tau_j^\prime\wedge L}^j$ (details are described in Supplementary Section B).
Our confidence set is the collection of all $t\in\{1,\cdots,\tau\}$, for which we fail to reject $H_{0,t}$:
\begin{equation}
    \label{eq:ci-simple}
  \mathcal{C}_1=\left\{t\in\N: t\leq \tau, M_t\leq\operatorname{Quantile}(1-\alpha r_{t}; M_t,M_{t,L}^1,\cdots,M_{t,L}^B)\right\},
\end{equation}
where $r_t$ is any unbiased estimator of $\P_{\theta_0,\infty}(\tau\geq t)$; see
Algorithm~\ref{algo:1-exact}.

\begin{algorithm}[h!]
\caption{CI for known pre-change $F_0$ and post-change $F_1$}
\label{algo:1-exact}
\begin{algorithmic}[1]
\Require $\alpha$, $N$, $B$, $\mathcal{A}$, and data $\{X_n\}_{n}$ until a change is detected at $\tau$ using $\mathcal{A}$
\Ensure A confidence set $\mathcal{C}$ for the changepoint $T$
\State $\mathcal{C} \gets \varnothing$
\For{$i = 1$ to $N$}
    \State Draw $\{Z_n^i\}_n \stackrel{}{\sim} F_0$, until $\mathcal{A}$ detects at $\tau_i$, or until $\tau$, whichever happens first
\EndFor
\For{$t = 1$ to $\tau$}
    \State Compute test statistic $M_t$ (e.g.\ \eqref{eq:test-stat-simp}) based on $\tau$ and $X_1,\dots,X_\tau$
    \For{$j = 1$ to $B$}
        \State Simulate $X_1^j,\dots,X_{t-1}^j \stackrel{}{\sim} F_0$; $X_{t}^j, X_{t+1}^j,\dots \stackrel{}{\sim} F_1$, until $\mathcal{A}$ detects a change at $\tau_j'$, or time $L$, whichever happens first
        \State \underline{Subroutine $\mathcal{L}$:} Compute $M_{t,L}^j$ based on $\tau_j'$ and $X_1^j,\dots,X_{\min(\tau_j', L)}^j$
    \EndFor
    \If{$M_t \leq \operatorname{Quantile}\left(1 - \frac{\alpha}{N} \times  \sum_{i=1}^N \mathds{I}(\tau_i \geq t) ;\, M_t, M_{t,L}^1,\dots,M_{t,L}^B \right)$}
        \State $\mathcal{C} \gets \mathcal{C} \cup \{t\}$
    \EndIf
\EndFor
\end{algorithmic}
\end{algorithm}

\begin{theorem}
\label{thm:exact-cond-coverage}
    For any fixed $\alpha\in(0,1)$, $B,N,T\in\mathbb N$, $L\in\N\cup\{\infty\}$, and change detection algorithm $\mathcal{A}$ with $\mathbb \P_{F_0,\infty}(\tau\geq T)\neq 0$, the confidence set $\mathcal{C}_1$  defined in \eqref{eq:ci-simple}, constructed using any data-independent unbiased estimator $r_t$ of $\mathbb P_{F_0,\infty}[\tau \geq t]$, satisfies
$  \P_{F_0,T,F_1} (T\in\mathcal{C}_1\mid \tau\geq T)\geq1-{\alpha}.$
\end{theorem}

\begin{remark}
Notably, this method allows arbitrary (known) dependence among observations and is valid for any choice of $B,N,L$, test statistic $M_t$, and detection algorithm $\mathcal{A}$, which can function as an entirely black-box or heuristic method. The only necessity is a stopping rule that can be applied to generated data sequences.
    
\end{remark}

\subsection{Adaptive confidence set for known pre-change distribution and parametric post-change class}
\label{sec:composite-post}

We now assume that the post-change distribution follows some parametric composite model $\mathcal{P}_1=\{F_{\theta}:  {\theta}\in \Theta_1\}$ and the pre-change distribution, $F_{\theta_0} (\theta_0\notin\Theta_1)$ is known.

Since the post-change parameter is unknown, for each $t$, we estimate it based on $X_t,\cdots,X_\tau$.  However, the sample size used for the estimate might be too small to have a reliable estimate. To address this issue, we would use interval estimates instead of point estimates, allowing us to achieve a predetermined control over the error. Notably, conventional confidence intervals are invalid here, since $\tau$ is random. Therefore, we rely on confidence sequences \citep{darling1967confidence,howard2021time}. A confidence sequence having coverage $1-c$ is defined as a sequence of confidence sets
$\{\operatorname{CS}( X_1,\cdots,X_n;1-c)\}_{n\in\N}$, satisfying the uniform guarantee: 
\begin{equation}
\label{eq:cs-theta1}
    \P(\forall n\in \N, \theta\in \operatorname{CS}(X_1,\cdots,X_n;1-c))\geq 1-c.
\end{equation}
We construct $\mathcal{S}^\prime_{\tau-t+1}=\operatorname{CS}(X_t,\cdots,X_{\tau};1-\beta r_t)$, which have $1-\beta\times\P_{\theta_0,\infty}(\tau\geq t)$ coverage,  where 
$r_{t}$ is any unbiased estimator of $\P_{\theta_0,\infty}(\tau\geq t)$. 
 
Fix $B\in\N$. For each $j=1,\cdots,B$, suppose we have a sequence $\{X^j_n(\theta_0,t,\theta)\}_n$ $\sim\P_{\theta_0,t,\theta}$. Note that for a fixed $t$ and $\theta$, the sequences must be independent across $j$. However, for any fixed $j$, the i.i.d. sequences $\{X^j_n(\theta_0,\theta)\}_n$ are allowed to be dependent across different values of $t$ and  $\theta\in\mathcal{S}^\prime_{\tau-t+1}$. A specific method for generating these sequences is detailed in Section C of the Supplementary. We draw the data until a change is detected at $\tau_{j,t}^\theta$ using $\mathcal{A}$ or until a fixed time $L$, whichever occurs first. Then, we compute the same test statistics truncated at $L$, based on $\tau_{j,t}^\theta$ and $X^j_1(\theta_0,t,\theta),\cdots,X^j_{\tau_{j,t}^\theta\wedge L}(\theta_0,t,\theta)$, which is denoted as $M^j_{t,L}(\theta)$.
For $L<\infty$,
$$\begin{aligned}
\label{eq:Mt-L-comp-post}
    M^j_{t,L}(\theta) := \begin{cases}
    -\infty, &\text{if }t> \tau_{j,t}^\theta,\\
        M_t^{(\tau_{j,t}^\theta)}(X_1^j(\theta_0,t,\theta),\cdots,X_{\tau_{j,t}^\theta}^j(\theta_0,t,\theta)), &\text{if } t\leq\tau_{j,t}^\theta\leq L\\
        \infty, &\text{if } t\leq\tau_{j,t}^\theta \text{ and }\tau_{j,t}^\theta > L.
    \end{cases}
\end{aligned}$$
And for $L=\infty$, define
    $M^j_{t,\infty} := 
M_t(X_1^j(\theta_0,t,\theta),\cdots,X_{\tau_{j,t}^\theta}^j(\theta_0,t,\theta)), \text{ if } t\leq\tau_{j,t}^\theta<\infty$ and equal to 
        $-\infty, \text{ if }t> \tau_{j,t}^\theta \text{ or }\tau_{j,t}^\theta =\infty.$
Our confidence set is the collection of all $t$ before $\tau$, for which we fail to reject $H_{0,t}$ at level $\alpha\times\P_{\theta_0,\infty}(\tau\geq t)$:
\begin{equation}
\label{eq:conf-comp-post}
    \mathcal{C}_1=\left\{t\in \N:t\leq \tau, M_t\leq\sup_{\theta\in\mathcal{S}^\prime_{\tau-t+1}}\operatorname{Quantile}\left(1- \alpha r_{t};M_t,\{M_{t,L}^j(\theta)\}_{j=1}^B\right)\right\}.
\end{equation}


Algorithm~\ref{algo:comp-post} summarizes the method. 
Note that, even when $\mathcal{S}^\prime_{\tau-t+1}$ is an infinite set, the threshold $\sup_{\theta\in\mathcal{S}^\prime_{\tau-t+1}}\operatorname{Quantile}(1-\hat\alpha_{t};M_t,\{M_{t,L}^j(\theta)\}_{j=1}^B)$ can be approximated by discretizing $\mathcal{S}^\prime_{\tau-t+1}$ into a grid and evaluating the quantile for values of 
$\theta$ on this grid, taking their maximum. While this approach is conceptually straightforward, it is computationally intensive.
Moreover, in many cases, we can derive a suitable upper bound of this threshold, which is computationally, statistically efficient and does not require any approximation. Details are provided in Section C of the Supplementary, along with a concrete example of the Gaussian mean change problems.

\begin{algorithm}[h!]
\caption{CI for known pre-change ($F_{\theta_0}$) and parametric post-change}
\label{algo:comp-post}
\begin{algorithmic}[1]
\Require $\alpha, \beta$, $N,B,L$, $\operatorname{CS(data;coverage)}$,  $\mathcal{A}$, and data $\{X_n\}_{n}$ until a change is detected at $\tau$ using $\mathcal{A}$
\Ensure Confidence sets $\mathcal{C}$ for changepoint $T$

\For{$i=1,\cdots,N$}
    \State Draw $\{X_n^i\}_n \stackrel{iid}{\sim} F_{\theta_0}$, until $\mathcal{A}$ detects at $\tau_i$  or until $\tau$, whichever happens first.
\EndFor

\State $\mathcal{C}= \varnothing$
\For{$t=1,\cdots,\tau$}
    \State Compute some test statistics $M_t$ based on $X_1,\cdots,X_\tau$ 
    \State Compute $r_t = \sum_{i=1}^N \mathds{I}(\tau_i\geq t)/N$
    \State Find a confidence sequence for $\theta_1$: $\mathcal{S}^\prime_{\tau-t+1} = \operatorname{CS}(X_t,\cdots,X_{\tau};1-\beta r_t)$
    
    \For{$j=1,\cdots,B$}
        \State \underline{Subroutine 2}:
        \For{$\theta \in \mathcal{S}^\prime_{\tau-t+1}(X_t,\cdots,X_{\tau})$}
            \State Get $\{X^j_n(\theta_0,t,\theta)\}_n$ following $\P_{\theta_0,t,\theta}$ (independently for all $j$), until its detection time $\tau_{j,t}^\theta$ using $\mathcal{A}$ or until time $L$, whichever happens first.
            \State Compute $M_{t,L}^j(\theta)$ based on  $\tau_{j,t}^\theta$, and $X^j_1(\theta_0,t,\theta),\cdots,X^j_{\tau_{j,t}^\theta \wedge L}(\theta_0,t,\theta)$
        \EndFor
    \EndFor

    \State \underline{Subroutine 3}:
    \If{there exists $\theta \in \mathcal{S}^\prime_{\tau-t+1}$ such that $M_t \leq \operatorname{Quantile}(1-\alpha r_t; M_t, \{M_{t,L}^j(\theta)\}_{j=1}^B)$}
        \State $\mathcal{C} = \mathcal{C} \cup \{t\}$
    \EndIf
\EndFor
\end{algorithmic}
\end{algorithm}

\begin{theorem}
\label{thm:coverage-comp-post} 
   Let $\mathcal{C}_1$ be the confidence set for changepoint $T\in\N$ as defined in \eqref{eq:conf-comp-post}, constructed using any data-independent unbiased estimator $r_t$ of $\mathbb P_{\theta_0,\infty}[\tau \geq t]$, for some $\alpha,\beta\in(0,1)$, $N,B\in\mathbb N$,  $L\in\N \cup \{\infty\}$. Then for any given detection algorithm $\mathcal{A}$ with $\mathbb \P_{\theta_0,\infty}(\tau\geq T)\neq 0$ and any $\theta_1\in\Theta_1$,
$  \P_{\theta_0,T,\theta_1} (T\in \mathcal{C}_1\mid \tau\geq T)\geq 1-\alpha-\beta.$ 

\end{theorem}

\subsection{Adaptive confidence set for parametric pre- and post-change classes}
\label{sec:composite}
We now assume that both pre- and post-change distributions follow some parametric composite models: $\mathcal{P}_0=\{F_{\theta}:\theta\in\Theta_0\}$ and $\mathcal{P}_1=\{F_{\theta}:\theta\in\Theta_1\}$, $\Theta_0\cap\Theta_1= \emptyset$.

Since both pre- and post-change parameters are unknown, for each $t$, we estimate them based on $X_1,\cdots,X_{t-1}$ and $X_t,\cdots,X_\tau$ respectively. Instead of point estimates, we rely on interval estimates to achieve a predetermined control over the error.
Suppose that we have some confidence interval (CI) procedure  $\operatorname{CI}(X_1,\cdots,X_n;1-c)$ that constructs CI using $\{X_i\}_{i=1}^n$ for $\theta_0$ with coverage $1-c$ and a confidence sequence (CS) procedure $\{\operatorname{CS}(X_1,\cdots,X_n;1-c)\}_n$, that constructs CS using $\{X_i\}_{i=1}^n$ for $\theta_1$ having coverage $1-c$. 

We construct $\mathcal{S}_{t-1}=\operatorname{CI}(X_1,\cdots,X_{t-1};1-\gamma r_t^*)$ and $\mathcal{S}^\prime_{\tau-t+1}=\operatorname{CS}(X_t,\cdots,X_{\tau};1-\beta r_t^*)$, where 
$r_{t}^*$ is computed under \Cref{assmp-1}.

For each $j=1,\cdots,B$, for some fixed $B\in\N$, suppose we have sequences $\{X^j_n(\theta,t,\theta^\prime)\}_n$ having joint distribution as $\P_{\theta,t,\theta^\prime}$. Note that the sequences must be independent as $j$ varies. However, for any fixed $j$, the i.i.d. sequences $\{X^j_n(\theta,t,\theta^\prime)\}_n$ can be dependent across $t$ and  $\theta\in\mathcal{S}_{t-1},\theta\in\mathcal{S}^\prime_{\tau-t+1}$.  A specific method for generating these sequences is detailed in Section D of the Supplementary. We draw the data until a change is detected at the stopping time $\tau^{j,t}_{\theta,\theta^\prime}$ using $\mathcal{A}$ or until time $L$, whichever occurs first. Then, we compute the same test statistics truncated at $L$, based on $\tau^{j,t}_{\theta,\theta^\prime}$ and $X^j_1(\theta,t,\theta^\prime),\cdots,X^j_{\tau^{j,t}_{\theta,\theta^\prime}\wedge L}(\theta,t,\theta^\prime)$, which is denoted as $M^j_{t,L}(\theta,\theta^\prime)$.

 Similar to \eqref{eq:conf-comp-post}, the adaptive confidence set is defined as
\begin{equation}
\label{eq:conf-comp}
\mathcal{C}_1=\left\{t\in \N:t\leq \tau, M_t\leq\sup_{(\theta, \theta^\prime)\in \mathcal{\mathbf S}^{(t)}}\operatorname{Quantile}(1-{\alpha}r_{t}^*;M_t,\{M_{t,L}^j(\theta,\theta^\prime)\}_{j=1}^B)\right\}
\end{equation}
where $\mathcal{\mathbf S}^{(t)}=\mathcal{S}_{t-1}\times\mathcal{S}^\prime_{\tau-t+1}$.
Algorithm~\ref{algo:comp} contains an overview.
If the detection algorithm controls PFA, a simpler algorithm can be used, which avoids \Cref{assmp-1} and provides coverage guarantees, as elaborated in Section F of the Supplementary.

As before, one can always adopt a grid-based approach, which is conceptually straightforward, but it is computationally intensive; corresponding experimental results are reported in Section F of the Supplementary.
Additionally, in many cases, we can derive a suitable upper bound of this threshold, which is computationally efficient and exact, but that can be a bit more conservative. We provide the details in Section D of the Supplementary, along with a concrete example.

\begin{algorithm}[h!]
\caption{CI for parametric pre-change and post-change}
\label{algo:comp}
\begin{algorithmic}[1]
\Require $\alpha, \beta, \gamma, \theta_0^*$, $N, B, L$, $\mathcal{A}$, $\operatorname{CI(data;coverage)}$, $\operatorname{CS(data;coverage)}$, and data $\{X_n\}_n$ until a change is detected at $\tau$ using $\mathcal{A}$
\Ensure A confidence set $\mathcal{C}$ for changepoint $T$
\State $\mathcal{C} \gets \varnothing$
\For{$i = 1$ to $N$}
    \State Draw $\{X_n^i\}_n \stackrel{\text{iid}}{\sim} F_{\theta_0^*}$, until $\mathcal{A}$ detects at $\tau_i^*$  or until $\tau$, whichever happens first.
\EndFor
\For{$t = 1$ to $\tau$}
    \State Compute some test statistics $M_t$ based on $X_1,\cdots,X_\tau$ 
    \State Compute $r_t^* = \sum_{i=1}^N \mathds{I}(\tau_i^* \geq t)/N$
    \State Find a confidence interval for $\theta_0$: $\mathcal{S}_{t-1} = \operatorname{CI}(X_1,\cdots,X_{t-1}; 1 - \gamma r_t^*)$
    \State Find a confidence sequence for $\theta_1$: $\mathcal{S}^\prime_{\tau - t + 1} = \operatorname{CS}(X_t,\cdots,X_{\tau}; 1 - \beta r_t^*)$
    \For{$j = 1$ to $B$}
        \State \underline{Subroutine 1:}
        \For{each $(\theta, \theta^\prime) \in \mathcal{S}_{t-1} \times \mathcal{S}^\prime_{\tau - t + 1}$}
            \State Simulate $\{X_n^j(\theta,t,\theta^\prime)\}_n$ following $\mathbb{P}_{\theta,t,\theta^\prime}$ (independently for all $j$) until its detection time $\tau^{j,t}_{\theta,\theta^\prime}$ using $\mathcal{A}$ or until time $L$, whichever happens first
            \State Compute $M_{t,L}^j(\theta,\theta^\prime)$ based on $\tau^{j,t}_{\theta,\theta^\prime}$, $X_1^j(\theta,t,\theta^\prime), \cdots, X_{\tau^{j,t}_{\theta,\theta^\prime} \wedge L}^j(\theta,t,\theta^\prime)$
        \EndFor
    \EndFor
    \State \underline{Subroutine 2:}
    \If{$M_t \leq \sup_{(\theta, \theta^\prime) \in \mathcal{S}_{t-1} \times \mathcal{S}^\prime_{\tau - t + 1}} \operatorname{Quantile}(1 - \alpha r_t^*; M_t, \{M_{t,L}^j(\theta,\theta^\prime)\}_{j=1}^B)$}
        \State $\mathcal{C} \gets \mathcal{C} \cup \{t\}$
    \EndIf
\EndFor
\end{algorithmic}
\end{algorithm}
\begin{theorem}
\label{thm:coverage-comp} 
   Let $\mathcal{C}_1$ be the confidence sets for changepoint $T\in\N$ defined in \eqref{eq:conf-comp}, constructed using any data-independent unbiased estimator $r_t^*$ of $\mathbb P_{\theta_0^*,\infty}[\tau \geq t]$, for some $\alpha,\beta,\gamma\in(0,1)$, $N,B\in\mathbb N$, $L\in\N \cup \{\infty\}$. For any $\theta_0\in\Theta_0, \theta_1\in\Theta_1$, and any change detection algorithm $\mathcal{A}$ satisfying \Cref{assmp-1} and  having $\mathbb \P_{\theta_0,\infty}(\tau\geq T)\neq 0$, the following  guarantee holds:
$\P_{\theta_0,T,\theta_1} (T\in \mathcal{C}_1\mid \tau\geq T)\geq 1-\alpha-\beta-\gamma.$
\end{theorem}

\begin{remark}
    As a special case, when pre- and post-change parameters are known, i.e., $\Theta_i=\{\theta_i\}$ are singleton sets, $i=0,1$, we have $\mathcal S=\{\theta_0\},\mathcal S^\prime=\{\theta_1\}$ with $\beta=0$. So, in this case, \Cref{algo:comp,algo:comp-post} reduces to our algorithm designed for the known pre- and post-change setting, and similarly, \Cref{thm:coverage-comp-post,thm:coverage-comp} reduces to \Cref{thm:exact-cond-coverage}.
\end{remark}

\section{Confidence set for pre- and post-change parameters}
\label{sec:ci-pre-post-para}

Beyond localizing the changepoint, another natural inferential goal is to estimate the pre- and post-change parameters, assuming that the pre- and post-change distributions follow some parametric composite models: $\mathcal{P}_i=\{F_{\theta}:\theta\in\Theta_i\}; i=0,1$. While our primary focus is on constructing a confidence set $\mathcal{C}$ for the changepoint $T$, we now outline how one can derive confidence sets for $\theta_0$ and $\theta_1$ using $\mathcal{C}$.

The key idea is as follows: once a valid confidence set $\mathcal{C} \subseteq \{1, \cdots, \tau\}$ for $T$ is constructed, this implicitly identifies a subset of data that is likely to have been generated under the pre- and post-change regimes. 

For estimating the post-change parameter, one natural idea might be to consider the samples starting from $\max\mathcal{C}$ till the detection time $\tau$. But since both starting and ending times are random here, this strategy does not yield valid coverage guarantees.
Instead, we fix each $t\in\mathcal{C}$ and construct a confidence sequence based on $X_{t},\cdots,X_{\tau}$ for $\theta_1$, denoted as $\operatorname{CS}(X_{t},\cdots,X_{\tau};1-\eta_1 r_t^*)$ (recall the definition in \eqref{eq:cs-theta1}), for some $\eta_1\in(0,1)$, whose value is less than the target non-coverage, $\alpha^\prime$. Finally, the union of these CS across $t\in\mathcal{C}$, yields the post-detection confidence interval for $\theta_1$:
 \begin{equation}
 \label{eq:ci-theta1}
\displaystyle\cup_{t\in\mathcal{C}}\operatorname{CS}(X_{t},\cdots,X_{\tau};1-\eta_1 r_t^*).
 \end{equation}

There is a natural tension between how quickly $\mathcal A$ detects a change from $\theta_0$ and the tightness of inference on $\theta_1$ --- if $\mathcal A$ stops very quickly, we will have very few data points with which to estimate $\theta_1$, leading to wider intervals. This is not a drawback of our methodology, but an inherent limitation of any post-detection inference method.

Analogously, for the pre-change parameter $\theta_0$, we fix each $t\in\mathcal{C}$ and construct a confidence interval for $\theta_0$ based on the data ${X_1, \ldots, X_{t-1}}$, denoted as
$\operatorname{CI}(X_{1},\cdots,X_{t-1};1-\eta_0 r_t^*)$, for some $\eta_0\in(0,1)$, chosen to be smaller than the target (conditional) non-coverage level, $\alpha^{\prime\prime}$. The final confidence set for $\theta_0$ is $\displaystyle\cup_{t\in\mathcal{C}}\operatorname{CI}(X_{1},\cdots,X_{t-1};1-\eta_0 r_t^*).$

Next, we formalize its coverage guarantee, which is nonasymptotic and agnostic to the detection algorithm, unlike existing methods.

\begin{theorem}
 \label{thm:coverage-theta0-theta1}
Let $\mathcal{C}$ be a confidence set, e.g., \eqref{eq:ci-nonpara} or \eqref{eq:conf-comp}, such that for any $\theta_0\in\Theta_0,\theta_1\in\Theta_1,T\in \N$, change detection algorithm $\mathcal{A}$ satisfying \Cref{assmp-1} and  $\mathbb \P_{\theta_0,\infty}(\tau\geq T)\neq 0$, we have: $\P_{\theta_0,T,\theta_1}(T\in \mathcal{C}\mid \tau\geq T)\geq 1-\alpha$. Then, for any $T\in\N$,
\begin{align*}
    &\P_{\theta_0,T,\theta_1}\left(\theta_0\in \displaystyle\cup_{t\in\mathcal{C}}\operatorname{CI}(X_{1},\cdots,X_{t-1};1-\eta_0 r_t^*)\mid \tau\geq T\right)\geq 1-\alpha-\eta_0, \text{ and } \\
&~~~~\P_{\theta_0,T,\theta_1}\left(\theta_1\in \displaystyle\cup_{t\in\mathcal{C}}\operatorname{CS}(X_{t},\cdots,X_{\tau};1-\eta_1 r_t^*)\mid \tau\geq T\right)\geq 1-\alpha-\eta_1.
\end{align*}
 \end{theorem} 
 \section{Experiments} 
 We validate our methods across various simulation and real-data settings, demonstrating their practicality.

\label{sec:expt}
\textbf{Known pre- and post-change (Setting I):}
\label{setting-i}
Here, we draw $X_1,\cdots,X_{T-1}\stackrel{iid}{\sim}N(0,1)$ and $X_T,X_{T+1},\cdots\stackrel{iid}{\sim}N(1,1)$, with $T$ being the only unknown parameter. We employ the CUSUM detector,
$\tau=\inf\left\{n\in\mathbb N:\max_{1\leq j\leq n}\prod_{i=j}^n\frac{f_1(X_i)}{f_0(X_i)}\geq A\right\},$ with $A=1000$.
After detection, we construct both the proposed confidence sets (universal  \eqref{eq:ci-nonpara-known-pre} and adaptive \eqref{eq:ci-simple}), with $B=100, N=50$, $\alpha=0.1$.  \Cref{fig:ci-normal} visualizes the confidence sets \eqref{eq:ci-simple} across $5$ runs.  \Cref{tab:simple} has the (conditional on the event $\tau\geq T$) average size and coverages of the confidence sets and delay of detection, i.e, $\tau-T$, across $500$ runs.

 \begin{figure}[h!]
\centering
\centering
\subfloat[\scriptsize Data till change is detected, $T=100$]{\includegraphics[width=0.5\linewidth,height=0.32\linewidth]{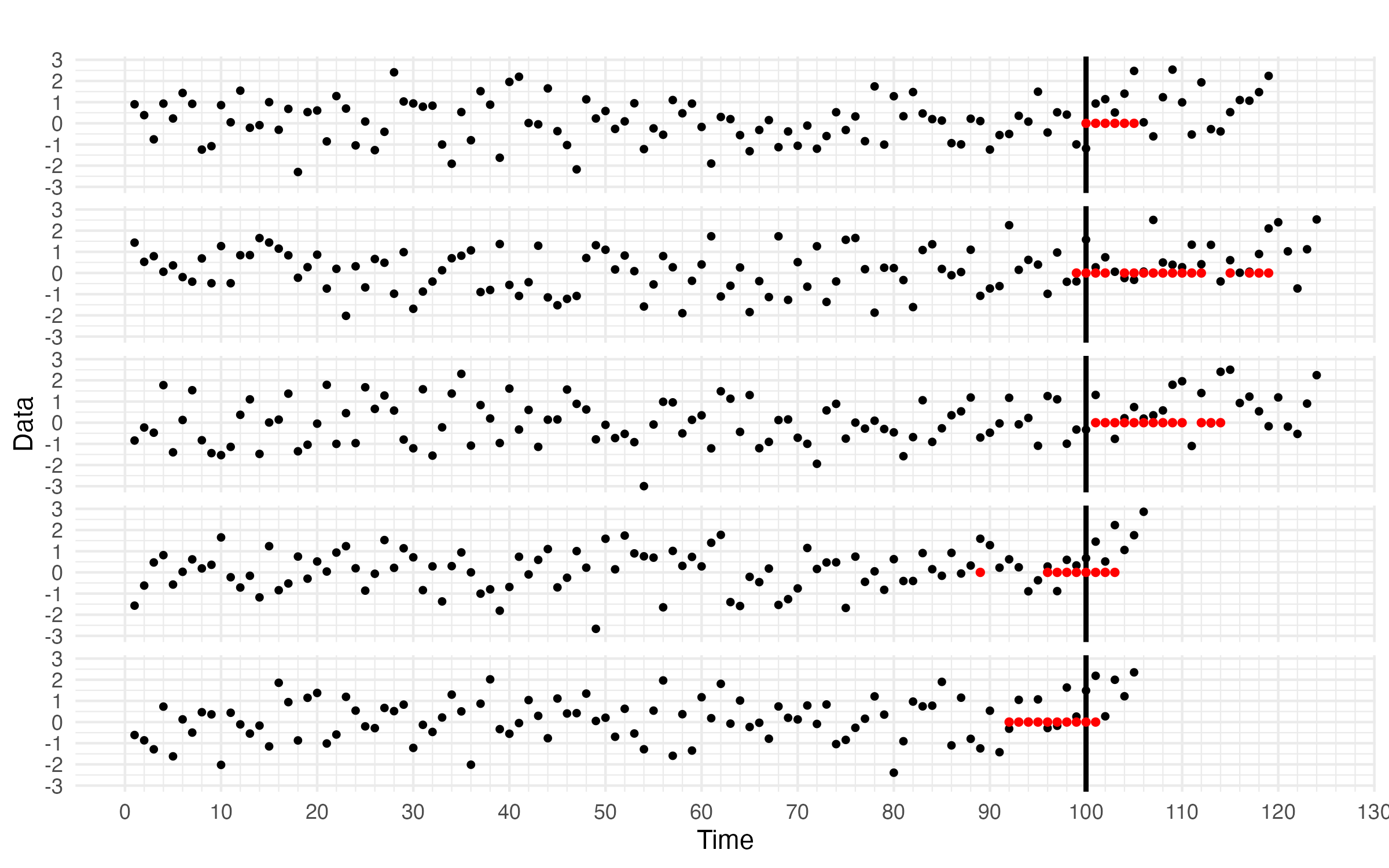}} 
\subfloat[\scriptsize Data till change is detected, $T=500$]{\includegraphics[width=0.5\linewidth,height=0.32\linewidth]{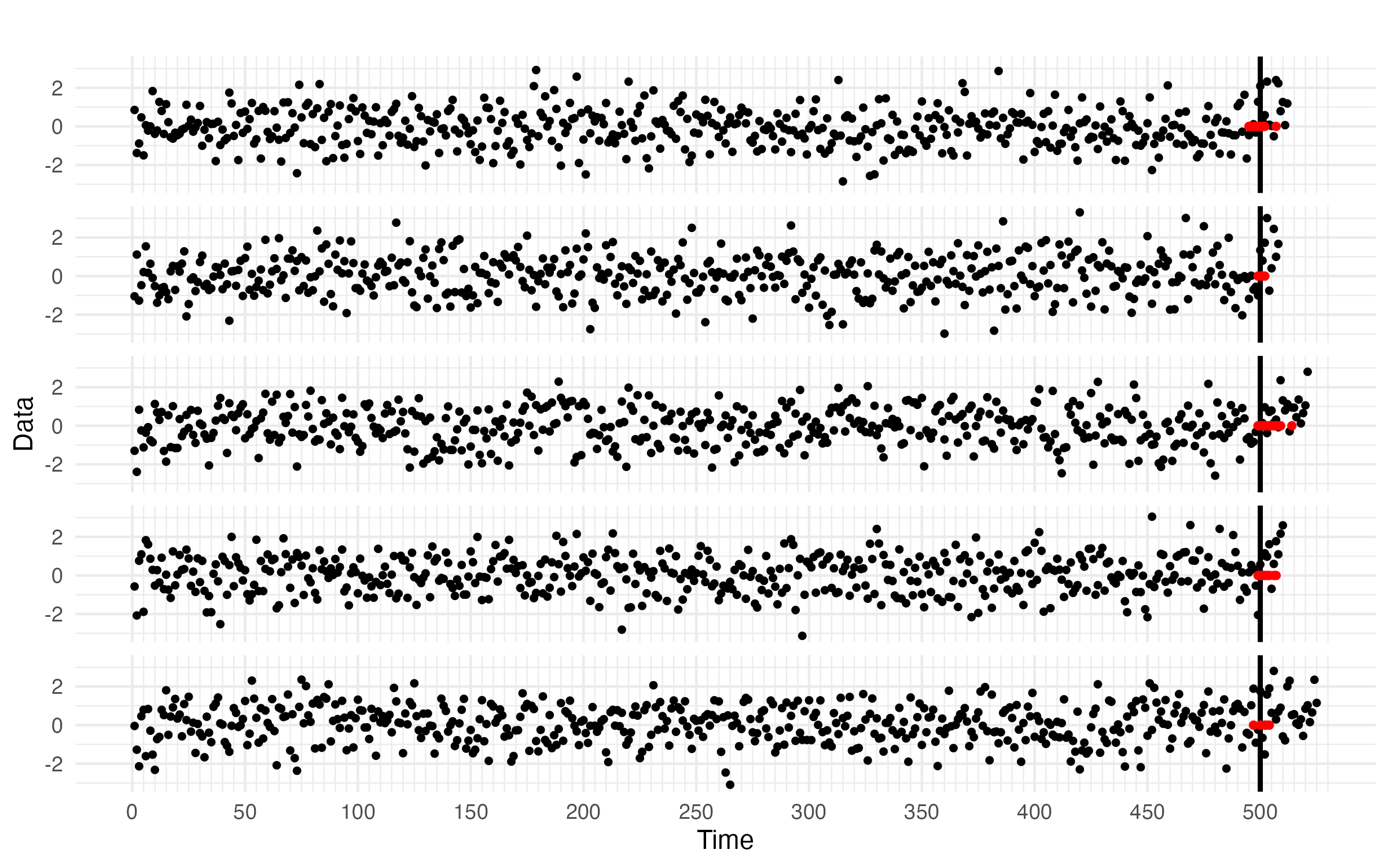}} 
\caption[]{ Setting I. The first $T-1$ observations are drawn from $N(0,1)$ and the rest from $N(1,1)$. The confidence sets (adaptive \eqref{eq:ci-simple}) are shown in red points, with $B=N=100$, $\alpha=0.1, L=\infty$. Results of $5$ independent simulations are shown.} 
\label{fig:ci-normal}  
\end{figure}
  \begin{table}[!ht]
    \centering
    \caption{Setting I: Known pre- and post-change.
  }
\label{tab:simple}
    \resizebox{0.75\linewidth}{!}{
    \begin{tabular}{cc|cccc|c}
    \toprule
    \addlinespace
$T$ & A &\multicolumn{2}{c}{Conditional Coverage} &
\multicolumn{2}{c}{Conditional Size} & \specialcell{Delay} \\
\cmidrule(lr){3-4} \cmidrule(lr){5-6}  &
&  Universal &  Adaptive &   Universal &  Adaptive & (Conditional)  \\
\midrule
\addlinespace 
100 & 1000&   0.948 &  0.906  & 12.412  & \textbf{9.694}  &  12.803 \\
500 & 1000  & 0.942  & 0.908 & 12.176& \textbf{9.278} & 13.227\\
\bottomrule
 \end{tabular}}
 \end{table}

\textbf{Known pre- and parametric post-change (Setting II):}
\label{setting-ii}
We perform experiments with $F_0=N(0,1)$ (known pre-change), post-change class $\mathcal{P}_1=\{N(\theta,1):\theta\in\Theta_1\}$, ; we vary $\Theta_1$ and $T$. The post-change samples are drawn from $N(1,1)$ (unknown).
For detection, we use the CUSUM detector
 with $A=1000$, a discrete weight distribution taking values
$\theta_1=\theta,\theta_2=\theta+d,\theta_3=\theta+2d,\cdots,\theta_{10}=\theta+9d$, $d=0.2$, $\theta=0.75$, when $\mathcal{P}_1=\{N(\mu,1):\mu> \theta\}$, with exponentially decaying weights $w_i=e^{-\frac{i-1}{2}}-e^{-\frac{i}{2}},\text{ for } i=1,\cdots,9, \text{ and } w_{10}=e^{-9/2}$, which appeared to us to be a sensible discretized mixture over $\mathcal{P}_1$. 
After detection, we construct both the proposed confidence sets for $T$ (universal  \eqref{eq:ci-nonpara-known-pre} with $\alpha=0.075$ and adaptive \eqref{eq:conf-comp-post}\footnote{Implementation details of the adaptive method in Setting II (and Setting III) are the same as discussed in Section C.1 (and Section D.1) of the Supplementary.} with $\alpha=0.1,\beta=0.025$) with target conditional coverage $0.925$.
For implementing \eqref{eq:ci-nonpara-known-pre} , we use  $R_n^{(t)}=\prod_{i=t}^{n} \frac{f_{\theta_0}(X_i)}{f_{\theta_1^*}(X_i)},$ where ${\theta_1^*}$ is the closest element of $\Theta_0$ from $\Theta_1$, and $
   S_n^{(t)}= \int_{\theta\in\Theta_1}\prod_{i=n}^{t-1} \frac{f_\theta(X_i)}{f_{\theta_0}(X_i)}w(\theta)d\theta$, with $w$ being the same weight mentioned above.
In \Cref{tab:comp-post}, we report the results averaged across $500$ runs.
 \Cref{fig:ci-comp}(a)  shows the confidence sets  \eqref{eq:ci-nonpara-known-pre} across $5$ random runs.

\begin{table}[!ht]
    \centering
    \caption{Known pre-change but unknown post-change. The pre-change parameter $\theta_0=0$ is known, the true (unknown) post-change parameter is $\theta_1=1$, while the algorithm only knows $\theta_1\in\Theta_1$, which we mention in the second column. (Setting II)
    }
\label{tab:comp-post}
    \resizebox{0.8\linewidth}{!}{
    \begin{tabular}{cc|cccc|c}
    \toprule
    \addlinespace
$T$ & $\Theta_1$ &\multicolumn{2}{c}{Conditional Coverage} &
\multicolumn{2}{c}{Conditional Size} &  \specialcell{Delay} \\
\cmidrule(lr){3-4} \cmidrule(lr){5-6}  &
&  Universal &  Adaptive &   Universal &  Adaptive &(Conditional) \\
\midrule
\addlinespace 
100 & $[0.75,\infty)$   & 0.972  & 0.966 & 16.870 & 14.824   & 13.878\\
100 &$[0.9,\infty)$  & 0.976  & 0.972 & 12.860 & 11.552  & 12.212\\
500 & $[0.75,\infty)$ &  0.969  & 0.968 & 16.168  &  14.266 & 13.638\\
500 &  $[0.9,\infty)$ &  0.972  & 0.971 & 14.734  & 12.602 & 12.814\\
\bottomrule
 \end{tabular}}
 \end{table}

\textbf{Composite parametric pre- and post-change (Setting III):}
\label{setting-iii} 
Here, we draw $X_1,\cdots,X_{T-1}\stackrel{iid}{\sim}N(0,1)$ and $X_T,X_{T+1},\cdots\stackrel{iid}{\sim}N(1,1)$.
Now we conduct experiments for pre-change class $\mathcal{P}_0=\{N(\theta,1):\theta\in\Theta_0\}$, post-change class $\mathcal{P}_1=\{N(\theta,1):\theta\in\Theta_1\}$; we vary $\Theta_0,\Theta_1$ and $T$. 
For detection, we employ a weighted CUSUM-type detector
\begin{equation}
\label{eq:wcusum-ripr}
    \tau_{\text{wcs-ripr}}=\inf\left\{n\in\mathbb N:\max_{1\leq j\leq n}\int_{\theta\in\Theta_1}\prod_{i=j}^n\frac{f_{\theta}(X_i)}{f_{\theta^*_0}(X_i)}dW(\theta)\geq A\right\},
\end{equation}
with $A=1000$, ${\theta^*_0}$ is the ``closest'' element of $\Theta_0$ from $\Theta_1$, and for the post-change parameter, the same discrete weight distribution as in setting II is considered.
We define $\hat \theta_{0,1:j}$ to be the MLE of  $\theta_0\in \Theta_0$ based on data $X_1,\cdots,X_{j}$.
After detection, we construct both the proposed confidence sets (universal \eqref{eq:ci-nonpara} with $\alpha=0.1$ and adaptive \eqref{eq:conf-comp} with $\alpha=0.05, \beta=\gamma=0.025$) for $T$, with $N=B=100$, both having target conditional coverage $0.9$. For implementing \eqref{eq:ci-nonpara}, we use  $R_n^{(t)}=\int_{\theta\in\Theta_0}\prod_{i=t}^{n} \frac{f_{\theta}(X_i)}{f_{\theta_1^*}(X_i)}w^\prime(\theta)d\theta,$  with $w^\prime(\theta)$ being the discrete distribution taking values
$\theta_1=\theta,\theta_2=\theta-d,\theta_3=\theta-2d,\cdots,\theta_{10}=\theta-9d$, $d=0.2$, when $\mathcal{P}_1=\{N(\mu,1):\mu> \theta\}$, with exponentially decaying weights $w_i=e^{-\frac{i-1}{2}}-e^{-\frac{i}{2}},\text{ for } i=1,\cdots,9, \text{ and } w_{10}=e^{-9/2}$ and $
   S_n^{(t)}= \int_{\theta\in\Theta_1}\prod_{i=n}^{t-1} \frac{f_\theta(X_i)}{f_{\theta_0^*}(X_i)}w(\theta)d\theta$, with $w(\theta)$ being the same as mentioned in Setting II. $\theta_i^*$ is the closest element of $\Theta_i$ from the other space. \Cref{fig:ci-comp}(b) visualizes the confidence sets \eqref{eq:ci-nonpara} across 5 random runs. \Cref{tab:comp} reports the results averaged across $500$ runs. 

\begin{figure*}[!ht]
\centering
\centering
\subfloat[\scriptsize Setting II. $F_0=N(0,1)$, $\mathcal{P}_1=\{N(\mu,1):\mu\geq 0.9\}$]{\includegraphics[width=0.5\linewidth,height=0.32\linewidth]{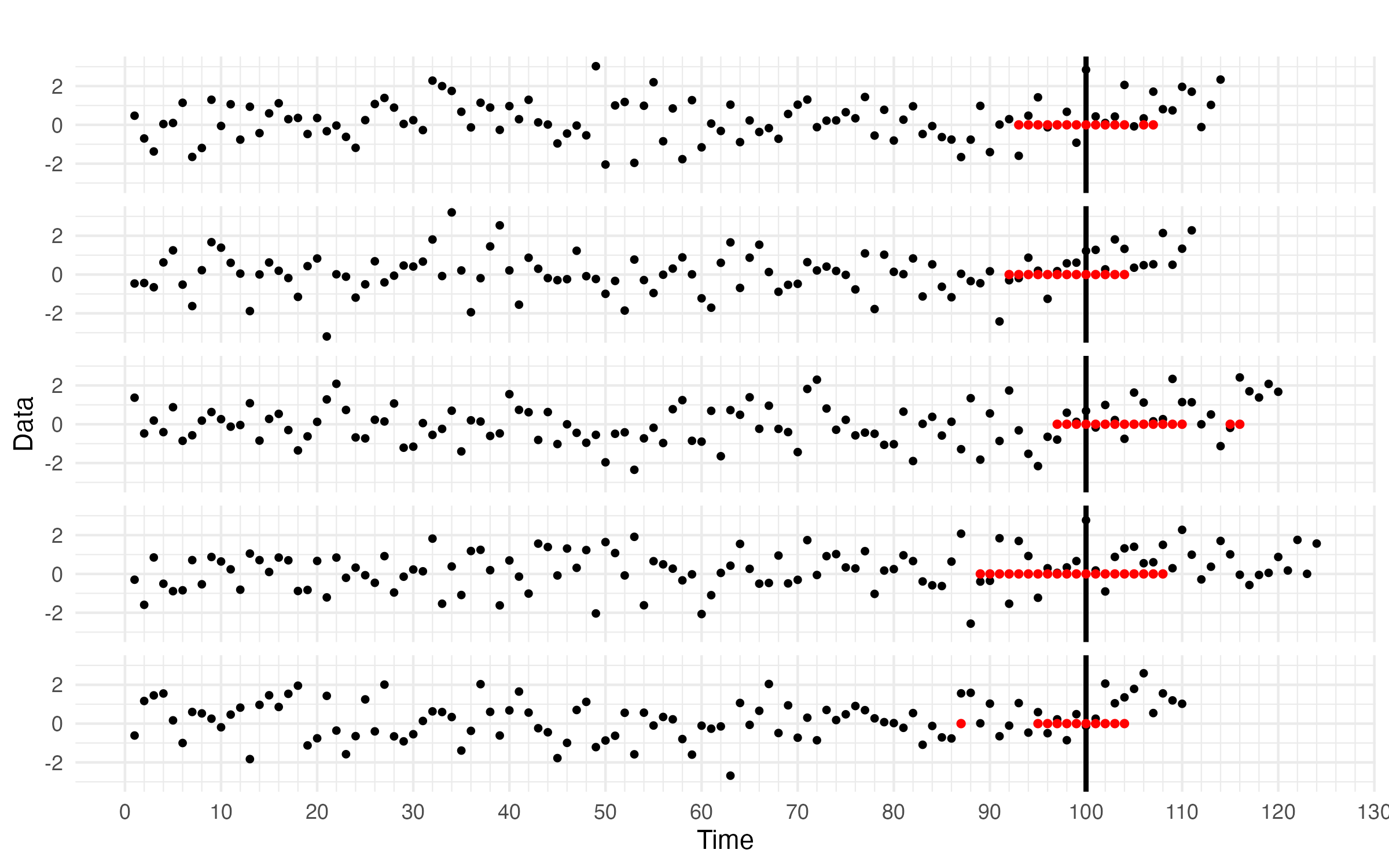}} 
\subfloat[\scriptsize Setting III. $\mathcal{P}_0=\{N(\mu,1):\mu\leq 0.1\}$, $\mathcal{P}_1=\{N(\mu,1):\mu> 0.9\}$]{\includegraphics[width=0.5\linewidth,height=0.32\linewidth]{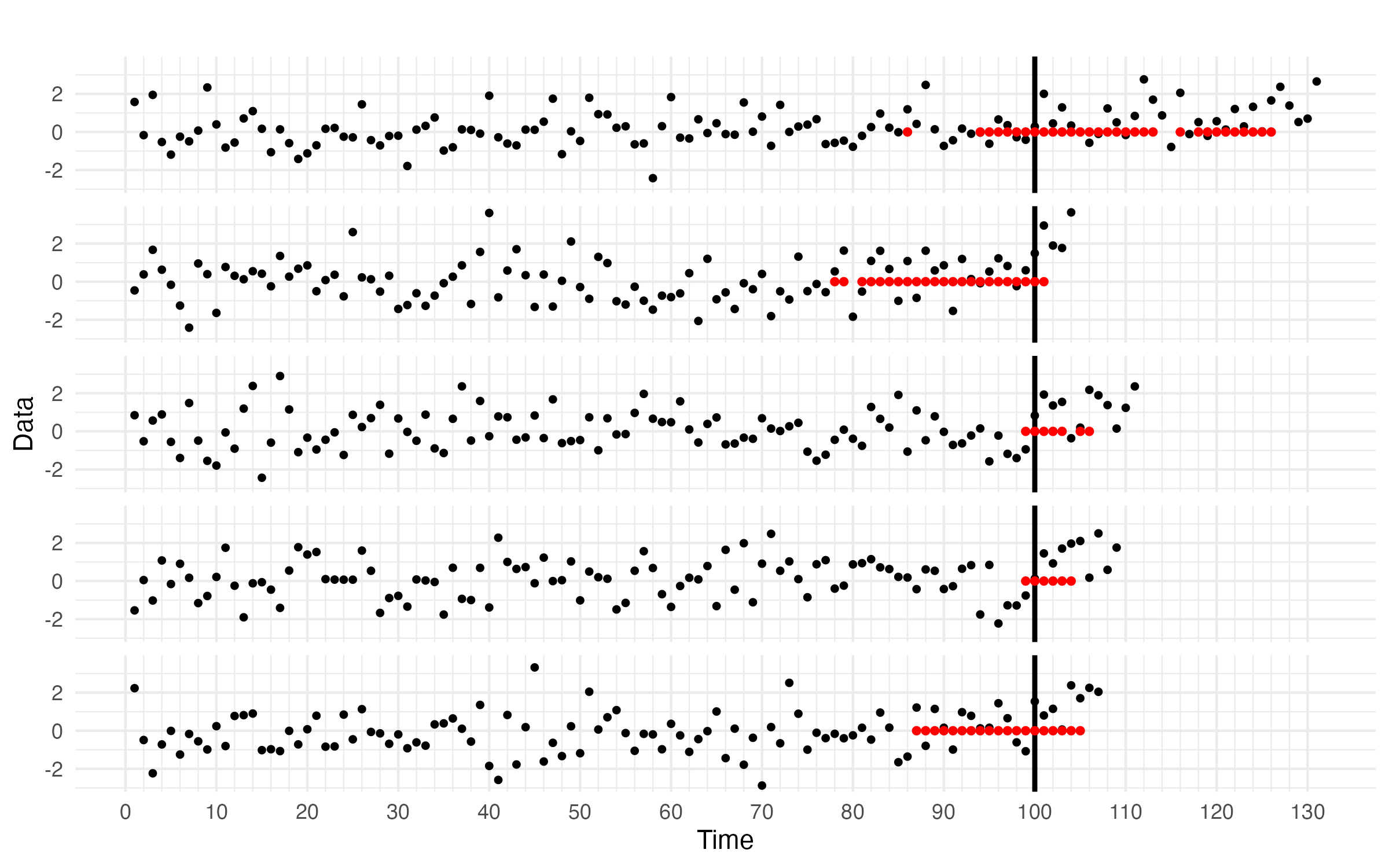}} 
\caption[]{  First $T-1$ samples are drawn from $N(0,1)$ and the remaining samples from $N(1,1), T=100$.
 Confidence sets (universal \eqref{eq:ci-nonpara}) are shown in red points. $N=100$, $\alpha=0.1$. 
 Results of $5$ random simulations are shown.
 } 
\label{fig:ci-comp}  
\end{figure*}

 \begin{table}[!ht]
    \centering
    \caption{Confidence interval for $T$ (Setting III): unknown pre- and post-change. The true (unknown) values of the parameters are $\theta_0=0$, $\theta_1=1$, while the algorithm only knows $\theta_i\in\Theta_i, i=0,1$, which are mentioned in the second and third columns. 
    }
\label{tab:comp}
   \resizebox{0.9\linewidth}{!}{
    \begin{tabular}{ccc|cccc|c}
    \toprule
    \addlinespace
$T$ & $\Theta_0$  &  $\Theta_1$ &\multicolumn{2}{c}{Conditional Coverage} &
\multicolumn{2}{c}{Conditional Size}  & \specialcell{Delay} \\
\cmidrule(lr){4-5} \cmidrule(lr){6-7}  &
& & Universal &  Adaptive &   Universal &  Adaptive & (Conditional)  \\
\midrule
\addlinespace 
100 &  $(\infty, 0.25]$ &  $[0.75,\infty)$   & 0.984 & 0.982 & 23.683   &   23.827 &  22.810\\
100 &  $(\infty,  0.1]$ & $[0.9,\infty)$ & 0.968 & 0.978 & 15.552 & 15.294 & 15.282\\
500 & $(\infty, 0.25]$ &  $[0.75,\infty)$ &   0.982 & 0.976 & {23.825} & 24.482 &  24.060\\
500 & $(\infty,  0.1]$ & $[0.9,\infty)$ &   0.970 & 0.942 & {16.860} & 16.956 & 15.576\\
\bottomrule
 \end{tabular}}
 \end{table}

\textbf{Composite parametric pre- and nonparametric post-change (Setting IV):}
Here, we consider the true (unknown) pre- and post-change distributions to be $\text{N}(1,1)$ and $\text{U}[-1.2,0.8]$ respectively. Let the pre-change class be 
$\mathcal{P}_0 = \left\{ \text{N}(\theta,1): \theta \geq \mu_1 \right\}, \mu_1=0.75$ and the post-change class be the set of all SubGaussians with a nonpositive mean:
$\mathcal{P}_1 = \left\{ F \in  \mathcal{M}: \mathbb{E}_F[e^{\lambda X}] \leq e^{-\frac{\lambda^2}{2}}, \forall \lambda\in[0,\infty)\right\}$, $\mathcal{M}$ is the set of all probability measures on $(\mathbb R,\mathcal{B}(\mathbb R))$.
We use an e-detector \citep{shin2022detectors}, $$\tau_e^\prime=\inf\left\{n\in\mathbb N:\max_{1\leq t\leq n} \prod_{i=t}^{n} e^{(\nu_i-\mu_1)X_i-\nu_i^2/2+\mu_1^2/2}\geq A\right\},$$ with $A=1000$,
$\nu_i=\min\{0,\frac{1}{i-1}\sum_{j=1}^{i-1}X_j\}.$ This $\mathcal{P}_0$ and $\tau_e^\prime$ satisfies \Cref{assmp-1} with $P_0^*=\text{N}(\mu_1,1)$. After detection, we construct the confidence sets \eqref{eq:ci-nonpara}.
The forward $t$-delay e-process we use is $\{R^{(t)}_n\}_{n\geq t}=\prod_{i=t}^ne^{\lambda_iX_i-\lambda_i^2/2}$, where $\lambda_i=\max\{\mu_1,\frac{1}{i-t}\sum_{j=t}^{i-1}X_i\}$ and the backward e-processes is constructed similarly: $\{S^{(t)}_n\}_{n < t}=\prod_{i=n}^te^{(\eta_i-\mu_1)X_i-\eta_i^2/2+\mu_1^2/2}$, where $\eta_i=\min\{0,\frac{1}{t-i}\sum_{j=i+1}^{t}X_j\}$ (see, e.g., Ch. 6.7.5 of \cite{ramdas2024hypothesis}).
 \Cref{tab:nonpara} contains the results across $500$ runs and \Cref{fig:ci-nonpara} visualizes across $5$ runs.
\begin{figure*}[!ht]
\centering
\centering
\subfloat[\scriptsize Setting IV. First $T-1$ samples are drawn from $N(1,1)$ and the remaining samples are from $U(-1.2,0.8)$.]{\includegraphics[width=0.5\linewidth,height=0.32\linewidth]{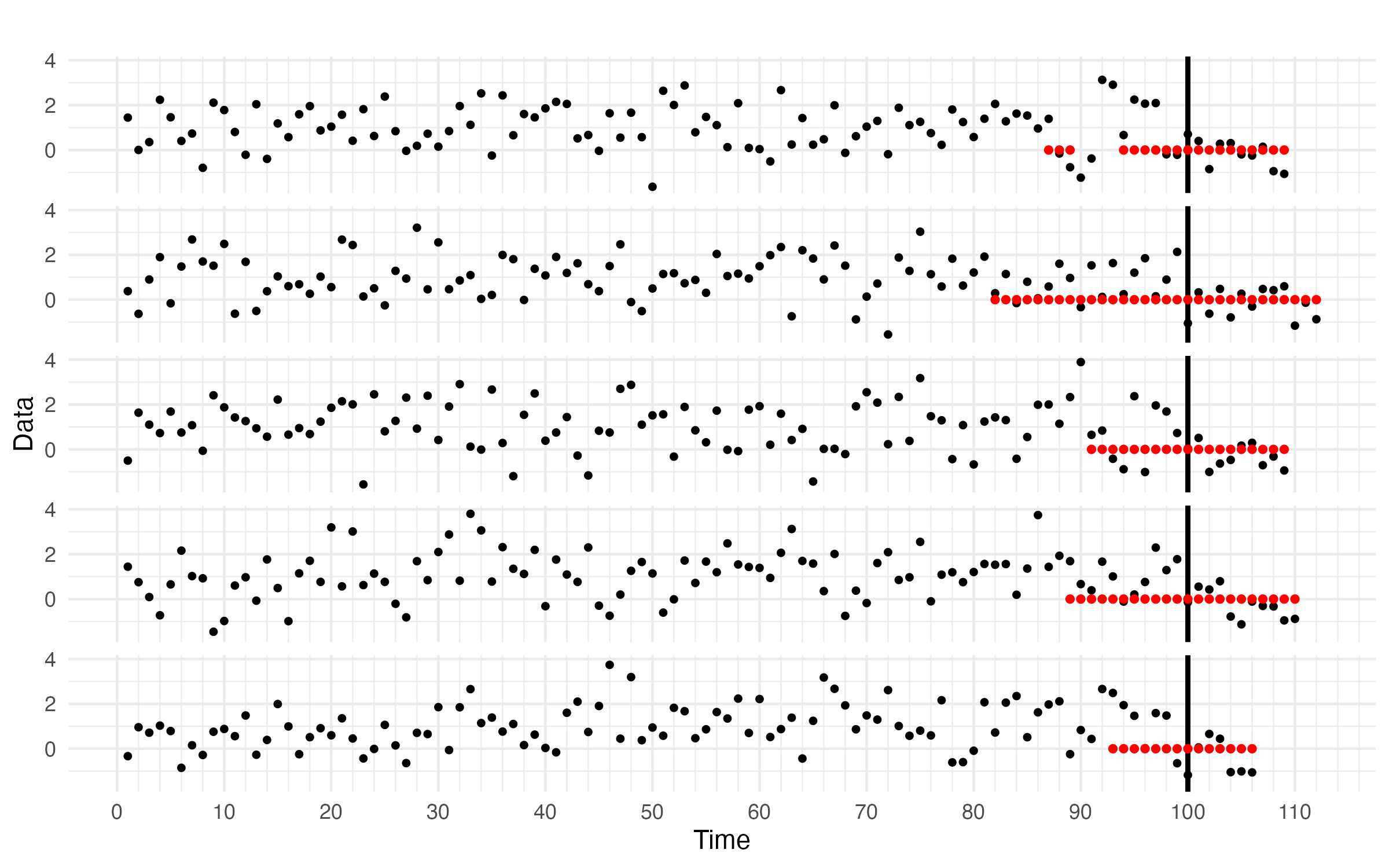}} 
\subfloat[\scriptsize Setting V. First $T-1$ samples are drawn from  $(1-\epsilon)\times N(0,1)+\epsilon \times \text{Cauchy}(-1,10)$ and the remaining samples are from $(1-\epsilon)\times N(1,1)+\epsilon\times  \text{Cauchy}(-1,10)$, $\epsilon=0.01$.]{\includegraphics[width=0.5\linewidth,height=0.32\linewidth]{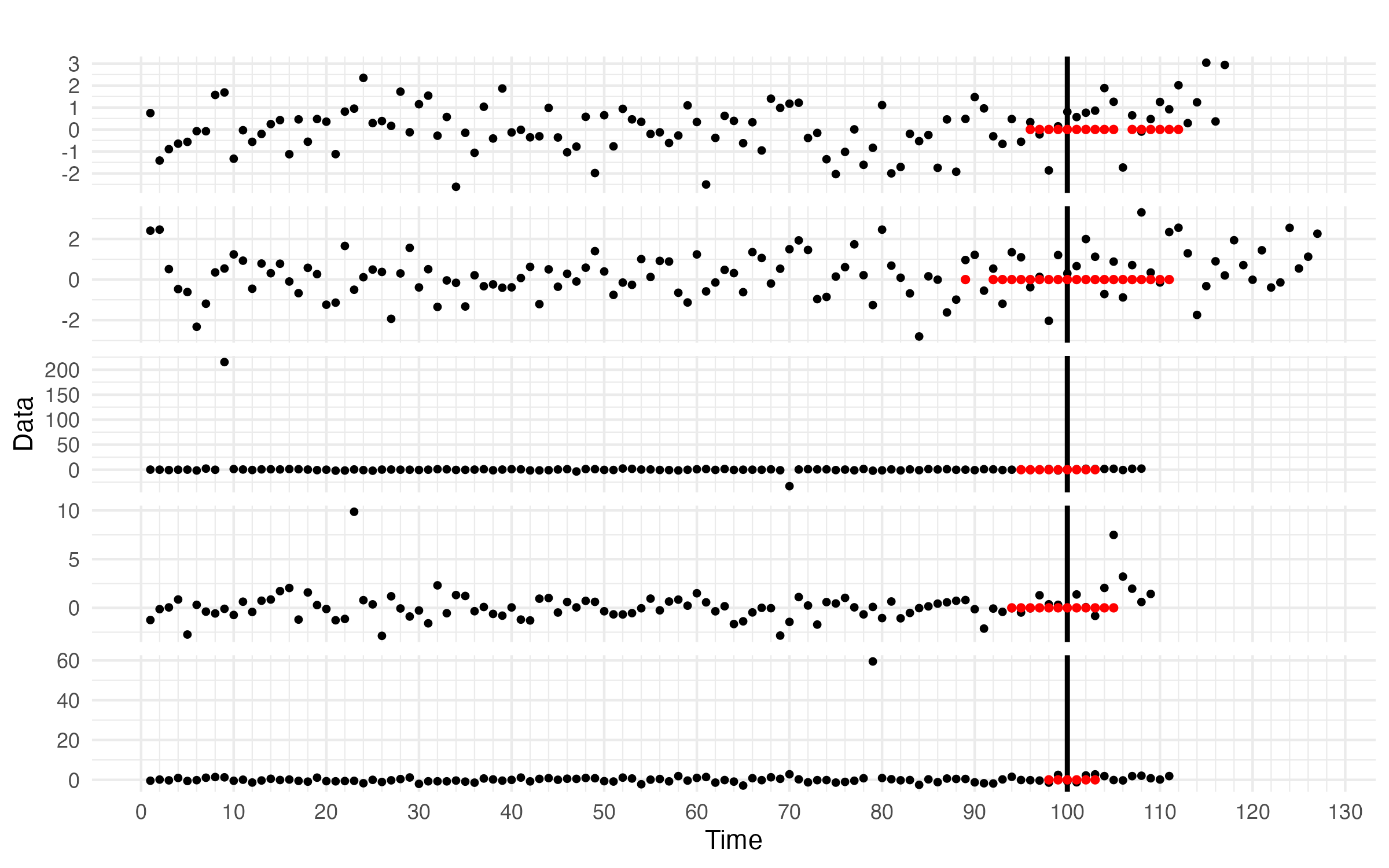}} 
\caption[]{ The confidence set \eqref{eq:ci-nonpara} is marked in red points. Results of $5$ independent simulations are shown.} 
\label{fig:ci-nonpara}  
\end{figure*}

\begin{table}[!ht]
    \centering
    \caption{Setting IV: Parametric pre-change, nonparametric post-change. 
    }
\label{tab:nonpara}
    \resizebox{0.7\linewidth}{!}{
    \begin{tabular}{c|ccc|c}
    \toprule
    \addlinespace
 T &  \specialcell{Coverage\\(Conditional)} & \specialcell{Coverage\\(Marginal)} & \specialcell{Size\\(Conditional)} &  \specialcell{Delay\\(Conditional)}  \\
    \midrule
\addlinespace 
100 &  0.984 & 0.982 &  23.859 &  9.959\\
500 &  0.992 & 0.930 & 25.381 &  9.606\\

\bottomrule
 \end{tabular}}
 \end{table}

 { 
\textbf{Huber-robust (nonparametric) pre- and post-change (Setting V):}

We perform experiments for the robust Gaussian mean-change problem under the $\epsilon$ contamination model, with the true changepoint at $ T=100,500$. The pre- and post-change classes are $\epsilon$ neighbourhoods around $N(\mu_0,1)$ and $N(\mu_1,1)$, respectively:  $\mathcal{P}_i=\{(1-\epsilon)N(\mu_i,1)+\epsilon F:F\in\mathcal{M}\};~~ i=0,1.$ In our experiment, $\mu_0=0,\mu_1=1$ and $\epsilon=0.01,0.001$. The true (unknown) pre- and post-change distributions are $(1-\epsilon)\times N(0,1)+\epsilon \times \text{Cauchy}(-1,10)$ and $(1-\epsilon)\times N(1,1)+\epsilon\times  \text{Cauchy}(-1,10)$ respectively. We use Huber's LFD (which are also log-optimal e-values \citep{saha2024huber}) to construct a CUSUM-style change-detector as below:
\begin{equation}
\label{eq:robust-detector}
    \tau=\inf\left\{t\in\N:\max_{1\leq j\leq t} \prod_{i=j}^t\pi(X_i)\geq A\right\},
\end{equation}
where $\pi(x)=\min\{\max\{\frac{f_1(x)}{f_0(x)},c^\prime\},c^{\prime\prime}\}$ for constants $c^{\prime\prime}>c^\prime\geq 0$ as mentioned in \cite{huber1965robust,saha2024huber}.
We set the threshold at $A=1000$.
After detection, we construct our universal confidence sets \eqref{eq:ci-nonpara} for $T$, with $R_n^{(t)}=1/\prod_{i=t}^n\pi(X_i)$, for $n\geq t$ $S_n^{(t)}=\prod_{i=n}^{t-1}\pi(X_i)$, for $n<t$ and $\alpha=0.1$. We report the results averaged across $500$ independent runs in \Cref{tab:comp-robust}, and \Cref{fig:ci-nonpara} visualizes the results across $5$ runs. 

 \begin{table}[!ht]
    \centering
    \caption{Setting V: Huber-robust (nonparametric) pre- and post-change 
    }
\label{tab:comp-robust}
    \resizebox{0.75\linewidth}{!}{
    \begin{tabular}{cc|ccc}
    \toprule
    \addlinespace
 T & $\epsilon $  & \specialcell{Conditional coverage} & Conditional size & \specialcell{Conditional delay}  \\
    \midrule
\addlinespace 
100 &  0.01 & 0.939 & 13.298& 14.140\\
500 &  0.01 & 0.948 & 13.634 & 13.943\\
100 &  0.001 & 0.944& 12.274& 12.638\\
500 &  0.001 & 0.958& 11.927 &12.496\\
\bottomrule
 \end{tabular}}
 \end{table}
}

\textbf{Comparison with \cite{wu2007inference}:}
\label{sec:comparison-wu}
We now compare our method with Ch. 3.5 of \cite{wu2007inference} on the Gaussian mean-change, having the (unknown) changepoint at $T=100$ or $500$. We recall that, unlike our general nonasymptotic guarantees for any $\mathcal A$, Wu's method only has asymptotic approximate guarantees in specific parametric settings, and only for CUSUM procedure~\citep{page1954continuous}, $ \tau^\prime_{\text{cusum}}:=\inf\left\{n\in\mathbb N:T_n\geq d\right\},$
where $T_n=\max\{0,T_{n-1}+X_n\}$ and $d$ is the threshold. We used $d=8.59$ and $7.56$ for ${N}(-0.25,1)$ vs. ${N}(0.25,1)$ and ${N}(-0.3,1)$ vs. ${N}(0.3,1)$ respectively. We implement  Wu's confidence set of the form $V_c\cup[L_s,\hat \nu)$, where $\hat \nu=\max\{k<\tau:T_k=0\}$, $V_c=\{k\geq \hat \nu: T_k\leq c\}$, and $L_k$ denote the $k$-th last zero point of $T_n$ counted backward starting from $L_0:=\hat\nu$, with their proposed approximations (for the case when $\theta_1=-\theta_0$) $s=\log(\alpha)\times(1/(\sqrt{2}\theta_1)+0.088)$ and $c=-\log(1-\sqrt{1-\alpha})/(2\theta_1)-0.583$ with $\alpha=0.1$. \Cref{tab:ci-comparison} reports the results (for target conditional coverage $0.9$) averaged across $500$ runs. 

\begin{table}[ht]
\centering
\caption{Comparison of the confidence sets for $T$ with  \cite{wu2007inference} (Setting I)}
\label{tab:ci-comparison}
\resizebox{\linewidth}{!}{
\begin{tabular}{ccc|
                cc
                cc
                cc
                cc}
\toprule
$F_0$ & $F_1$ & $T$ &
\multicolumn{2}{c}{Conditional Coverage} &
\multicolumn{2}{c}{Conditional Size} &
\multicolumn{2}{c}{Marginal Coverage} \\
\cmidrule(lr){4-5} \cmidrule(lr){6-7} \cmidrule(lr){8-9} &
& &  Ours &  Wu (2007) &  Ours &  Wu (2007) &  Ours &  Wu (2007) \\
\midrule
${N}(-0.25,1)$ & ${N}(0.25,1)$ & 100 & 0.916 &   0.838 & 34.595 & 39.706 & 0.864 & 0.792\\
${N}(-0.3,1)$ & ${N}(0.3,1)$ & 100 & 0.896 & 0.774 & 23.827 & 26.678 & 0.828 & 0.736\\
${N}(-0.25,1)$ & ${N}(0.25,1)$ & 500  & 0.902 &   0.841 & 34.382 & 38.195 & 0.546 &  0.578\\
${N}(-0.3,1)$ & ${N}(0.3,1)$ & 500 &  0.906 & 0.796 & 23.188 & 27.343 & 0.512 & 0.442\\
\bottomrule
\end{tabular}}
\end{table}

The results show that \cite{wu2007inference} (which is only asymptotically valid and relies on approximations) exhibits noticeable undercoverage, while the sizes for both methods are comparable. However, our method \eqref{eq:ci-simple} achieves the target conditional coverage.

{ 
\textbf{Sentiment change in language data:}
We consider the Stanford
Sentiment Treebank (SST-2) dataset of movie reviews labeled with a binary sentiment \citep{socher2013recursive}. We want to localize the change from  $>60\%$ positive reviews to $<40\%$ positive reviews. In our experimental setup, the true changepoint is at 500; the pre-change data are i.i.d. draws of $75\%$ positive reviews, and the post-change data are i.i.d. draws of
$25\%$ negative reviews (these are all unknown to our algorithm).  We construct a change detection method using a pre-trained DistilBERT-base model fine-tuned on the uncased SST-2 \citep{sanh2019distilbert}, which is a binary sentiment classifier that maps each review to either 
$0$ (negative) or 
$1$ (positive), thereby producing a sequence of Bernoulli observations. After obtaining the binary sequence, we construct the test statistic \eqref{eq:test-stat-nonpara}
 using the likelihood ratio process for $\text{Bern}(0.4)$ against $\text{Bern}(0.6)$ (forward) and for $\text{Bern}(0.6)$ against $\text{Bern}(0.4)$ (backward). Note that given the change detector (and as a byproduct, the sequence of Bernoulli), it is simple to compute  \eqref{eq:test-stat-nonpara} in $O(T)$ time per $t$. So, the confidence set \eqref{eq:ci-nonpara} takes total $O(T^2)$ time. \Cref{fig:sentiment-change} provides a visualization.
}
 \begin{figure}
     \centering     \includegraphics[width=0.95\linewidth,height=0.43\linewidth]{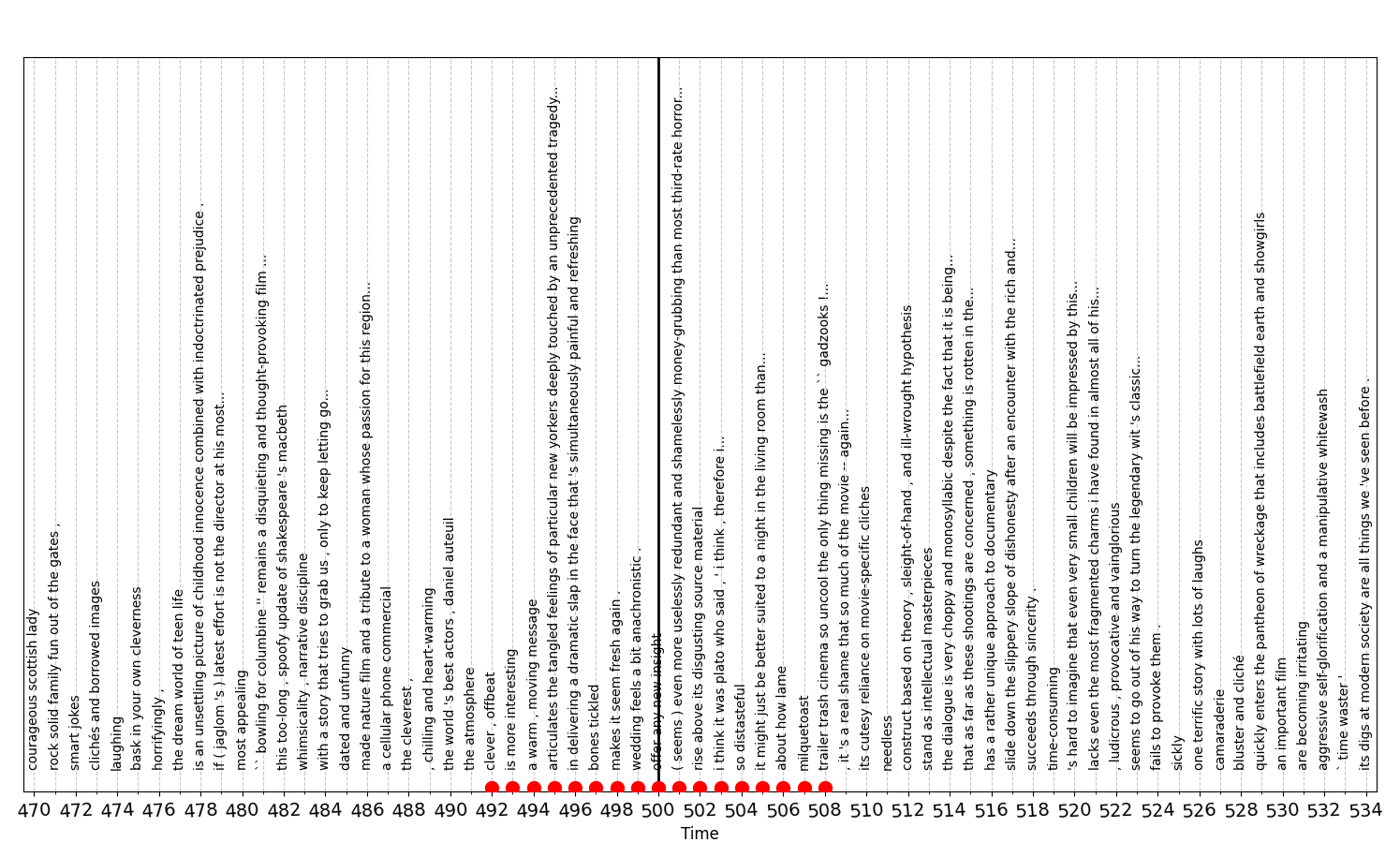}
     \caption{ The confidence set \eqref{eq:ci-nonpara} is marked in red points. The true changepoint $T=500$ is marked with the black line. Texts are plotted from time $t=470$ upto detection time, $\tau=534$.}
     \label{fig:sentiment-change}
 \end{figure}

 \begin{remark}
In Settings I, II, and III, both of our proposed methods --- universal (\Cref{sec:general}) and adaptive (\Cref{sec:parametric}) --- are applicable. For nonparametric settings (IV and V), the only valid method currently known is the universal method. Prior works are valid (asymptotically) only for Setting I.
As shown in \Cref{tab:simple} for known pre-distributions (Settings I and II), the adaptive method consistently outperforms the universal method, which tends to be conservative. Also, our adaptive method is the only method we know that remains valid under dependence (see experimental results in Supplementary Section G.2). However, for setting III, the universal and adaptive methods perform similarly.
Moreover, the universal method is also significantly more efficient --- it runs in $O(T^2)$ time (assuming $M_t$ can be computed in $O(T)$ time), while the adaptive method incurs a higher computational cost of $O(T^3 B)$ (assuming $\mathcal{A}$ runs in $O(T)$ time and $M_t$ can be computed in $O(T)$ time). 
See Supplement Sections E, F, and G for further experiments, involving confidence intervals for pre- and post-change parameters, settings with dependent data, settings with detection algorithms controlling PFA, etc.
\end{remark}

\section{Conclusion}
\label{sec:conc}

We developed a general framework for constructing confidence sets for the unknown changepoint $T$ after a sequential detector signals a change at a data-dependent stopping time. It applies to any general nonparametric post-change class when the pre-change distribution is known, and we further extend it to settings with a composite pre-change class under a suitable assumption. For the parametric setting, we proposed an alternative simulation-based method that empirically yields tighter confidence sets when both pre- and post-change distributions are known and can handle dependent data as well. However, we note that our primary (universal) method is nonparametric, computationally more efficient, and simpler to implement. 
Crucially, all our methods are agnostic to the detection algorithm used and provide non-asymptotic (conditional) coverage guarantees, whereas prior works are tailored for exponential family distributions with the CUSUM procedure and provide only asymptotic (conditional) coverage. By bridging a longstanding gap in the literature, our work offers the first nonparametric, theoretically principled, and broadly applicable method for post-detection inference. A particularly fruitful direction for future work is to extend this framework to the distribution-free setting.

\section*{Acknowledgement} 
Part of this work was carried out while AS was a student at the Indian Statistical Institute. The authors thank the Editor, the Associate Editor, and the reviewers for their careful reading of an earlier version of the article and for their helpful comments. The authors also thank Swapnaneel Bhattacharyya for identifying and helping eliminate an inconsistency in an earlier version of the article. AR acknowledges support from NSF DMS-1916320.

\bibliographystyle{plain}  
\bibliography{ref}

@book{tartakovsky2014sequential,
  title={Sequential Analysis: Hypothesis Testing and Changepoint Detection},
  author={Tartakovsky, A. and Nikiforov, I. and Basseville, M.},
  isbn={9781439838211},
  lccn={2014022603},
  year={2014},
  publisher={CRC Press}
}

@article{larsson2024numeraire,
  title={The numeraire e-variable and reverse information projection},
  author={Larsson, Martin and Ramdas, Aaditya and Ruf, Johannes},
  journal={Annals of Stat.},
  year={2025}
}

@article{lardy2023universal,
  author={Lardy, Tyron and Grünwald, Peter and Harremoës, Peter},
  journal={IEEE Transactions on Information Theory}, 
  title={Reverse Information Projections and Optimal E-Statistics}, 
  year={2024},
  volume={70},
  number={11},
  pages={7616-7631},
}

@article{hinkley1970inferencce,
    author = {Hinkley, David V.},
    title = "{Inference about the change-point in a sequence of random variables}",
    journal = {Biometrika},
    volume = {57},
    number = {1},
    pages = {1-17},
    year = {1970}
}

@article{worsley1986confidence,
 author = {K. J. Worsley},
 journal = {Biometrika},
 number = {1},
 pages = {91--104},
 title = {Confidence Regions and Tests for a Change-Point in a Sequence of Exponential Family Random Variables},
 volume = {73},
 year = {1986}
}

@article{verzelen2023optimal,
  title={Optimal change-point detection and localization},
  author={Verzelen, Nicolas and Fromont, Magalie and Lerasle, Matthieu and Reynaud-Bouret, Patricia},
  journal={The Annals of Statistics},
  volume={51},
  number={4},
  pages={1586--1610},
  year={2023}
}

@article{jang2024fast,
  title={Fast and Optimal Changepoint Detection and Localization using {B}onferroni Triplets},
  author={Jang, Jayoon and Walther, Guenther},
  journal={arXiv preprint arXiv:2410.14866},
  year={2024}
}

@article{dumbgen1991asymptotic,
  title={The asymptotic behavior of some nonparametric change-point estimators},
  author={Dumbgen, L},
  journal={The Annals of Statistics},
  year={1991}
}

@article{shewhart1925application,
  title={The application of statistics as an aid in maintaining quality of a manufactured product},
  author={Shewhart, Walter A},
  journal={Journal of the American Statistical Association},
  volume={20},
  number={152},
  pages={546--548},
  year={1925},
  publisher={Taylor \& Francis}
}

@article{page1954continuous,
  title={Continuous inspection schemes},
  author={Page, Ewan S},
  journal={Biometrika},
  volume={41},
  number={1/2},
  pages={100--115},
  year={1954},
  publisher={JSTOR}
}

@article{shiryaev1963optimum,
  title={On optimum methods in quickest detection problems},
  author={Shiryaev, Albert N},
  journal={Theory of Probability \& Its Applications},
  volume={8},
  number={1},
  year={1963}
}

@article{roberts1966comparison,
  title={A comparison of some control chart procedures},
  author={Roberts, SW},
  journal={Technometrics},
  volume={8},
  number={3},
  pages={411--430},
  year={1966},
  publisher={Taylor \& Francis}
}

@article{lorden1971procedures,
  title={Procedures for reacting to a change in distribution},
  author={Lorden, Gary},
  journal={The Annals of Mathematical Statistics},
  year={1971}
}

@article{siegmund1995using,
  title={Using the generalized likelihood ratio statistic for sequential detection of a change-point},
  author={Siegmund, David and Venkatraman, ES},
  journal={The Annals of Statistics},
  pages={255--271},
  year={1995},
  publisher={JSTOR}
}

@article{shin2022detectors,
  author = {Jaehyeok Shin and Aaditya Ramdas and Alessandro Rinaldo},
    title = {E-detectors: A Nonparametric Framework for Sequential Change Detection},
    journal = {The New England Journal of Statistics in Data Science},
    volume = {2},
    number = {2},
    year = {2023},
    pages = {229--260},
    publisher = {New England Statistical Society}
}

@article{harchaoui2008kernel,
  title={Kernel change-point analysis},
  author={Harchaoui, Zaid and Moulines, Eric and Bach, Francis},
  journal={Neural Information Proc. Systems},
  volume={21},
  year={2008}
}

@article{darkhovskh1976a,
author = {Darkhovskh, B. S.},
title = {A Nonparametric Method for the a Posteriori Detection of the “Disorder” Time of a Sequence of Independent Random Variables},
journal = {Theory of Probability \& Its Applications},
volume = {21},
number = {1},
pages = {178-183},
year = {1976}
}

@article{darling1967confidence,
  title={Confidence sequences for mean, variance, and median},
  author={Darling, Donald A and Robbins, Herbert},
  journal={Proceedings of the National Academy of Sciences},
  volume={58},
  number={1},
  pages={66--68},
  year={1967},
  publisher={National Acad Sciences}
}

@article{howard2021time,
  title={Time-uniform, nonparametric, nonasymptotic confidence sequences},
  author={Howard, Steven R and Ramdas, Aaditya and McAuliffe, Jon and Sekhon, Jasjeet},
  journal={The Annals of Statistics},
  volume={49},
  number={2},
  pages={1055--1080},
  year={2021}
}

@article{ville1939etude,
  title={Etude critique de la notion de collectif},
  author={Ville, Jean},
  journal={Bull. Amer. Math. Soc},
  volume={45},
  number={11},
  pages={824},
  year={1939}
}

@article{ramdas2022game,
  title={Game-theoretic statistics and safe anytime-valid inference},
  author={Ramdas, Aaditya and Gr{\"u}nwald, Peter and Vovk, Vladimir and Shafer, Glenn},
  journal={Statistical Science},
  year={2023}
}

@article{wu2005inference,
author = { Wu, Yanhong},
title = {Inference for Change-Point and Post-Change Mean with Possible Change in Variance},
journal = {Sequential Analysis},
volume = {24},
number = {3},
pages = {279--302},
year = {2005},
publisher = {Taylor \& Francis}
}

@book{wu2007inference,
  title={Inference for Change Point and Post Change Means After a {CUSUM} Test},
  author={Wu, Y.},
  isbn={9780387262697},
  lccn={2004058917},
  series={Lecture Notes in Statistics},
  year={2007},
  publisher={Springer New York}
}

@article{ding2003lower,
title = {A lower confidence bound for the change point after a sequential {CUSUM} test},
journal = {J. Statist. Planng Inf.},
volume = {115},
number = {1},
pages = {311-326},
year = {2003},
author = {Keyue Ding}
}

@article{Brodsky30042010,
author = {Boris Brodsky},
title = {Sequential Detection and Estimation of Change-Points},
journal = {Sequential Analysis},
volume = {29},
number = {2},
pages = {217--233},
year = {2010},
publisher = {Taylor \& Francis}
}

@article{Srivastava01011999,
author = {M.S Srivastava and Yanhong Wu},
title = {Quasi-stationary biases of change point and change magnitude estimation after sequential {CUSUM} test},
journal = {Sequential Analysis},
volume = {18},
number = {3-4},
pages = {203--216},
year = {1999},
publisher = {Taylor \& Francis}  
}

@article{Gombay10012003,
author = {Edit Gombay},
title = {Sequential Change-Point Detection and Estimation},
journal = {Sequential Analysis},
volume = {22},
number = {3},
pages = {203--222},
year = {2003},
publisher = {Taylor \& Francis}
}

@article{wu2004bias,
author = {Yanhong Wu},
title = {Bias of estimator of change point detected by a {CUSUM} procedure},
journal = {Annals of the Institute of Statistical Mathematics
},
volume = {56},
number = {1},
pages = {127--142},
year = {2004},
publisher = {Taylor \& Francis}
}

@book{shaked2007stochastic,
  title={Stochastic orders},
  author={Shaked, Moshe and Shanthikumar, J George},
  year={2007},
  publisher={Springer}
}

@article{WU20063625,
title = {Inference for post-change mean by a {CUSUM} procedure},
journal = {J. Statist. Planng Inf.},
volume = {136},
number = {10},
pages = {3625-3646},
year = {2006},
author = {Yanhong Wu}
}

@article{grunwald2024safe,
    author = {Grünwald, Peter and de Heide, Rianne and Koolen, Wouter},
    title = {Safe testing},
    journal = {J. R. Statist. Soc. B},
    volume = {86},
    number = {5},
    pages = {1091-1128},
    year = {2024},
    month = {03}
}

@article{ramdas2024hypothesis,
  title={Hypothesis testing with e-values},
  author={Ramdas, Aaditya and Wang, Ruodu},
  journal={Foundations and Trends in Statistics},
  volume = {1},
  number = {1},
  year={2025}
}

@article{ramdas2022testing,
  title={Testing exchangeability: Fork-convexity, supermartingales and e-processes},
  author={Ramdas, Aaditya and Ruf, Johannes and Larsson, Martin and Koolen, Wouter M},
  journal={International Journal of Approximate Reasoning},
  volume={141},
  pages={83--109},
  year={2022},
  publisher={Elsevier}
}

@article{wasserman2020universal,
  title={Universal {I}nference},
  author={Wasserman, Larry and Ramdas, Aaditya and Balakrishnan, Sivaraman},
  journal={Proceedings of the National Academy of Sciences},
  volume={117},
  number={29},
  pages={16880--16890},
  year={2020},
  publisher={National Acad Sciences}
}

@article{waudby2023estimating,
  title={Estimating means of bounded random variables by betting},
  author={Waudby-Smith, Ian and Ramdas, Aaditya},
  journal={Journal of the Royal Statistical Society Series B (Methodology), with discussion},
  year={2023}
}

@article{saha2024huber,
  author={Saha, Aytijhya and Ramdas, Aaditya},
  journal={IEEE Transactions on Information Theory}, 
  title={Huber-Robust Likelihood Ratio Tests for Composite Nulls and Alternatives}, 
  year={2026},
  volume={72},
  number={1},
  pages={501-520}
}

@inproceedings{saha2024testing,
  title={Testing exchangeability by pairwise betting},
  author={Saha, Aytijhya and Ramdas, Aaditya},
  booktitle={International Conference on Artificial Intelligence and Statistics},
  pages={4915--4923},
  year={2024},
  organization={PMLR}
}

@article{howard2020time,
  title={Time-uniform {C}hernoff bounds via nonnegative supermartingales},
  author={Howard, Steven R and Ramdas, Aaditya and McAuliffe, Jon and Sekhon, Jasjeet},
  journal={Probability Surveys},
  year={2020}
}

@article{huber1965robust,
  title={A robust version of the probability ratio test},
  author={Huber, Peter J},
  journal={The Annals of Mathematical Statistics},
  pages={1753--1758},
  year={1965},
  publisher={JSTOR}
}

@inproceedings{socher2013recursive,
  title={Recursive deep models for semantic compositionality over a sentiment treebank},
  author={Socher, Richard and Perelygin, Alex and Wu, Jean and Chuang, Jason and Manning, Christopher D and Ng, Andrew Y and Potts, Christopher},
  booktitle={Proceedings of the 2013 conference on empirical methods in natural language processing},
  pages={1631--1642},
  year={2013}
}

@article{huber1973minimax,
  title={Minimax tests and the {N}eyman-{P}earson lemma for capacities},
  author={Huber, Peter J and Strassen, Volker},
  journal={The Annals of Statistics},
  pages={251--263},
  year={1973},
  publisher={JSTOR}
}

@article{sanh2019distilbert,
  title={Distil{BERT}, a distilled version of {BERT}: smaller, faster, cheaper and lighter},
  author={Sanh, Victor and Debut, Lysandre and Chaumond, Julien and Wolf, Thomas},
  journal={arXiv preprint arXiv:1910.01108},
  year={2019}
}

@article{chu1996monitoring,
  title={Monitoring structural change},
  author={Chu, Chia-Shang James and Stinchcombe, Maxwell and White, Halbert},
  journal={Econometrica: Journal of the Econometric Society},
  pages={1045--1065},
  year={1996},
  publisher={JSTOR}
}

@article{kirch2022sequential,
  title={Sequential change point tests based on {U}-statistics},
  author={Kirch, Claudia and Stoehr, Christina},
  journal={Scandinavian Journal of Statistics},
  volume={49},
  number={3},
  pages={1184--1214},
  year={2022},
  publisher={Wiley Online Library}
}

@article{aue2024state,
  title={The state of cumulative sum sequential changepoint testing 70 years after {P}age},
  author={Aue, Alexander and Kirch, Claudia},
  journal={Biometrika},
  volume={111},
  number={2},
  pages={367--391},
  year={2024},
  publisher={Oxford University Press}
}

\appendix
\newpage
\renewcommand{\thepage}{\arabic{page}} 

\renewcommand{\thesection}{\Alph{section}} 
\setcounter{section}{0}                     

\renewcommand{\theequation}{S\arabic{equation}} 
\setcounter{equation}{0}                       
\setcounter{footnote}{0}

\allowdisplaybreaks

\begin{center}
  \vspace*{0.5cm}
  {{\LARGE\bfseries Supplementary Material 
  }}
\end{center}
\vspace{1cm}

\section*{Outline of the Supplementary Material}
Omitted proofs can be found in \Cref{a-proof}. We provide the implementation details of \Cref{algo:1-exact} in \Cref{sec:subroutine-L} and of \Cref{algo:comp-post,algo:comp} in \Cref{sec:subroutines-comp-post,sec:subroutines-comp} respectively, along with two concrete examples in both sections. Experimental results for the confidence intervals for pre- and post-change parameters are in \Cref{sec:expt-paramter}. Simpler algorithms for PFA-controlling change detectors are presented in \Cref{sec:pfa}, along with their implementations in simulation experiments. Additional experiments with varying $T$ vs.\ ARL and scenarios with dependent (Markov) pre- and post-change data are presented in \Cref{sec:add-expt}. 

\section{Omitted proofs}
\label{a-proof}
In this section, we provide the mathematical details that were omitted in the main paper.

\begin{proof}[Proof of \Cref{prop:impossibility}.]
    Assume that $\P_T(T\in\mathcal{C}^*)\geq1-\alpha$ for some $\mathcal{C}^*\subseteq \{1,\cdots,\tau\}$. Now we observe that  $$\P_\infty(\tau \geq T)=\P_T(\tau \geq T)\geq \P_T(T\in\mathcal{C}^*)\geq1-\alpha.$$ 
    Therefore, $\P_\infty(\tau =\infty)= 1-\P_\infty(\cup_{n=1}^\infty[\tau <n])=1-\lim_{n\to\infty}\P_\infty(\tau<n)=\lim_{n\to\infty}\P_\infty(\tau\geq n)\geq 1-\alpha.$
\end{proof}

\begin{proof}[Proof of \Cref{prop:dual}.]
First note the fact that \(\P_{F_0,T,F_1}(\tau< T)=\mathbb \P_{F_0,\infty}(\tau< T),\)
since the event $[\tau< T]=\displaystyle\cup_{i=1}^{T-1}[\tau=i]$ does not depend on $\{X_i:i\geq T\}$. 

Let $\mathcal{C} \subseteq \{1, \dots, \tau\}$ be a confidence set satisfying $\P_{F_0,t,F_1}(t \in \mathcal{C}\mid\tau \geq t)\geq 1-\alpha$, for all $t\in\N$. Since $\tau$ is a stopping time, $[\tau< t]=\displaystyle\cup_{i=1}^{t-1}[\tau=i]$ does not depend on $\{X_i:i\geq t\}$. This implies $\P_{F_0,\infty}(\tau \ge t)=\P_{F_0,t,F_1}(\tau \ge t).$ Therefore,
 
\begin{align*}
     \P_{F_0,t,F_1}(\phi^{(t)}(X_1,\cdots,X_\tau)=1)&=\P_{F_0,t,F_1}(t \notin \mathcal{C},t\leq\tau)\\
     &=\P_{F_0,t,F_1}(t \notin \mathcal{C}\mid \tau \ge t)\P_{F_0,t,F_1}(\tau \ge t)\\
     &\leq \alpha\P_{F_0,\infty}(\tau \ge t), ~\forall F_0\in\mathcal{P}_0, F_1\in\mathcal{P}_1.
 \end{align*}

For the converse part,  $\phi^{(t)}$  satisfies 
$\P_{F_0,t,F_1}(\phi^{(t)}(X_1,\cdots,X_\tau)=1)\leq \alpha \cdot \P_{F_0,\infty}(\tau \ge t), \forall F_0\in\mathcal{P}_0, F_1\in\mathcal{P}_1$.
So, $\mathcal{C}=\{t\in\N, t\leq \tau:\phi^{(t)}(X_1,\cdots,X_\tau)=0 \}$ satisfies
\begin{align*}
    &\P_{F_0,T,F_1}(\phi^{(T)}(X_1,\cdots,X_\tau)=0) \\
    &=\P_{F_0,T,F_1}(\phi^{(T)}(X_1,\cdots,X_\tau)=0,\tau \ge T)+\P_{F_0,T,F_1}(\phi^{(T)}(X_1,\cdots,X_\tau)=0, \tau < T)\\
    & \leq\P_{F_0,T,F_1}(T \in \{t\in\N, t\leq \tau:\phi^{(t)}(X_1,\cdots,X_\tau)=0\})+\P_{F_0,\infty}(\tau < T)\\
    &=\P_{F_0,T,F_1}(T \in \mathcal{C})+\P_{F_0,\infty}(\tau < T).
\end{align*}
Therefore, $\P_{F_0,T,F_1}(T \in \mathcal{C})\geq 1-\alpha \cdot \P_{F_0,\infty}(\tau \ge T)-\P_{F_0,\infty}(\tau < T)$, which implies
\begin{equation*}
    \P_{F_0,T,F_1}(T \in  \mathcal{C}  \mid \tau \ge T) =\frac{\P_{F_0,T,F_1}(T \in \mathcal{C},\tau \ge T)}{\P_{F_0,T,F_1}(\tau \ge T)}=\frac{\P_{F_0,T,F_1}(T \in \mathcal{C})}{\P_{F_0,\infty}(\tau \ge T)}\ge 1 - \alpha.
\end{equation*}
This completes the proof.
\end{proof}

\begin{lemma}[Ville’s inequality \cite{ville1939etude}]
\label{lem:ville}
    If $\{R_n\}_{n\in\N}$ is an e-process for $\mathcal P$ with respect to a filtration $\{\mathcal{F}_n\}_{n\in\N}$, then 
    \begin{equation}
        \P\left(\exists n\in\N: R_n\geq1/\alpha\right)\leq \alpha, \forall \P\in \mathcal P.
    \end{equation}
    
\end{lemma}

 For a short proof, see \cite{howard2020time}[Lemma 1].

\begin{proof}[Proof of \Cref{lem:univ-threshold}.]
    Note that by definition of $M_t$, we have $\P_{F_0,t,F_1}[M_t\geq2/\alpha, t=\hat{T}\leq \tau<\infty]=0$, $\P_{F_0,t,F_1}[M_t\geq2/\alpha, t>\tau]=0$ and $\P_{F_0,t,F_1}[M_t>2/\alpha, \tau=\infty]=0$. Therefore, the only cases where $M_t>2/\alpha$ might occur are $t<\hat{T}\leq \tau<\infty$ or $\hat{T}<t\leq \tau<\infty$. Hence, for any $t\in\N$,
     \begin{align*}
         &\P_{F_0,t,F_1}[M_t\geq2/\alpha]\\
         &=\P_{F_0,t,F_1}\left[\max\{\max_{t\leq n\leq\tau} R^{(t)}_{n},\max_{1\leq n\leq t-1}S^{(t)}_{n}\}\geq2/\alpha,t\leq\tau<\infty\right]\\
         &\leq \P_{F_0,t,F_1}\left[\max_{t\leq n\leq\tau} R^{(t)}_{n}\geq2/\alpha,t\leq\tau<\infty\right]+\P_{F_0,t,F_1}\left[\max_{1\leq n\leq t-1}S^{(t)}_{n}\geq2/\alpha,t\leq\tau<\infty\right].
     \end{align*}  
     Now, from Ville's inequality (\Cref{lem:ville})  on the e-processes $\{R^{(t)}_{n}\}_{n\geq t}$,  $$P_{F_1}\left(\exists n\geq t, R^{(t)}_{n}\geq 2/\alpha\right)\leq\alpha/2.$$
     Therefore, 
\begin{align*}
\P_{F_0,t,F_1}\left[\max_{t\leq n\leq\tau} R^{(t)}_{n}\geq2/\alpha,t\leq\tau<\infty\right]
    &\leq \P_{F_0,t,F_1}\left(\cup_{n=t}^\infty \left[R^{(t)}_{n}\geq2/\alpha\right]\right)\\
    &=\P_{F_0,t,F_1}\left(\exists n\geq t, R^{(t)}_{n}\geq2/\alpha\right)\leq \alpha/2.
\end{align*}
     Similarly, for the e-processes $\{S^{(t)}_{n}\}_{n<t}$, we get $\P_{F_0,t,F_1}\left[S^{(t)}_{(t-1)\wedge\hat{T}}\geq2/\alpha\right]\leq \alpha/2.$ Therefore, from the above steps, we finally have
     $$\P_{F_0,t,F_1}[M_t\geq2/\alpha]\leq \alpha/2+\alpha/2= \alpha.$$
     This completes the proof.
\end{proof}

 \begin{proof}[Proof of \Cref{thm:coverage-nonpara-known-pre}.]
   Observe that
     \begin{align*}
  &\P_{F_0,T,F_1}( M_T< 2/(\alpha r_T))\\
   &=\P_{F_0,T,F_1}(T\in \{t\in\N: M_t< 2/(\alpha r_t)\})\\
   &\leq\P_{F_0,T,F_1}(T\in \{t\in\N: M_t< 2/(\alpha r_t)\},\tau\geq T) +\mathbb \P_{F_0,T,F_1}(\tau< T)\\
    &= \P_{F_0,T,F_1}(T\in \{t\in\N: t\leq \tau, M_t< \frac{2}{\alpha r_t}\}\mid \tau\geq T)\P_{F_0,T,F_1}(\tau\geq T)+\mathbb \P_{F_0,T,F_1}(\tau< T)\\
    &= \P_{F_0,T,F_1}(T\in \mathcal{C}\mid \tau\geq T)\mathbb \P_{F_0,\infty}(\tau\geq T)+1-\mathbb \P_{F_0,\infty}(\tau\geq T),
\end{align*}
where we twice used the fact that \(\P_{F_0,T,F_1}(\tau< T)=\mathbb \P_{F_0,\infty}(\tau< T)\), which holds because $\tau$ is a stopping time and so the event $[\tau< T]=\displaystyle\cup_{i=1}^{T-1}[\tau=i]$ does not depend on $\{X_i:i\geq T\}$.

 From \Cref{lem:univ-threshold}, we have $\P_{F_0,t,F_1}[M_t\geq\frac{2}{\alpha r_t}\mid r_t]\leq \alpha r_t$, since $r_t$ and $M_t$ are independent of each other. Now, since $r_t$ is an unbiased estimator of $\P_{F_0,\infty}[\tau\geq t]$, 
  $\P_{F_0,t,F_1}[M_t\geq\frac{2}{\alpha r_t}]=\mathbb E[\P_{F_0,t,F_1}[M_t\geq\frac{2}{\alpha r_t}\mid r_t]]\leq \mathbb E[\alpha r_t]=\alpha\P_{F_0,\infty}[\tau\geq t]$, $\forall t\in\N.$ 
So, we obtain
\begin{align*}
    \P_{F_0,T,F_1}(T\in \mathcal{C}\mid \tau\geq T)\geq  1-&\frac{\P_{F_0,T,F_1}( M_T> 2/(\alpha r_T))}{ \P_{F_0,\infty}(\tau\geq T)}\geq 1-\alpha.
\end{align*}
This completes the proof.
\end{proof}
\begin{proof}[Proof of \Cref{thm:len}.]
We first decompose the length of the confidence set into contributions from the pre- and post-change regions, and then control each contribution separately.
\paragraph{\textbf{Decomposition of the confidence set.}}
Whenever $\tau\geq T>t,$ $$M_t\geq \prod_{i=t}^{{T}-1} \frac{f_0 (X_i)}{f_1 (X_i)}.$$ Similarly, whenever $T<t\leq\tau,$ $$M_t\geq\prod_{i=T}^{{t}-1} \frac{f_1 (X_i)}{f_0 (X_i)}.$$ Therefore, the following holds almost surely:
\begin{align*}
&|{\mathcal{C}}|\times\mathds{1}(\tau\geq T)\\
&\leq\sum_{t=1}^{{T-1}}\mathds{1}\left(\prod_{i=t}^{{T}-1} \frac{f_0 (X_i)}{f_1 (X_i)}\leq 2/(\alpha r_t )\right)+1+\sum_{t=T+1}^{{\tau}}\mathds{1}\left(\prod_{i=T}^{{t}-1} \frac{f_1 (X_i)}{f_0 (X_i)}\leq 2/(\alpha r_t )\right)\\
&=\sum_{t=1}^{{T-1}}\mathds{1}\left(\prod_{i=t}^{{T}-1} \frac{f_0 (X_i)}{f_1 (X_i)}\leq 2/(\alpha r_t )\right)+1+\sum_{t=T+1}^{{\infty}}\mathds{1}\left(\prod_{i=T}^{{t}-1} \frac{f_1 (X_i)}{f_0 (X_i)}\leq 2/(\alpha r_t ),  t \leq\tau\right).
\end{align*}
Also, the following holds almost surely:
\begin{equation}
\label{bdd-gen}
|{\mathcal{C}}|\times\mathds{1}(\tau\geq T)\leq\sum_{t=1}^{{T-1}}\mathds{1}\left(\prod_{i=t}^{{T}-1} \frac{f_0 (X_i)}{f_1 (X_i)}\leq 2/(\alpha r_t )\right)+1+|\tau-T|.
  \end{equation}
So, taking the conditional expectation on both sides and combining both bounds, we get
\begin{align}
\label{eq:bdd-2}
\nonumber&\mathbb E_{F_0,T,F_1} (|{\mathcal{C}}|\mid \tau\geq T)\leq 1+\sum_{t=1}^{{T-1}}\P_{F_0,T,F_1}\left(\prod_{i=t}^{{T}-1} \frac{f_0 (X_i)}{f_1 (X_i)}\leq 2/(\alpha r_t )\mid \tau\geq T\right)\\
&+ \min\left\{\sum_{t=T+1}^{{\infty}}\P_{F_0,T,F_1}\left(\prod_{i=T}^{{t}-1} \frac{f_1 (X_i)}{f_0 (X_i)}\leq 2/(\alpha r_t ), t \leq\tau\mid \tau\geq T\right), \Delta_{F_0,T,F_1}\right\}.
\end{align}

\paragraph{\textbf{Controlling the Pre-Change Tail.}}

 Fix some $s\in(0,1)$. Then, using Markov's equality,
\begin{align*}
&\P_{F_0,T,F_1}\left(\prod_{i=t}^{{T}-1} \frac{f_0 (X_i)}{f_1 (X_i)}\leq 2/(\alpha r_t )\right)\\
&\leq \P_{F_0,T,F_1}\left(\prod_{i=t}^{{T}-1} \frac{f_0 (X_i)}{f_1 (X_i)}\leq 2/(\alpha r_{T} )\right) \quad\text{[Since } r_t \leq  r_{T} ]\\
&=\mathbb E\left(\P_{F_0 ,T,F_1}\left(\prod_{i=t}^{{T}-1} \left(\frac{f_1 (X_i)}{f_0 (X_i)}\right)^s\geq (\alpha  r_{T} )^s/2^s\mid  r_{T}\right)\right) \\
&\leq \mathbb E\left(\left(\frac{2}{\alpha  r_{T}}\right)^s\mathbb E_{F_0 ,T,F_1}\left(\prod_{i=t}^{{T}-1} \left(\frac{f_1 (X_i)}{f_0 (X_i)}\right)^s\right)\right)\\
&= \mathbb E\left(\left(\frac{2}{\alpha  r_{T}}\right)^s\left( \mathbb E_{F_0}\left(\frac{f_1 (X_i)}{f_0 (X_i)}\right)^s\right)^{T-t}\right)\\
&= \mathbb E\left(\frac{2}{\alpha  r_{T}}\right)^s\rho_0(s)^{T-t}.
\end{align*}
This holds for all $s\in(0,1)$ and hence for $s=s_0.$ As noted earlier, $\rho_0(s_0)<1$. So, we have
\begin{align*}
    &\sum_{t=1}^{{T-1}}\P_{F_0,T,F_1}\left(\prod_{i=t}^{{T}-1} \frac{f_0 (X_i)}{f_1 (X_i)}\leq 2/(\alpha r_t )\mid \tau\geq T\right)\\
    &\leq\frac{1}{p_T}\sum_{t=1}^{{T-1}}\P_{F_0,T,F_1}\left(\prod_{i=t}^{{T}-1} \frac{f_0 (X_i)}{f_1 (X_i)}\leq 2/(\alpha r_t )\right)\\
    &\leq \mathbb E\left(\frac{2}{\alpha  r_{T}}\right)^s\frac{1}{p_T}\sum_{t=1}^{{T-1}}\rho_0(s)^{T-t}\\
    &\to \frac{2^s}{\alpha^s p_{{T}}^{s+1}}\times\frac{\rho_0(s)-\rho_0(s)^{T-1}}{1-\rho_0(s)},
\end{align*}
as $N\to\infty.$ To justify the convergence, we note that $\left(\frac{2}{\alpha  r_{T}}\right)^s$ is uniformly integrable as a sequence in $N\in\N$ (follows from \Cref{lem:ui}), and by the strong law of large numbers, $\left(\frac{2}{\alpha  r_{T}}\right)^s\to \left(\frac{2}{\alpha  p_{T}}\right)^s$ almost surely.
Since the above holds for any $s\in(0,1)$, it also holds for $s=s_0.$ 
\paragraph{\textbf{Deriving the first bound.}}
So, taking the conditional expectation on both sides of \eqref{bdd-gen}, we get
\begin{align*}
&\mathbb E_{F_0,T,F_1} (|{\mathcal{C}}|\mid \tau\geq T)\leq 1+\sum_{t=1}^{{T-1}}\P_{F_0,T,F_1}\left(\prod_{i=t}^{{T}-1} \frac{f_0 (X_i)}{f_1 (X_i)}\leq 2/(\alpha r_t )\mid \tau\geq T\right)+\Delta_{F_0,T,F_1}.
\end{align*}
Therefore, it follows from the above calculation that 
\begin{equation}
  \limsup_{N\to\infty} \mathbb E_{F_0,T,F_1} (|{\mathcal{C}}|\mid \tau\geq T)\leq  \frac{2^{s_0}}{\alpha^{s_0} p_{{T}}^{s_0+1}}\times\frac{\rho_0(s_0)-\rho_0(s_0)^{T-1}}{1-\rho_0(s_0)}+1+\Delta_{F_0,T,F_1}.
\end{equation}

\paragraph{\textbf{Controlling the ``Post-Change" Tail when $\tau$ is sensitive.}}
 
  Similarly, for $t\geq T+1$, we have
\begin{align*}
&\P_{F_0,T,F_1}\left(\prod_{i=T}^{{t}-1} \frac{f_1 (X_i)}{f_0 (X_i)}\leq 2/(\alpha r_t ),t\leq\tau\right)\\
&\leq \P_{F_0,T,F_1}\left(\prod_{i=T}^{{t}-1} \frac{f_1 (X_i)}{f_0 (X_i)}\leq 2/(\alpha r_\tau )\right)\\
&\leq \P_{F_0,T,F_1}\left(\prod_{i=T}^{{t}-1} \frac{f_1 (X_i)}{f_0 (X_i)}\leq 2/(\alpha r_\tau  ),r_\tau>\rho_1(s)^{t-T}\right)+\P_{F_0,T,F_1}[r_\tau\leq\rho_1(s)^{t-T}]\\
&\leq \P_{F_0,T,F_1}\left(\prod_{i=T}^{{t}-1} \left(\frac{f_0 (X_i)}{f_1 (X_i)}\right)^s\geq (\alpha \rho_1(s)^{t-T})^s/2^s\right)+\P_{F_0,T,F_1}[r_\tau\leq\rho_1(s)^{t-T}]\\
&\leq \left(\frac{2}{\alpha  \rho_1(s)^{t-T}}\right)^s\mathbb E_{F_1}\left(\prod_{i=T}^{{t}-1} \left(\frac{f_0 (X_i)}{f_1 (X_i)}\right)^s\right)+\P_{F_0,T,F_1}[r_\tau\leq\rho_1(s)^{t-T}]\\
&\leq \left(\frac{2}{\alpha  \rho_1(s)^{t-T}}\right)^s\rho_1(s)^{t-T}+\rho_1(s)^{t-T}.
\end{align*}

The last step follows from the assumption that $\tau$ is sensible (recall that if $\tau$ is sensitive, then $\tau_j$ is stochastically larger than $\tau$) and from \Cref{lem:r_tau}. Also, as we noted before, $\rho_1(s_1)<1.$ Therefore,
\begin{align*}
&\sum_{t=T
+1}^{{\infty}}\P_{F_0,T,F_1}\left(\prod_{i=T}^{{t}-1} \frac{f_1 (X_i)}{f_0 (X_i)}\leq 2/(\alpha r_t ), \tau\geq t\mid \tau\geq T\right)\\
&\leq\frac{1}{p_T}\sum_{t=T+1}^{{\infty}}\P_{F_0,T,F_1}\left(\prod_{i=T}^{{t}-1} \frac{f_1 (X_i)}{f_0 (X_i)}\leq 2/(\alpha r_t ), \tau\geq t\right)\\
&\leq \frac{1}{p_T}\sum_{t=T+1}^{{\infty}}\left(  \left(\frac{2}{\alpha}\right)^s\rho_1(s)^{s(t-T)}+\rho_1(s)^{t-T}\right)\\
&\leq \frac{1}{p_T}\left(  \left(\frac{2}{\alpha}\right)^s\times\frac{\rho_1(s)^{s}}{1-\rho_1(s)^{s}}+\frac{\rho_1(s)}{1-\rho_1(s)}\right).
\end{align*}
  Since the above holds for any $s\in(0,1)$, it also holds for $s=s_1.$

  \paragraph{\textbf{Deriving the final bound.}}
Finally, it follows from \eqref{eq:bdd-2} and the above calculations that
\begin{align}
\nonumber&\limsup_{N\to\infty}\mathbb E_{F_0,T,F_1} (|{\mathcal{C}}|\mid \tau\geq T)\leq \frac{2^{s_0}}{\alpha^{s_0} p_{{T}}^{s_0+1}}\times\frac{\rho_0(s_0)-\rho_0(s_0)^{T-1}}{1-\rho_0(s_0)}+1\\
&+\min\left\{\frac{1}{p_T}\left(  \left(\frac{2}{\alpha}\right)^{s_1}\times\frac{\rho_1(s_1)^{s_1}}{1-\rho_1(s_1)^{s_1}}+\frac{\rho_1(s_1)}{1-\rho_1(s_1)}\right), \Delta_{F_0,T,F_1}\right\}.
\end{align}
This completes the proof.
\end{proof}
\begin{proof}[Proof of \Cref{prop:mlr}.]
We only prove the case when $\tau$ is the CUSUM type stopping time $\tau_{\text{wcs-ripr}}$ defined as

\begin{equation*}
    \tau_{\text{cs}}=\inf\left\{n\in\mathbb N:\max_{1\leq j\leq n}\prod_{i=j}^n\frac{p_1^*(X_i)}{p_0^*(X_i)}\geq A\right\}.
\end{equation*}
Proofs for the other case,  when $\tau$ is the SR-type stopping time, follow analogously.
\begin{align*}
    \P_{P,\infty}[\tau_{\text{cs}}\geq t]
&=\P_{P,\infty}\left[\displaystyle\max_{1\leq k\leq t-1}\max_{1\leq j\leq k}\prod_{i=j}^k\frac{p_{1}^*(X_i)}{p_{0}^*(X_i)}<A\right].
\end{align*} 
Since $(P_0^*, P_1^*)$ is an LFD pair for testing $\mathcal{P}_0$ vs.\ $\mathcal{P}_1$, for any $P_0\in\mathcal{P}_0,$ we have 
\begin{align*}
  &{R(P_0, \phi) \leq R(P_0^*, \phi)} \\
  &{\implies C_0 {P_0}(\phi \text{ accepts } {H_1}) 
      \leq C_0 {P_0^*}(\phi \text{ accepts } {H_1})} \text{ for every LRT } \phi \text{ between } P_0^* \text{ and } P_1^*\\
  &{\implies {P_0}\!\left(\frac{p_1^*(X)}{p_0^*(X)} > \eta\right) 
      \leq {P_0^*}\!\left(\frac{p_1^*(X)}{p_0^*(X)} > \eta\right), \text{ for all }\eta }
\end{align*}
which is same as
\[
\frac{p_1^*(X)}{p_0^*(X)} \preceq_{\text{st}} \frac{p_1^*(X')}{p_0^*(X')},
\]
where $X\sim P$ and $X'\sim P^*$.
By Theorem 1.A.2 of \cite{shaked2007stochastic}, the stochastic ordering implies that there exist a random variable $Z$ and functions $\psi_{1}$ and $\psi_{2}$ such that $\psi_{1}(z)\leq \psi_{2}(z)$ for all $z$ and 
$\frac{p_1^*(X)}{p_0^*(X)}\stackrel{d}{=}\psi_{1}(Z)$ and 
$\frac{p_1^*(X')}{p_0^*(X')}\stackrel{d}{=}\psi_{2}(Z)$. Therefore,
\begin{align*}
    \P_{P,\infty}[\tau_{\text{cs}}\geq t]
&=\P\left[\displaystyle\max_{1\leq k\leq t-1}\max_{1\leq j\leq k}\prod_{i=j}^k\psi_{1}(Z_{i})dW(\theta_1)<A\right]\\
&\geq\P\left[\displaystyle\max_{1\leq k\leq t-1}\max_{1\leq j\leq k}\prod_{i=j}^k\psi_{2}(Z_{i})dW(\theta_1)<A\right]\\
&=\P_{P^*,\infty}[\tau_{\text{cs}}\geq t].
\end{align*}
This completes the proof.
\end{proof}

\begin{proof}[Proof of \Cref{thm:coverage-nonpara-comp-pre}.]
Since $r_t$ is an unbiased estimator of $\P_{F_0^*,\infty}[\tau\geq t]$ and is independent of data, we have $\forall t\in\N$, 
    
  $\P_{F_0,t,F_1}[M_t\geq2/(\alpha r_t^*)]=\mathbb E[\P_{F_0,t,F_1}[M_t\geq2/(\alpha r_t^*)\mid r_t^*]]\leq \mathbb E[\alpha r_t^*]=\alpha\P_{F_0^*,\infty}[\tau\geq t]\leq\alpha\P_{F_0,\infty}[\tau\geq t],$  $\forall F_0\in\mathcal P_0$.
 
 The last steps follow from \Cref{assmp-1}.
The rest of the proof is the same as in the proof of \Cref{thm:coverage-nonpara-known-pre}.
\end{proof}

\begin{proof}[Proof of \Cref{thm:len:comp}.]
    It follows from the definition of the least favorable distribution that 
$$\P_{F_0,T,F_1}\left(\prod_{i=t}^{{T}-1} \frac{f_0^* (X_i)}{f_1^* (X_i)}\leq 2/(\alpha r_t )\right)\leq \P_{F_0^*,T,F_1^*}\left(\prod_{i=t}^{{T}-1} \frac{f_0^* (X_i)}{f_1^* (X_i)}\leq 2/(\alpha r_t )\right).$$ Therefore,
for any $t\leq T-1$, using the above inequality and the assumption \eqref{eq-a} we get
    \begin{align*}
       & \P_{F_0,T,F_1}\left(\prod_{i=t}^{{T}-1} \frac{f_0^* (X_i)}{f_1^* (X_i)}\leq 2/(\alpha r_t )\mid \tau\geq T\right)\\
    &\leq\frac{1}{\P_{F_0,\infty}(\tau\geq T)}\P_{F_0,T,F_1}\left(\prod_{i=t}^{{T}-1} \frac{f_0^* (X_i)}{f_1^* (X_i)}\leq 2/(\alpha r_t )\right)\\
    &\leq\frac{1}{\P_{F_0^*,\infty}(\tau\geq T)}\P_{F_0^*,T,F_1^*}\left(\prod_{i=t}^{{T}-1} \frac{f_0^* (X_i)}{f_1^* (X_i)}\leq 2/(\alpha r_t )\right).
    \end{align*}
    Similarly, for any $t\geq T+1$,
\begin{align*}
   & \P_{F_0,T,F_1}\left(\prod_{i=T}^{{t}-1} \frac{f_1^* (X_i)}{f_0^* (X_i)}\leq 2/(\alpha r_t ), \tau\geq t\mid \tau\geq T\right)\\
&\leq\frac{1}{\P_{F_0,\infty}(\tau\geq T)}\P_{F_0,T,F_1}\left(\prod_{i=T}^{{t}-1} \frac{f_1^* (X_i)}{f_0^*(X_i)}\leq 2/(\alpha r_\tau)\right)\\
&\leq\frac{1}{\P_{F_0,\infty}(\tau\geq T)}\left[\P_{F_0,T,F_1}\left(\prod_{i=T}^{{t}-1} \frac{f_1^* (X_i)}{f_0^* (X_i)}\leq 2/(\alpha \eta )\right)+\P_{F_0,T,F_1}(r_\tau\leq\eta)\right]\\
&\leq\frac{1}{\P_{F_0^*,\infty}(\tau\geq T)}\left[\P_{F_0^*,T,F_1^*}\left(\prod_{i=T}^{{t}-1} \frac{f_1^* (X_i)}{f_0^* (X_i)}\leq 2/(\alpha \eta )\right)+\P_{F_0^*,T,F_1^*}(r_\tau\leq\eta)\right].
\end{align*}
In the last step above, we used the inequality $\P_{F_0,\infty}(\tau\geq T)\geq \P_{F_0^*,\infty}(\tau\geq T)$, which follows from \eqref{eq-a}, and 
$\P_{F_0,T,F_1}\left(\prod_{i=T}^{{t}-1} \frac{f_1^* (X_i)}{f_0^* (X_i)}\leq 2/(\alpha \eta )\right)$ $\leq \P_{F_0^*,T,F_1^*}\left(\prod_{i=T}^{{t}-1} \frac{f_1^* (X_i)}{f_0^* (X_i)}\leq 2/(\alpha \eta )\right)$ 
follows from the fact that $(F_0^*,F_1^*)$ is an LFD pair, and 
$\P_{F_0,T,F_1}(r_\tau\leq\eta)\leq \P_{F_0^*,T,F_1^*}(r_\tau\leq\eta)$ follows from  \eqref{eq-c}.

The remaining calculations are the same as those done in the proof of the previous theorem, with $f_i$ replaced by $f_i^*.$
\end{proof}

\begin{proof}[Proof of \Cref{thm:exact-cond-coverage}.]
First note the fact that \(\P_{F_0,T,F_1}(\tau< T)=\mathbb \P_{F_0,\infty}(\tau< T),\)
since the event $[\tau< T]=\displaystyle\cup_{i=1}^{T-1}[\tau=i]$ does not depend on $\{X_i:i\geq T\}$. 

 Defining $C_{t,L}=\operatorname{Quantile}(1-\hat\alpha_{t};M_t,M_{t,L}^1,\cdots,M_{t,L}^B)$, and $\hat\alpha_{t}=\alpha r_{t}$. Now,
\begin{align*}
  &\P_{F_0,T,F_1}( M_T\leq C_{T,L})\\
   &=\P_{F_0,T,F_1}(T\in \{t\in\N: M_t\leq C_{t,L}\})\\
   &=\P_{F_0,T,F_1}(T\in \{t\in\N: M_t\leq C_{t,L}\},\tau\geq T) +\P_{F_0,T,F_1}(T\in \{t\in\N: M_t\leq C_{t,L}\},\tau< T)\\
    &\leq \P_{F_0,T,F_1}(T\in \{t\in\N: t\leq \tau, M_t\leq C_{t,L}\}\mid \tau\geq T)\P_{F_0,T,F_1}(\tau\geq T)+\mathbb \P_{F_0,T,F_1}(\tau< T)\\
    &= \P_{F_0,T,F_1}(T\in \mathcal{C}\mid \tau\geq T)\mathbb \P_{F_0,\infty}(\tau\geq T)+1-\mathbb \P_{F_0,\infty}(\tau\geq T),
\end{align*}
where we twice invoked \(\P_{F_0,T,F_1}(\tau< T)=\mathbb \P_{F_0,\infty}(\tau< T)\) as explained earlier.
So, we obtain
\begin{equation}
\label{eq:cond-level}
    \P_{F_0,T,F_1}(T\in \mathcal{C}\mid \tau\geq T)\geq  1-\frac{P_T( M_T> C_{T,L})}{ P_\infty(\tau\geq T)}.
\end{equation}
 First, observe that $M_t,M_{t,\infty}^1,\cdots,M_{t,\infty}^B$ are i.i.d.\ under $H_{0,t}$. Hence, $M_t\leq \operatorname{Quantile}(1-c;M_t,M_{t,\infty}^1,\cdots,M_{t,\infty}^B)$ $\iff M_t$ is one among the $\ceil{(1-c)(B+1)}$ smallest of the collection $\{M_t,M_{t,\infty}^1,\cdots,M_{t,\infty}^B\}$. So, for any fixed $c\in(0,1)$,
\begin{equation*}
  \P_{F_0,t,F_1}(M_t\leq \operatorname{Quantile}(1-c;M_t,M_{t,\infty}^1,\cdots,M_{t,\infty}^B))\geq\frac{\ceil{(1-c)(B+1)}}{B+1}\geq 1-c \text{ for all } t.
\end{equation*}
Since $\hat\alpha_{t}$ is independent of $M_t$, it follows that
\begin{equation*}
  \P_{F_0,t,F_1}(M_t\leq C_{t,\infty}\mid \hat\alpha_{t})\geq\frac{\ceil{(1- \hat\alpha_{t})(B+1)}}{B+1}\geq 1- \hat\alpha_{t}, \text{ almost surely for all } t.
\end{equation*}
Therefore, we have
\begin{equation*}
  \P_{F_0,t,F_1}(M_t\leq C_{t,\infty})=\mathbb E_t(\mathbb P_t(M_t\leq C_{t,\infty}\mid \hat\alpha_{t}))\geq \mathbb E_t( 1-\hat\alpha_{t})=1-\alpha\P_{F_0,t,F_1}(\tau\geq t).
\end{equation*}
As noted earlier, $\P_{F_0,t,F_1}(\tau\geq t)=\P_{F_0,\infty}(\tau\geq t)$ and hence, $\P_{F_0,t,F_1}(M_t\leq C_{t,\infty})\geq 1-\alpha\P_{F_0,\infty}(\tau\geq t).$
Also, for all $L\in\N$, it follows from the definitions that $M_{t,L}^j\geq M_{t,\infty}^j$. Therefore, $C_{t,L} \geq C_{t,\infty}$.
 Thus, for all $t$,
 \begin{equation}
 \label{level-h0t}
     \P_{F_0,t,F_1}(M_t> C_{t,L})\leq \P_{F_0,t,F_1}(M_t> C_{t,\infty})\leq \alpha \P_{F_0,\infty}(\tau\geq t).
 \end{equation}
Therefore, combining \eqref{eq:cond-level} and \eqref{level-h0t}, we obtain $\P_{F_0,T,F_1}(T\in \mathcal{C}\mid \tau\geq T)\geq  1-\alpha.$
\end{proof} 
\begin{proof}[Proof of \Cref{thm:coverage-comp-post}.]
  First, observe that $M_t,\{M_{t,\infty}^j(\theta_1)\}_{j=1}^B$ are i.i.d.\ under $\P_{\theta_0,t,\theta_1}$.
So, $M_t\leq \operatorname{Quantile}(1-\alpha;M_t,\{M_{t,\infty}^j(\theta_1)\}_{j=1}^B)$ $\iff M_t$ is one among the $\ceil{(1-\alpha)(B+1)}$ smallest of $M_t,\{M_{t,\infty}^j(\theta_1)\}_{j=1}^B\}$. Hence, for any fixed $c\in(0,1)$,
\begin{equation}
  \P_{\theta_0,t,\theta_1}(M_t\leq \operatorname{Quantile}(1-c;M_t,\{M_{t,\infty}^j(\theta_1)\}_{j=1}^B))\geq\frac{\ceil{(1-c)(B+1)}}{B+1}\geq 1-c, \text{ for all } t.
\end{equation}
Define, $C_{t,L}^\alpha=\sup_{\theta\in\mathcal{S}^\prime_{\tau-t+1}}\operatorname{Quantile}(1-\alpha;M_t,\{M_{t,L}^j(\theta)\}_{j=1}^B)$.
Note that $\theta_1\in\mathcal{S}^\prime_{\tau-t+1}$ imply that $C_{t,L}^\alpha\geq C_{t,\infty}^\alpha\geq\operatorname{Quantile}(1-\alpha;M_t,\{M_{t,\infty}^j(\theta_1)\}_{j=1}^B\}),$ for all $t\in\N,\alpha\in(0,1)$. 

For all $t$, with  $r_t = \sum_{i=1}^N \mathds{I}(\tau_i\geq 
 t)/N$, since $r_t$ is independent to $\{X_n\}_n$, we have
 \begin{align*}
     \P_{\theta_0,t,\theta_1}(\theta_1\in\mathcal{S}^\prime_{\tau-t+1})&=\mathbb E(\P_{\theta_0,t,\theta_1}(\theta_1\in\operatorname{CS}(X_t,\cdots,X_{\tau};1-\beta r_t)\mid r_t))\\
     &\geq \mathbb E(1-\beta r_t)= 1-\beta\P_{\theta_0,\infty}(\tau\geq t).
 \end{align*}
Then, we have
\begin{align*}
    &\P_{\theta_0,t,\theta_1}( M_t> C_{t,L}^{\alpha r_t})\\
    &\leq \P_{\theta_0,t,\theta_1}( M_t> C_{t,L}^{\alpha r_t}, \theta_1\in\mathcal{S}^\prime_{\tau-t+1})+\P_{\theta_0,t,\theta_1}(\theta_1\notin\mathcal{S}^\prime_{\tau-t+1})\\
    &\leq \P_{\theta_0,t,\theta_1}( M_t> \operatorname{Quantile}(1-\alpha r_{t};M_t,\{M_{t,\infty}^j(\theta_1)\}_{j=1}^B))+\beta \P_{\theta_0,\infty}(\tau\geq t)\\
    &\leq \mathbb E\left(\P_{\theta_0,t,\theta_1}\left( M_t> \operatorname{Quantile}(1-\alpha r_{t};M_t,\{M_{t,\infty}^j(\theta_1)\}_{j=1}^B)\mid r_{t}\right) \right)+\beta \P_{\theta_0,\infty}(\tau\geq t)\\
&\leq \mathbb E(\alpha r_{t})+\beta \P_{\theta_0,\infty}(\tau\geq t)\\
&= (\alpha+\beta) \times \P_{\theta_0,\infty}(\tau\geq t).
\end{align*}
We also have that
\begin{align*}
    \P_{\theta_0,T,\theta_1}(T\in \{t \geq 1: M_t\leq C_{t,L}^{\alpha}\},\tau< T)\leq \P_{\theta_0,T,\theta_1}(\tau< T)=\mathbb \P_{\theta_0,\infty}(\tau< T),
\end{align*}
the last equality holds because the event $[\tau< T]$ does not depend $X_i$ for $i\geq T$. Now,
\begin{align*}
  &\P_{\theta_0,T,\theta_1}( M_T\leq C_{T,L}^{\alpha r_t})\\
   &=\P_{\theta_0,T,\theta_1}(T\in \{t\geq 1: M_t\leq C_{t,L}^{\alpha r_t}\})\\
   &\leq\P_{\theta_0,T,\theta_1}(T\in \{t\geq 1: M_t\leq C_{t,L}^{\alpha r_t}\}\mid\tau\geq T)\P_{\theta_0,T,\theta_1}(\tau\geq T)+ \P_{\theta_0,T,\theta_1}(\tau< T)\\
    &= \P_{\theta_0,T,\theta_1}(T\in \mathcal{C}_1\mid \tau\geq T)\P_{\theta_0,\infty}(\tau\geq T)+1-\P_{\theta_0,\infty}(\tau\geq T).
\end{align*}
Since $\P_{\theta_0,T,\theta_1}( M_T\leq C_{T,L}^{\alpha r_t})\geq1-(\alpha+\beta) \times \P_{\theta_0,\infty}(\tau\geq T)$, we obtain $\P_{\theta_0,T,\theta_1}(T\in \mathcal{C}_1\mid \tau\geq T)\geq 1-\alpha-\beta$ from the above inequality.
\end{proof}

\begin{proof}[Proof of \Cref{thm:coverage-comp}.]
First, observe that $M_t,\{M_{t,\infty}^j(\theta_0,\theta_1)\}_{j=1}^B$ are i.i.d.\ under $\P_{\theta_0,t,\theta_1}$.

So, $M_t\leq \operatorname{Quantile}(1-\alpha;M_t,\{M_{t,\infty}^j(\theta_0,\theta_1)\}_{j=1}^B)$ $\iff M_t$ is one among the $\ceil{(1-\alpha)(B+1)}$ smallest of $M_t,\{M_{t,\infty}^j(\theta_0,\theta_1)\}_{j=1}^B\}$.
Hence, for any fixed $c\in(0,1)$,
\begin{equation}
\P_{\theta_0,t,\theta_1}(M_t\leq \operatorname{Quantile}(1-c;M_t,\{M_{t,\infty}^j(\theta_0,\theta_1)\}_{j=1}^B))\geq\frac{\ceil{(1-c)(B+1)}}{B+1}\geq 1-c, \text{ for all } t.
\end{equation}
Define, $C_{t,L}^\alpha=\operatorname{Quantile}(1-\alpha;M_t,\{\max_{\theta\in \mathcal{S}_{t-1},\theta^\prime\in\mathcal{S}^\prime_{\tau-t+1}}M_{t,L}^j(\theta,\theta^\prime)\}_{j=1}^B)$.
Note that $\theta_0\in \mathcal{S}_{t-1}$ and $\theta_1\in\mathcal{S}^\prime_{\tau-t+1}$ imply that $C_{t,L}^\alpha\geq C_{t,\infty}^\alpha\geq\operatorname{Quantile}(1-\alpha;M_t,\{M_{t,\infty}^j(\theta_0,\theta_1)\}_{j=1}^B\}),$ for any $\alpha\in(0,1).$ 
  By Assumption 4.1, we have that $\P_{\theta_0,\infty}(\tau\geq t)\geq \P_{\theta_0^*,\infty}(\tau\geq t)$.

Also, for all $t$, with  $r_{t}^* = \sum_{i=1}^N \mathds{I}(\tau_i\geq                            t)/N$, since $r_{t}^*$ is independent to $\{X_n\}_n$, 
\begin{align*}
\P_{\theta_0,t,\theta_1}(\theta_1\in\mathcal{S}^\prime_{\tau-t+1})&=\mathbb E(\P_{\theta_0,t,\theta_1}(\theta_1\in\operatorname{CS}(X_t,\cdots,X_{\tau};1-\beta r_{t}^*)\mid r_{t}^*))\\
&\geq \mathbb E(1-\beta r_{t}^*)= 1-\beta\P_{\theta_0^*,\infty}(\tau\geq t).
\end{align*}
And similarly,
\begin{align*}
\P_{\theta_0,t,\theta_1}(\theta_0\in\mathcal{S}_{t-1})&=\mathbb E(\P_{\theta_0,t,\theta_1}(\theta_1\in\operatorname{CI}(X_1,\cdots,X_{t-1};1-\gamma r_{t}^*)\mid r_{t}^*))\\
&\geq \mathbb E(1-\gamma r_{t}^*)= 1-\gamma\P_{\theta_0^*,\infty}(\tau\geq t).
\end{align*}
Then, we have
\begin{align*}
&\P_{\theta_0,t,\theta_1}( M_t> C_{t,L}^{\alpha r_{t}^*})\\
&\leq \P_{\theta_0,t,\theta_1}( M_t> C_{t,L}^{\alpha r_{t}^*}, \theta_0\in \mathcal{S}_{t-1},\theta_1\in\mathcal{S}^\prime_{\tau-t+1})+\P_{\theta_0,t,\theta_1}(\theta_0\notin \mathcal{S}_{t-1}))+\P_{\theta_0,t,\theta_1}(\theta_1\notin\mathcal{S}^\prime_{\tau-t+1})\\
&\leq \P_{\theta_0,t,\theta_1}( M_t> \operatorname{Quantile}(1-\alpha r_{t}^*;M_t,\{M_{t,\infty}^j(\theta_0,\theta_1)\}_{j=1}^B))+(\beta+\gamma)\P_{\theta_0^*,\infty}(\tau\geq t)\\
&\leq \mathbb E\left(\P_{\theta_0,t,\theta_1}\left( M_t> \operatorname{Quantile}(1-\alpha r_{t}^*;M_t,\{M_{t,\infty}^j(\theta_0,\theta_1)\}_{j=1}^B)\mid r_{t}^*\right) \right)+(\beta+\gamma) \P_{\theta_0,\infty}(\tau\geq t)\\
&\leq \mathbb E(\alpha r_{t}^*)+(\beta + \gamma)\times \P_{\theta_0,\infty}(\tau\geq t)\\
&=\alpha\times \P_{\theta_0^*,\infty}(\tau\geq t)+(\beta+\gamma) \times \P_{\theta_0,\infty}(\tau\geq t)\\
&\leq (\alpha+\beta+\gamma) \times \P_{\theta_0,\infty}(\tau\geq t).
\end{align*}
We also have that
\begin{align*}
\P_{\theta_0,T,\theta_1}(T\in \{t \geq 1: M_t\leq C_{t,L}^{\alpha r_{t}^*}\},\tau< T)\leq \P_{\theta_0,T,\theta_1}(\tau< T)=\mathbb \P_{\theta_0,\infty}(\tau< T),
\end{align*}
the last equality holds because the event $[\tau< T]$ does not depend $X_i$ for $i\geq T$.  Now,
\begin{align*}
&\P_{\theta_0,T,\theta_1}( M_T\leq C_{t,L}^{\alpha r_{t}^*})\\
&=\P_{\theta_0,T,\theta_1}(T\in \{t\geq 1: M_t\leq C_{t,L}^{\alpha r_{t}^*}\})\\
&=\P_{\theta_0,T,\theta_1}(T\in \{t\geq 1: M_t\leq C_{t,L}^{\alpha r_{t}^*}\}\mid\tau\geq T)\P_{\theta_0,T,\theta_1}(\tau\geq T) +\P_{\theta_0,T,\theta_1}(\tau< T)\\
&\leq \P_{\theta_0,T,\theta_1}(T\in \mathcal{C}_1\mid \tau\geq T)\P_{\theta_0,\infty}(\tau\geq T)+1-\P_{\theta_0,\infty}(\tau\geq T).
\end{align*}
Since $\P_{\theta_0,T,\theta_1}( M_T\leq C_{t,L})\geq1-(\alpha+\beta+\gamma) \times \P_{\theta_0,\infty}(\tau\geq T)$, we obtain $\P_{\theta_0,T,\theta_1}(T\in \mathcal{C}_1\mid \tau\geq T)\geq 1-\alpha-\beta-\gamma$ from the above inequality.
\end{proof}

\begin{proof}[Proof of \Cref{thm:coverage-theta0-theta1}.]
First note that $\P_{\theta_0,T,\theta_1}\left(\theta_1\notin \displaystyle\operatorname{CS}(X_{T},\cdots,X_{\tau};1-\eta_1 r_T^*)\mid r_T^*\right)\leq\eta_1 r_T^*$ and $\P_{\theta_0,T,\theta_1}\left(\theta_0\notin \displaystyle\operatorname{CI}(X_{1},\cdots,X_{T-1};1-\eta_0 r_T^*)\mid r_T^*\right)\leq\eta_0 r_T^*$, because $r_T^*$ is independent of $X_i$'s. Now,
\begin{align*}
&\P_{\theta_0,T,\theta_1}\left(\theta_1\notin \displaystyle\cup_{t\in\mathcal{C}}\operatorname{CS}(X_{t},\cdots,X_{\tau};1-\eta_1 r_t^*)\mid \tau\geq T\right)\\
&\leq\P_{\theta_0,T,\theta_1}\left(\theta_1\notin \displaystyle\cup_{t\in\mathcal{C}}\operatorname{CS}(X_{t},\cdots,X_{\tau};1-\eta_1 r_t^*),T\in\mathcal{C}\mid \tau\geq T\right)+\P_{\theta_0,T,\theta_1}\left(T\notin\mathcal{C}\mid \tau\geq T\right)\\
&\leq\P_{\theta_0,T,\theta_1}\left(\theta_1\notin \displaystyle\operatorname{CS}(X_{T},\cdots,X_{\tau};1-\eta_1 r_T^*)\mid \tau\geq T\right)+\alpha\\
&\leq\P_{\theta_0,T,\theta_1}\left(\theta_1\notin \displaystyle\operatorname{CS}(X_{T},\cdots,X_{\tau};1-\eta_1 r_T^*)\right)/\P_{\theta_0,\infty}(\tau\geq T)+\alpha\\
&=\mathbb E(\P_{\theta_0,T,\theta_1}\left(\theta_1\notin \displaystyle\operatorname{CS}(X_{T},\cdots,X_{\tau};1-\eta_1 r_T^*)\mid r_T^*\right))/\P_{\theta_0,\infty}(\tau\geq T)+\alpha\\
&\leq\mathbb E(\eta_1 r_T^*)/\P_{\theta_0,\infty}(\tau\geq T)+\alpha\\
&= \alpha+\eta_1 \P_{\theta_0^*,\infty}(\tau\geq T)/\P_{\theta_0,\infty}(\tau\geq T)\\
&\leq \alpha+\eta_1,
\end{align*}
where the last inequality follows from \Cref{assmp-1}. Similarly,
\begin{align*}
&\P_{\theta_0,T,\theta_1}\left(\theta_1\notin \displaystyle\cup_{t\in\mathcal{C}}\operatorname{CI}(X_{1},\cdots,X_{t-1};1-\eta_0 r_t^*)\mid \tau\geq T\right)\\
&\leq\P_{\theta_0,T,\theta_1}\left(\theta_1\notin \displaystyle\cup_{t\in\mathcal{C}}\operatorname{CI}(X_{1},\cdots,X_{t-1};1-\eta_0 r_t^*),T\in\mathcal{C}\mid \tau\geq T\right)+\P_{\theta_0,T,\theta_1}\left(T\notin\mathcal{C}\mid \tau\geq T\right)\\
&\leq\P_{\theta_0,T,\theta_1}\left(\theta_1\notin \displaystyle\operatorname{CI}(X_{1},\cdots,X_{T-1};1-\eta_0 r_T^*)\mid \tau\geq T\right)+\alpha\\
&\leq\P_{\theta_0,T,\theta_1}\left(\theta_1\notin \displaystyle\operatorname{CI}(X_{1},\cdots,X_{T-1};1-\eta_0 r_T^*)\right)/\P_{\theta_0,\infty}(\tau\geq T)+\alpha\\
&=\mathbb E(\P_{\theta_0,T,\theta_1}\left(\theta_1\notin \displaystyle\operatorname{CI}(X_{1},\cdots,X_{T-1};1-\eta_0 r_T^*)\mid r_T^*\right))/\P_{\theta_0,\infty}(\tau\geq T)+\alpha\\
&\leq\mathbb E(\eta_0 r_T)/\P_{\theta_0,\infty}(\tau\geq T)+\alpha\\
&\leq \alpha+\eta_0.
\end{align*}
This completes the proof.
\end{proof}

\subsection{Auxilary Lemmas}
\begin{lemma}
\label{lem:ui}
 Define, $r_T=\frac{1+\sum_{i=1}^N\mathds{1}(\tau\geq T)}{N+1}$. Then, for fixed $s\in(0,1)$,      $\{r_T^{-s}\}_{N \geq 1}$ is uniformly integrable.
\end{lemma}
\begin{proof}
 A sufficient condition is to show that a higher moment is bounded: $\sup_N E[|r_T^{-s}|^{1+\delta}] < \infty$ for some $\delta > 0$. Take $\delta=\frac{1}{s}-1$, and  we need to show $\sup_N \mathbb E[r_T^{-1}] < \infty$.
Now, consider the variable $r_T^{-1} = \frac{N+1}{K_N+1}$, $K_N=\sum_{j=1}^N \mathds{1}(\tau_j \ge T)\sim$ Binomial$(N,p_T)$. We have $\mathbb E\left[ \frac{1}{K_N + 1} \right] = \frac{1}{(N+1)p_T} \sum_{j=1}^{N+1} \binom{N+1}{j} p_T^j (1-p_T)^{(N+1)-j}=\frac{1 - (1-p_T)^{N+1}}{(N+1)p_T}$, and 
$$\mathbb E[r_T^{-1}] = (N+1) \mathbb E\left[ \frac{1}{K_N + 1} \right] = \frac{1 - (1-p_T)^{N+1}}{p_T}\leq \frac{1}{p_T}.$$
This bound $\frac{1}{p_T}$ is a constant that does not depend on $N$. Hence,
 $r_t^{-s}$ is uniformly integrable.
\end{proof}
\begin{lemma}
\label{lem:r_tau}
Suppose, $\tau_1,\cdots,\tau_N$ are i.i.d. random variables. Define, $r_t=\frac{1+\sum_{i=1}^N\mathds{1}(\tau_i\geq t)}{N+1}$, for $t\in\N.$  If each $\tau_j $ is stochastically larger than a random variable $\tau$, then $r_{\tau}$ is a p-value.
\end{lemma}
\begin{proof}
Suppose $\tau_{N+1}$ is independent and identically distributed (i.i.d.) with  $\tau_1,\cdots,\tau_N$.
By Theorem 1.A.2 of \cite{shaked2007stochastic}, $\tau_{N+1}$ is stochastically larger than $\tau$ implies that there exist a random variable $Z$ and functions $\psi_{1}$ and $\psi_{2}$ such that $\psi_{1}(z)\leq \psi_{2}(z)$ for all $z$ and 
$\tau\stackrel{d}{=}\psi_{1}(Z)$ and 
$\tau_{N+1}\stackrel{d}{=}\psi_{2}(Z)$. 
By direct calculation,
    \begin{align*}
        \P[r_{\tau}  \leq (1+k)/(N+1)]=\P[\sum_{j=1}^N \mathds{1}(\tau_j  \ge \tau)\leq k]&=\P[\sum_{j=1}^N \mathds{1}(\tau_j  \ge \psi_{1}(Z))\leq k]\\
        &\leq \P[\sum_{j=1}^N \mathds{1}(\tau_j  \ge \psi_{2}(Z))\leq k]\\
        &=\P[\sum_{j=1}^N \mathds{1}(\tau_j  \ge \tau _{N+1})\leq k]\leq \frac{1+k}{N+1},
    \end{align*}
    for all $k=0,1,\cdots,N.$
Therefore, $r_{\tau}$ is a p-value.
\end{proof}

\section{Subtleties regarding Subroutine \texorpdfstring{$\mathcal L$}{Lg} in \Cref{algo:1-exact}}
\label{sec:subroutine-L}

The problem with $L=\infty$ is that the stopping time may never be reached for some sampled sequence. One can set $L=\infty$ if $\mathcal A$ is known to satisfy
    $\P_T(\tau<\infty)=1$, for all $T\in\N$
 (if there is a change, $\mathcal A$ stops almost surely). This is a very mild assumption because any reasonable change detector should eventually raise an alarm if there is a change; indeed, for example, this holds for all algorithms with finite ARL. If $\mathcal A$ is not known to satisfy this condition, one can simply pick $L < \infty$ (The value of 
$L$ does not need to be a fixed constant; it can be set adaptively depending on the observed data sequence, for instance, one can choose $L$ to be $2\tau$.)
If $L<\infty$ and the $j$th sequence stopped because $\mathcal A$ raised an alarm at or before $L$, then we define $M^j_{t,L}$ as usual. If $\mathcal A$ had not stopped by time $L$, we then define $M^j_{t,L} = \infty$ to ensure a lower bound on the original test statistic (if we could compute it). For $L<\infty$,
\begin{equation}
\label{eq:Mt-L}
    M^j_{t,L} := \begin{cases}
    -\infty, &\text{if }t>\tau_j^\prime,\\
        M_t^{(\tau_j^\prime)}(X_1^j,\cdots,X_{\tau_j^\prime}^j), &\text{if } t\leq\tau_j^\prime\leq L,\\
        \infty, &\text{if }t\leq\tau_j^\prime \text{ and } \tau_j^\prime > L.\\
    \end{cases}
\end{equation}
\vspace{-0.1cm}
We still need a definition of $M^j_{t,\infty}$ in the case that $L=\infty$ and $\tau_j =\infty$. We define
$ M^j_{t,\infty} :=
M_t^{(\tau_j^\prime)}(X_1^j,\cdots,X_{\tau_j^\prime}^j), \text{if } t\leq\tau_j^\prime<\infty, $ and $
        M^j_{t,\infty} := -\infty, \text{if } t>\tau_j^\prime \text{ or }\tau_j^\prime =\infty.$


\section{Computationally and statistically efficient subroutines in \Cref{algo:comp-post}}
\label{sec:subroutines-comp-post}

Note that the test statistic $M_t$ for testing $\mathcal{H}_{0,t}$ reduces to
     \begin{equation}
    \label{eq:test-stat-comp-post}
        M_t=\begin{cases}
             \frac{L_{\hat{T}}(\hat \theta_{1,\hat{T}:\tau};X_1,\cdots,X_\tau)}{L_t(\hat \theta_{1,t:\tau};X_1,\cdots,X_\tau)}, &\text{if } t\leq\tau<\infty, \\
             -\infty, &\text{if }t>\tau, \text{ or }\tau =\infty.\\

        \end{cases}
    \end{equation}
For computational and statistical efficiency, we generate sequences from all $t\in\{1,\cdots,\tau\}$ and $\theta\in\mathcal{S}^\prime_{\tau-t+1}$ simultaneously using a single shared source of randomness. To achieve this, we assume that we have a fixed element $\theta^{\prime\prime}\in\Theta$ in the parameter space and a function $G$, such that $X\sim F_{\theta^{\prime\prime}}$ implies $G(X,\theta)\sim F_\theta, \forall \theta\in\Theta$. Therefore, for each $j=1,\cdots,B$, we generate only one data sequence $\{X_n^j\}_n\stackrel{i.i.d.}{\sim}F_{\theta^{\prime\prime}}$, and then for each $t\in\{1,\cdots,\tau\}$ and $\theta\in\mathcal{S}^\prime_{\tau-t+1}$, we obtain sequences $\{Y_n^t(X_{n}^j,\theta)\}_n$ (having joint distribution $\P_{\theta_0,t,\theta}$) by applying the function  $G(.,\theta_0)$ to $X_{1}^j,\cdots,X_{t-1}^j$ and $G(.,\theta)$ to $X_{t}^j,X_{t+1}^j,\cdots$, i.e.,
\begin{equation}
    Y_n^t(x,\theta):=\begin{cases}
        G(x,\theta_0), &\text{if } n<t\\
        G(x,\theta), &\text{ otherwise. }
    \end{cases}
\end{equation}
Let us provide a few examples of the function $G$. If the distributions in the parametric family are continuous, the inverse distribution function $F_{\theta}^{-1}$ exists, and it is easy to verify that $G(x, \theta) = F_{\theta}^{-1}(F_{\theta'_t}(x))$ works. In the case of a location family, a simple transformation is $G(x, \theta) = x+\theta-\theta^{\prime\prime}$ and similarly, for scale family, we have 
$G(x, \theta) = x\times\theta/\theta^{\prime\prime}$.

Let us now fix $t\in\{1,\cdots,\tau\}$ and consider the test statistic $M_t$ of the form in \eqref{eq:test-stat-nonpara}, with some $R^{(t)}_{n}=R^{(t)}_{n}(\{X_n\}_n)$ and $S^{(t)}_{n}=S^{(t)}_{n}(\{X_n\}_n)$.
Computing the detection time $\tau_{j,t}^\theta$ for each $\theta\in\mathcal{S}^\prime_{\tau-t+1}$, and finding whether $\text{there exists }\theta \in \mathcal{S}^\prime_{\tau-t+1}$ such that $M_t\leq\operatorname{Quantile}(1-\hat\alpha_{t};M_t,\{M_{t,L}^j(\theta)\}_{j=1}^B)$ are both computationally infeasible, unless $\mathcal{S}^\prime_{\tau-t+1}$ is a finite set.
However, it is often straightforward to compute  $t_1^j= \sup_{\theta\in \mathcal{S}^\prime_{\tau-t+1}}\tau^\theta_{j,t}$.
Also, we can often derive a suitable upper bound on the threshold, which is computationally and statistically efficient. We consider the following upper bound:
\begin{equation}
\label{eq:vt}
    V_{t,L}^j:=\begin{cases}
            \displaystyle\max\left\{ \max_{t\leq n\leq t_1^j}\sup_{\theta\in\mathcal{S}^\prime_{\tau-t+1}} R^{(t)}_{n}\left(\{Y_n^t(X_{n}^j,\theta)\}_n\right), \max_{1\leq n\leq t-1}\sup_{\theta\in\mathcal{S}^\prime_{\tau-t+1}}S^{(t)}_{n}\left(\{Y_n^t(X_{n}^j,\theta)\}_n\right)\right\},\\
            \hspace{12cm}\text{if } t\leq t_1^j\leq L;\\
             -\infty, \text{ if } t>t_1^j;\\
             \infty, \text{ if } t\leq t_1^j \text{ and } t_1^j>L.
        \end{cases}
\end{equation}
 We consider the following confidence set:
 \begin{equation}
 \label{set-gaussian-known-pre}
     \mathcal{C}^\prime_1=\{t\in\{1,\cdots,\tau\}: M_t\leq\operatorname{Quantile}(1-\hat\alpha_{t};M_t,V_{t,L}^1,\cdots,V_{t,L}^B\} \text{ and}
 \end{equation}
We observe that $\mathcal{C}_1^\prime\supseteq \mathcal{C}_1$, where $\mathcal{C}_1$ is the set produced by \Cref{algo:comp-post}, since
 \begin{align*}
 \operatorname{Quantile}(1-\hat\alpha_{t};M_t,V_{t,L}^1,\cdots,V_{t,L}^B\}&\geq  \operatorname{Quantile}(1-\hat \alpha_{t};M_t,\{\sup_{\theta\in\mathcal{S}^\prime_{\tau-t+1}}M_{t,L}^j(\theta)\}_{j=1}^B)\\
 &\geq \sup_{\theta\in\mathcal{S}^\prime_{\tau-t+1}}\operatorname{Quantile}(1-\hat \alpha_{t};M_t,\{M_{t,L}^j(\theta)\}_{j=1}^B).
 \end{align*}
  Thus,  we obtain the following coverage guarantee as an immediate corollary of 
\Cref{thm:coverage-comp-post}.

 \begin{corollary}
 $\mathcal{C}^\prime_1$ defined in \eqref{set-gaussian-known-pre}  satisfies the same coverage guarantees as $\mathcal{C}_1$ in \Cref{thm:coverage-comp-post}.

     \end{corollary}

     Next, we present a concrete example of the Gaussian mean shift problem to clarify our method. 
     
\subsection{A concrete example: Gaussian mean shift problem}
\label{sec:gaussian-example-comp-post}

Let the pre-change distribution $F_0$  be $N(\theta_0,1)$ and the post-change model be $\mathcal{P}_1=\{N(\theta,1): \theta\in\Theta_1\}$, with $\Theta_1=\{\theta:\theta\geq\theta_1\}$ such that $\theta_1>\theta_0$.


In this setting, we have the following example of a $1-\beta$ confidence sequence \citep{howard2021time} for $\theta_1$
\begin{equation}
\label{eq:conf-seq-gaussian}
    \mathcal{S}^\prime_{\tau-t+1}=\left[\bar{X}_{t:\tau}-\frac{s_{\tau-t+1}(\beta)}{\sqrt{\tau-t+1}},\bar{X}_{t:\tau}+\frac{s_{\tau-t+1}(\beta)}{\sqrt{{\tau-t+1}}}\right]\cap \Theta_1,
\end{equation}
where $\bar{X}_{t:\tau}=\frac{1}{\tau-t+1}\sum_{i=t}^\tau X_i$ and $s_n(\beta)=\sqrt{\log\log (2n)+0.72 \log(10.4/\beta)}$.

We fix $t\in\{1,\cdots,\tau\}$. Note that we can consider $\theta^{\prime\prime}=\theta_0$ and $G(x,\theta)=x+\theta-\theta_0$. Therefore, only need to generate $\{X_n^j\}_{n\in\N}\stackrel{iid}{\sim}N(\theta_0,1)$ for $j=1,\cdots,B$. For $\theta\in\mathcal{S}^\prime_{\tau-t+1}$, 
\begin{equation}
    Y_n^t(X_n^j,\theta):=\begin{cases}
        \epsilon_n^j+\theta_0, \text{ if } n<t,\\
        \epsilon_n^j+\theta, \text{ otherwise, }\end{cases}
\end{equation}
where $\epsilon_n^j:=X_n^j-\theta_0\sim N(0,1)$. We consider the test statistic $M_{t}$, as defined in \eqref{eq:test-stat-nonpara}, with $R_n^{(t)}=\prod_{i=t}^{n} \frac{f_{\theta_0}(X_i)}{f_{\theta_1}(X_i)},$ and $
   S_n^{(t)}= \int_{\theta\in\Theta_1}\prod_{i=n}^{t-1} \frac{f_\theta(X_i)}{f_{\theta_0}(X_i)}w(\theta)d\theta$, with $w\geq 0$ being  some weight function such that $\int_{\theta\in\Theta_1}w(\theta)d\theta=1$.
Let, $\tau^\theta_j$ is the detection time for the sequence $\{
Y_n^t(X_n^j,\theta)\}_n$ and $M_{t,\infty}^j(\theta)=M_t(Y_1^t(X_1^j,\theta),\cdots,Y_{\tau^\theta_j}^t(X_{\tau^\theta_j}^j,\theta))$. Suppose, we can find $t_1^j\in\N$, such that $t_1^j=\sup_{\theta\in \mathcal{S}^\prime_{\tau-t+1}}\tau^\theta_j$ (these are easy to compute for specific detection algorithms, such as CUSUM or SR type detectors, as we explain in \Cref{max-min-tau}) and they are finite. So, we can have $L=\infty$. Note that, here we have 
$$\sup_{\theta\in\mathcal{S}^\prime_{\tau-t+1}} R^{(t)}_{n}\left(\{Y_n^t(X_{n}^j,\theta)\}_n\right)=\exp\left((\theta_0-\theta_1)\sum_{i=t}^n(\epsilon_n^j+\theta^\prime_{t,\min} )+\frac{n-t+1}{2}(\theta_1^2-\theta_0^2)\right),$$ 
and
$$\sup_{\theta^\prime\in\mathcal{S}^\prime_{\tau-t+1}} S^{(t)}_{n}\left(\{Y_n^t(X_{n}^j,\theta^\prime)\}_n\right)
=
\displaystyle\underset{\theta\in\Theta_1}{\int}\exp\left((\theta-\theta_0)\sum_{i=n}^{t-1}(\epsilon_n^j+\theta^\prime_{t,\min} )+\frac{t-n}{2}(\theta_0^2-\theta^2)\right)w(\theta)d\theta,$$ 
where $\theta^\prime_{t,\min}=\inf\mathcal{S}^\prime_{\tau-t+1}$.
 And hence, $V_{t,\infty}^j$ in \eqref{eq:vt} can be computed easily. 
 
 \begin{remark}
\label{max-min-tau}     
For weighted CUSUM or SR detectors, finding $t_1^j$ is often straightforward. For some weight $w(\theta)$ such that $\int_{\theta\in\Theta_1} w(\theta)=1$, these are defined as
\begin{equation}
\label{eq:wCUSUM}
    \tau_{\text{wcusum}}=\inf\left\{n\in\mathbb N:\max_{1\leq j\leq n}\int_{\theta\in\Theta_1}\prod_{i=j}^n\frac{f_{\theta}(X_i)}{f_{\theta_0}(X_i)}w(\theta)\geq A\right\}
\end{equation}
\begin{equation}
\label{eq:wsr}
 \text{ and }~   \tau_{\text{wsr}}=\inf\left\{n\in\mathbb N:\sum_{j=1}^n\int_{\theta\in\Theta_1}\prod_{i=j}^n\frac{f_{\theta}(X_i)}{f_{\theta_0}(X_i)}w(\theta)\geq A\right\}.
\end{equation}
For example, consider $\mathcal{P}_1=\{N(\theta,1):\theta\geq \theta_1\}$, for some $\theta_1\geq \theta_0$.
Observe that for all $\theta\in \Theta_1$, $\prod_{i=j}^n\frac{f_{\theta}(X_i)}{f_{\theta_0}(X_i)}=\exp\{\frac{1}{2}(\theta-\theta_0)\times\sum_{i=j}^n(2X_i-\theta_0-\theta)\}$ increases as $X_i$'s increases and hence, both $\tau_{\text{wcusum}}$ and $\tau_{\text{wsr}}$ decreases as $X_i$'s increases. 
Therefore, $t_1^j=\sup_{\theta\in \mathcal{S}^\prime_{\tau-t+1}}\tau^\theta_j=\tau^{\inf \mathcal{S}^\prime_{\tau-t+1}}_j$. 
Similarly, for $\mathcal{P}_1=\{N(\theta,1):\theta\leq \theta_1\}$, for some $\theta_1\leq \theta_0$, we have $t_1^j=\tau^{\sup \mathcal{S}^\prime_{\tau-t+1}}_j$. 
One can check that similar results extend to other parametric families, such as the Gaussian/Laplace/exponential scale change problems with weighted CUSUM and SR detectors, as well as the PFA-controlling change detectors.
 \end{remark}

 The Gaussian mean shift example we provided serves as a concrete and illustrative case to clarify the methodology. It is worth noting that $C^\prime$, as defined in \eqref{set-gaussian-known-pre}, is practically implementable as long as the following relatively weak conditions are met: 
 \begin{enumerate}
     \item The coupling function $G$ is available (which is always the case when distributions are continuous).
     \item  We can construct a confidence sequence for the parameter of interest (which is always achievable as long as one can construct a confidence interval for the parameter).  
     \item The quantities $t_j^1$ are computable (either exactly or numerically). 
 \end{enumerate}
For instance, one can easily verify that this approach applies to Gaussian, Laplace, or exponential scale change problems.

\section{Computationally and statistically efficient subroutines in Algorithm 3}
\label{sec:subroutines-comp}

As in \Cref{sec:subroutines-comp-post}, we assume that we have a fixed element $\theta^{\prime\prime}$ in the parameter space and a function $G$, such that $X\sim F_{\theta^{\prime\prime}}$ implies $G(X,\theta)\sim F_\theta, \forall \theta\in\Theta_0\cup\Theta_1$. Therefore, for each $j=1,\cdots,B$, we generate data sequence $X_1^j,X_2^j\cdots\stackrel{iid}{\sim}F_{\theta^{\prime\prime}}$ only once. Then, for each $t\in\{1,\cdots,\tau\}$ and $(\theta\times\theta^\prime)\in\mathcal{S}_{t-1}\times\mathcal{S}^\prime_{\tau-t+1}$, we obtain sequences $\{Y_n^t(X_n^j,\theta,\theta^\prime)\}_n$ having joint distribution $\P_{\theta,t,\theta^\prime},$ by applying the function $G(.,\theta)$ to $X_{1}^j,\cdots,X_{t-1}^j$ and $G(.,\theta^\prime)$ to $X_{t}^j,X_{t+1}^j,\cdots$, i.e.,
\begin{equation}
    Y_n^t(x,\theta,\theta^\prime):=\begin{cases}
        G(x,\theta), &\text{if } n<t\\
        G(x,\theta^\prime), &\text{ otherwise. }
    \end{cases}
\end{equation}

Let us now fix $t\in\{1,\cdots,\tau\}$ and consider the test statistic $M_t$ of the form in \eqref{eq:test-stat-nonpara}, with some $R^{(t)}_{n}=R^{(t)}_{n}(\{X_n\}_n)$ and $S^{(t)}_{n}=S^{(t)}_{n}(\{X_n\}_n)$..
Computing the detection time $\tau^j_{\theta,\theta^\prime}$ for each $(\theta, \theta^\prime)\in \mathcal{S}_{t-1}\times \mathcal{S}^\prime_{\tau-t+1}$, and finding whether $\text{there exists }(\theta, \theta^\prime)\in \mathcal{S}_{t-1}\times \mathcal{S}^\prime_{\tau-t+1}$ such that $M_t\leq\operatorname{Quantile}(1-\hat\alpha_{t};M_t,\{M_t^j(\theta,\theta^\prime)\}_{j=1}^B)$ are both computationally infeasible, unless $\mathcal{S}_{t-1}$ and $\mathcal{S}^\prime_{\tau-t+1}$ are a finite sets. 

However, it is often straightforward to compute  $t_1^j=\displaystyle\sup_{(\theta, \theta^\prime)\in \mathcal{S}_{t-1}\times \mathcal{S}^\prime_{\tau-t+1}}\tau_{\theta,\theta^\prime}^j$.
Also, we can often derive a suitable upper bound on the threshold, which does not require any approximation and is computationally and statistically efficient. E.g., we can bound it by $\operatorname{Quantile}(1-\hat\alpha_{t};M_t,\{\sup_{(\theta, \theta^\prime)\in \mathcal{S}_{t-1}\times \mathcal{S}^\prime_{\tau-t+1}}M_{t,L}^j(\theta,\theta^\prime)\}_{j=1}^B)$.
Again, we bound $\sup_{(\theta, \theta^\prime)\in \mathcal{S}_{t-1}\times \mathcal{S}^\prime_{\tau-t+1}}M_{t,L}^j(\theta,\theta^\prime)$ by $U_{t,L}^j$, where 
\begin{equation}
\label{eq:ut}
    U_{t,L}^j=\begin{cases}
            \displaystyle\max\bigg\{ \max_{t\leq n\leq t_1^j}\sup_{(\theta, \theta^\prime)\in \mathcal{S}_{t-1}\times \mathcal{S}^\prime_{\tau-t+1}} R^{(t)}_{n}\left(\{Y_n^t(X_{n}^j,\theta,\theta^\prime)\}_n\right)
    \\\hspace{1cm}+\displaystyle\max_{1\leq n\leq t-1}\sup_{(\theta, \theta^\prime)\in \mathcal{S}_{t-1}\times \mathcal{S}^\prime_{\tau-t+1}}S^{(t)}_{n}\left(\{Y_n^t(X_{n}^j,\theta,\theta^\prime)\}_n\right)\bigg\},
            \text{ if } t\leq t_1^j\leq L;\\
             -\infty, \text{ if } t>t_1^j;\\
             \infty, \text{ if } t\leq t_1^j \text{ and } t_1^j>L.
        \end{cases}
\end{equation}
 We consider the following confidence set: 
 \begin{equation}
 \label{set-eff-comp}
     \mathcal{C}^{\prime\prime}_1=\{t\in\{2,\cdots,\tau\}: M_t\leq\operatorname{Quantile}(1-\hat \alpha_{t};M_t,U_{t,L}^1,\cdots,U_{t,L}^B)\} \text{ and}
 \end{equation}
Observe that $\mathcal{C}^{\prime\prime}_1\supseteq \mathcal{C}_1$, where $\mathcal{C}_1$ is the set produced by \Cref{algo:comp}, since
\begin{align*}
    \operatorname{Quantile}(1-\hat\alpha_{t};M_t,\{U_{t,L}^j\}_{j=1}^B)\}
 &\geq \operatorname{Quantile}(1-\hat\alpha_{t};M_t,\{\sup_{(\theta, \theta^\prime)\in \mathcal{S}_{t-1}\times \mathcal{S}^\prime_{\tau-t+1}}M_{t,L}^j(\theta,\theta^\prime)\}_{j=1}^B)\\
 &\geq \sup_{(\theta, \theta^\prime)\in \mathcal{S}_{t-1}\times \mathcal{S}^\prime_{\tau-t+1}}\operatorname{Quantile}(1-\hat\alpha_{t};M_t,\{M_{t,L}^j(\theta,\theta^\prime)\}_{j=1}^B).
\end{align*}
 Thus, we obtain the following coverage guarantee as an immediate corollary of 
\Cref{thm:coverage-comp}.

 \begin{corollary}
 $\mathcal{C}^{\prime\prime}_1$  defined in \eqref{set-eff-comp} satisfies the same coverage guarantees as $\mathcal{C}_1$ in \Cref{thm:coverage-comp}.
 \end{corollary}
 In the next subsection, we present a concrete example of the Gaussian mean shift problem to clarify our method. 
  
\subsection{A concrete example: Gaussian mean shift problem}
\label{sec:gaussian-example-comp}

We consider the most common changepoint model, where pre-change and post-change models are $\mathcal{P}_i=\{N(\theta,1): \theta\in\Theta_i\}$, for $i=0,1$, with $\Theta_0=\{\theta:\theta\leq\theta_0\}$ and $\Theta_1=\{\theta:\theta\geq\theta_1\}$, for some $\theta_1>\theta_0$.
Then, we have the following $1-\gamma$ confidence interval for $\theta_0$ for each $t\geq 2$:
\begin{equation}
    \mathcal{S}_{t-1}=\left[\bar{X}_{1:t-1}-\frac{q_{\gamma/2}}{\sqrt{t-1}},\bar{X}_{1:t-1}+\frac{q_{\gamma/2}}{\sqrt{{t-1}}}\right]\cap \Theta_0,
\end{equation}
and the $1-\beta$ confidence sequence for $\theta_1$, as defined in \eqref{eq:conf-seq-gaussian}.

We fix $t\in\{2,\cdots,\tau\}$. Note that we can consider $\theta^{\prime\prime}=0$ and $G(x,\theta)=x+\theta$. Therefore, only need to generate $\{\epsilon_n^j\}_{n\in\N}\stackrel{iid}{\sim}N(0,1)$ for $j=1,\cdots,B$. For $\theta\in\mathcal{S}_{t-1},\theta^\prime\in\mathcal{S}^\prime_{\tau-t+1}$, define 
\begin{equation}
    Y_n^{t,j}(\theta,\theta^\prime) =\begin{cases}
        \theta+\epsilon_n^j, \text{ if } n<t\\
        \theta^\prime+\epsilon_n^j, \text{ otherwise. }
    \end{cases}
\end{equation}
We consider the test statistic $M_t$, as defined in \eqref{eq:test-stat-nonpara}, with $R_n^{(t)}=\int_{\theta\in\Theta_0}\prod_{i=t}^{n} \frac{f_{\theta}(X_i)}{f_{\theta_1}(X_i)}w^\prime(\theta)d\theta,$  and $
   S_n^{(t)}= \int_{\theta\in\Theta_1}\prod_{i=n}^{t-1} \frac{f_\theta(X_i)}{f_{\theta_0}(X_i)}w(\theta)d\theta$, with $w, w^\prime\geq 0$ being  some weight functions such that $\int_{\theta\in\Theta_1}w(\theta)d\theta=1$ and $\int_{\theta\in\Theta_0}w^\prime(\theta)d\theta=1$. Note that $t_1^j$ is easy to compute for some CUSUM/SR type detection algorithms (see \Cref{max-min-tau-comp})
 and they are finite. So, we can have $L=\infty$. Note that, here the term $\sup_{(\theta, \theta^\prime)\in \mathcal{S}_{t-1}\times \mathcal{S}^\prime_{\tau-t+1}} R^{(t)}_{n}\left(\{Y_n^{t,j}(\theta,\theta^\prime) \}_n\right)$ simplifies to 
$$
\displaystyle\underset{\theta\in\Theta_0}{\int}\exp\left((\theta-\theta_1)\sum_{i=t}^n(\epsilon_n^j+\theta^\prime_{t,\min} )+\frac{n-t+1}{2}(\theta_1^2-\theta^2)\right)w^\prime(\theta)d\theta,$$ 
and the term $\sup_{(\theta, \theta^\prime)\in \mathcal{S}_{t-1}\times \mathcal{S}^\prime_{\tau-t+1}} S^{(t)}_{n}\left(\{Y_n^{t,j}(\theta,\theta^\prime) \}_n\right)$ simplifies to
$$
\displaystyle\underset{\theta\in\Theta_1}{\int}\exp\left((\theta-\theta_0)\sum_{i=n}^{t-1}(\epsilon_n^j+\theta^\prime_{t,\min} )+\frac{t-n}{2}(\theta_0^2-\theta^2)\right)w(\theta)d\theta,$$ 
where $\theta^\prime_{t,\min}=\inf\mathcal{S}^\prime_{\tau-t+1}$. And hence, $U_{t,\infty}^j$ in \eqref{eq:ut} can be computed easily.

 \begin{remark}
\label{max-min-tau-comp}     
For weighted CUSUM or SR detectors, finding $t_1^j$ is often straightforward. So, for some weight $w(\theta)$ such that $\int_{\theta\in\Theta_1} w(\theta)=1$, the weighted CUSUM or SR detectors (with the Reverse Information Projection (RIPr) \cite{larsson2024numeraire,grunwald2024safe,lardy2023universal} being the representative of the pre-change class) are defined as
\begin{equation}
\label{eq:wCUSUM-ripr}
    \tau_{\text{wcs-ripr}}=\inf\left\{n\in\mathbb N:\max_{1\leq j\leq n}\int_{\theta\in\Theta_1}\prod_{i=j}^n\frac{f_{\theta}(X_i)}{f_{\theta}^*(X_i)}w(\theta)\geq A\right\},
\end{equation}
\begin{equation}
\label{eq:wsr-ripr}
  \tau_{\text{wsr-ripr}}=\inf\left\{n\in\mathbb N:\sum_{j=1}^n\int_{\theta\in\Theta_1}\prod_{i=j}^n\frac{f_{\theta}(X_i)}{f_{\theta}^*(X_i)}w(\theta)\geq A\right\},
\end{equation}
where $f_{\theta}^*$ is the RIPr of $f_{\theta}$ on $\mathcal{P}_0$.
For example, consider $\mathcal{P}_0=\{N(\theta,1):\theta\leq a\},\mathcal{P}_1=\{N(\theta,1):\theta\geq b\}$, for some $b\geq a$. Observe that for all $\theta\in \Theta_1, f_{\theta}^*=f_a$, and $\prod_{i=j}^n\frac{f_{\theta}(X_i)}{f_{\theta}^*(X_i)}=\exp\{\frac{1}{2}(\theta-a)\times\sum_{i=j}^n(2X_i-a-\theta)\}$ increses as $X_i$'s increases and hence, both $\tau_{\text{wcusum}}$ and $\tau_{\text{wsr}}$ decreases as $X_i$'s increases.
Therefore, 
$$t_1^j=\sup_{(\theta, \theta^\prime)\in \mathcal{S}_{t-1}\times \mathcal{S}^\prime_{\tau-t+1}}\tau_{\theta, \theta^\prime}^j=\tau_{\inf \mathcal{S}_{t-1},\inf \mathcal{S}^\prime_{\tau-t+1}}^j.$$
 Similarly, one can check that for $\mathcal{P}_0=\{N(\theta,1):\theta\geq a\},\mathcal{P}_1=\{N(\theta,1):\theta\leq b\}$, for $b\leq a$,
$$t_1^j=\sup_{(\theta, \theta^\prime)\in \mathcal{S}_{t-1}\times \mathcal{S}^\prime_{\tau-t+1}}\tau_{\theta, \theta^\prime}^j=\tau_{\sup \mathcal{S}_{t-1},\sup \mathcal{S}^\prime_{\tau-t+1}}^j.$$
\end{remark}

 The Gaussian example we provided serves as a concrete and illustrative case to clarify the methodology.  
 Note that $\mathcal{C}^{\prime\prime}$ in \eqref{set-eff-comp} is computable if similar conditions as listed in \Cref{sec:gaussian-example-comp-post} are met.
For instance, one can check that these conditions hold for analogous Laplace and exponential scale distributions.

\section{Experiments for confidence intervals for pre/post change parameters}
\label{sec:expt-paramter}

\subsection{Known pre-change but unknown post-change}

\Cref{tab:comp-post-theta1} contains conditional and marginal (empirical) coverages, (conditional) average of the sizes and of the confidence intervals (i.e., the upper and lower boundaries) (defined in \eqref{eq:ci-theta1}) for $\theta_1$ across $500$ independent runs with $\alpha=\beta=\eta=0.05$. The data generation and change detection procedure remain the same as in the experiment described in Setting II of \Cref{sec:expt}. Across all settings, the observed conditional coverage consistently exceeds the theoretical lower bound of $0.85$, indicating that our method is reliable but conservative.

\begin{table}[!ht]
    \centering
    \caption{Confidence interval for $\theta_1$: known pre-change but unknown post-change. The pre-change parameter $\theta_0=0$ is known, the true (unknown) post-change parameter is $\theta_1=1$, while the algorithm only knows $\theta_1\in\Theta_1$, which we mention in the second column of the table. (Setting II)
    }
\label{tab:comp-post-theta1}
    \resizebox{\linewidth}{!}{
    \begin{tabular}{cc|ccccc}
    \toprule
    \addlinespace
 T  & {$\Theta_1$}  & \specialcell{Conditional coverage} & \specialcell{Marginal coverage} & Conditional Length & Conditional CI \\
    \midrule
\addlinespace 
100 &  $[0.75,\infty)$   & 0.98  & 0.98 & 1.37 & [0.76, 2.13]\\
100 & $[0.9,\infty)$  & 0.98  & 0.96 & 1.23 & [0.90, 2.16] \\
500 & $[0.75,\infty)$ &  0.99  & 0.93 & 1.48 & [0.75, 2.23] \\
500 &  $[0.9,\infty)$ &  0.98  & 0.91 & 1.31 & [0.90, 2.21] \\
\bottomrule
 \end{tabular}}
 \end{table}

\subsection{Unknown pre- and post-change}
\Cref{tab:comp-theta} contains conditional and marginal (empirical) coverages, (conditional) average of the sizes and of the confidence intervals (i.e., the upper and lower boundaries), and \eqref{eq:ci-theta1}), for $\theta_0$ and $\theta_1$ across $500$ independent runs with $\alpha=\beta=\eta_0=\eta_1=0.05$. he data generation and change detection procedure remain the same as in the experiment described in Setting III of \Cref{sec:expt}. Across all settings, the observed conditional coverage consistently exceeds the theoretical lower bound of $0.80$.

\begin{table}[!ht]
    \centering
    \caption{Confidence intervals for $\theta_0,\theta_1$ with target conditional coverages $0.85$(Setting III): unknown pre- and post-change. The true (unknown) values of the parameters are $\theta_0=0$ and $\theta_1=1$, while the algorithm only knows $\theta_i\in\Theta_i, i=0,1$, which are mentioned in the second and third columns of the table. 
    }
\label{tab:comp-theta}
    \resizebox{\linewidth}{!}{
    \begin{tabular}{ccc|cc|cc|cc|cc}
    \toprule
    \addlinespace
 \multirow{2}{*}{T}  & \multirow{2}{*}{\specialcell{$ \Theta_0$}} & \multirow{2}{*}{\specialcell{$\Theta_1$}}  & \multicolumn{2}{c}{\specialcell{Cond. coverage}} & \multicolumn{2}{c}{\specialcell{Marg. coverage}} & \multicolumn{2}{c}{Cond. length} & \multicolumn{2}{c}{Cond. CI} \\
 & & & $\theta_0$ & $\theta_1$ & $\theta_0$ & $\theta_1$ & $\theta_0$ & $\theta_1$ & $\theta_0$ & $\theta_1$\\
    \midrule
\addlinespace 
100 & $(\infty, 0.25]$ &  $[0.75,\infty)$   & 1.00  & 0.99 &  1.00  & 0.99 & 0.55 & 1.19 & [-0.32, 0.23] & [0.76, 1.95]\\
100 & $(\infty,  0.1]$ & $[0.9,\infty)$   & 0.98  & 0.99 & 0.98  & 0.98 & 0.61 &  1.16 & [-0.51, 0.10] & [0.90, 2.06]\\
500 &$(\infty,  0.25]$ & $[0.75,\infty)$ &  0.99  & 0.99 & 0.93  & 0.93 & 0.21 & 1.22 & [-0.10, 0.11] & [0.75, 1.97]\\
500 & $(\infty,  0.1]$ & $[0.9,\infty)$ &  0.99  & 0.99 & 0.94  & 0.94 & 0.19 & 1.24 & [-0.09, 0.10] & [0.90, 2.14]\\
\bottomrule
 \end{tabular}}
 \end{table}

\subsection{Comparison with Wu \cite{WU20063625}}
 Next, we compare our confidence interval for the post-change parameter with that proposed by  \cite{WU20063625} and discussed in \cite{wu2007inference}. We implement Wu's interval as discussed in Chapter 4.3 of \cite{wu2007inference} (or Section 3 of \cite{WU20063625}), replacing $\theta$ by $\hat\theta$ in the bias and variance terms, which is given by $\hat{\theta}-\frac{4}{5\sqrt{d\hat{\theta}(\tau-\hat{\nu})}}\pm z_{\alpha/2}\frac{\sqrt{1-{13}/(16d\hat{\theta})}}{\sqrt{\tau-\hat{\nu}}}$, where $\hat{\theta}=\frac{T_\tau}{\tau-\hat{\nu}}$, $\tau=\tau^\prime_{\text{CUSUM}}$, and $z_{\alpha^\prime/2}$ is the $\alpha^\prime/2$ quantile of the standard normal distribution. We use CUSUM \cite{page1954continuous} (the same change detector as in the second last experiment of \Cref{sec:expt}), with $d=10$ for both methods. For our method, we assume the post-change space to be $\Theta_1=[0,\infty)$ with $\alpha=\beta=\eta=0.05$. We report the average size and the conditional and unconditional coverage rates of both methods across $500$ independent runs in \Cref{tab:ci-comparison-theta1}, with a target conditional coverage of $\alpha^\prime=0.85$. The results show that the existing method \cite{WU20063625}, which is based on asymptotic approximations, undercovers $\theta_1$ (both conditionally and marginally). At the same time, our finite-sample confidence intervals are conservative --- achieving higher-than-nominal (conditional) coverage at the expense of larger (conditional) length.

\begin{table}[ht]
\centering
\caption{Comparison of the confidence sets for $\theta_1$ (target conditional coverage $0.85$) with Wu \cite{WU20063625}  (Setting II)}
\label{tab:ci-comparison-theta1}
\resizebox{\linewidth}{!}{
\begin{tabular}{ccc|
                cc
                cc
                cc
                cc}
\toprule
$F_0$ & $F_1$ & $T$ &
\multicolumn{2}{c}{Conditional Length} &
\multicolumn{2}{c}{Conditional Coverage} &
\multicolumn{2}{c}{Marginal Coverage} \\
\cmidrule(lr){4-5} \cmidrule(lr){6-7} \cmidrule(lr){8-9} &
& &  Ours &  Wu (2006) &  Ours &  Wu (2006) &  Ours &  Wu (2006) \\
\midrule
${N}(-0.25,1)$ & ${N}(0.25,1)$ & 100 & 1.46 & 0.47 & 0.99 & 0.77 & 0.98 & 0.74  \\
${N}(-0.3,1)$ & ${N}(0.3,1)$ & 100 & 1.43& 0.51 & 0.98 & 0.76 & 0.98 & 0.75 \\
${N}(-0.25,1)$ & ${N}(0.25,1)$ & 500 & 1.48 & 0.49 & 0.99 & 0.78 & 0.91 & 0.64 \\
${N}(-0.3,1)$ & ${N}(0.3,1)$ & 500 & 1.49 & 0.51 &  0.99 & 0.79 &  0.91 & 0.71 \\
\bottomrule
\end{tabular}}
\end{table}

\section{Simpler confidence sets for detection algorithms controlling PFA}
\label{sec:pfa}
In this section, we shall discuss how one can construct simpler confidence sets when the change detection algorithm $\mathcal{A}$ controls the probability of false alarm (PFA) at level $\delta\in(0,1)$.
Importantly, this method allows us to avoid \Cref{assmp-1}, and our results hold in the fully nonparametric case with arbitrary pre- and post-change classes, $\mathcal{P}_0$ and $\mathcal{P}_1$, respectively.

We consider the test statistic $M_t$ introduced in \eqref{eq:test-stat-nonpara}, and define the following confidence set for $T$:
\begin{equation}
\label{eq:ci-nonpara-pfa}
            \mathcal{C}=\left\{t\in\N: t\leq\tau, M_t<\frac{2}{\alpha}\right\},
        \end{equation}
        which is a simplified version of \Cref{eq:ci-nonpara}, obtained by directly replacing $\alpha r_t^*$ with $\alpha$. This avoids the need to compute $r_t^*$ (and hence $\tau_1^*,\cdots,\tau_N^*$), making the method easier to implement and free of \Cref{assmp-1}.

We now investigate the coverage properties of the above confidence set.

\begin{theorem}
\label{thm:coverage-comp-post-PFA} 
   Let $\mathcal{C}$ be the confidence set as defined in \eqref{eq:ci-nonpara-pfa}, for some $\alpha\in(0,1)$. If $\mathcal{A}$ controls PFA at level $\delta\in(0,1)$, then for any $F_0\in\mathcal{P}_0$ and $F_1\in\mathcal{P}_1$,
    \begin{equation*}
      \P_{F_0,T,F_1}(T\in \mathcal{C})\geq 1-\alpha-\delta ~\text{, and } \P_{F_0,T,F_1}(T\in \mathcal{C}\mid \tau\geq T)\geq 1-\frac{\alpha}{1-\delta}, 
    \end{equation*} 
    where the latter guarantee holds if $\mathbb \P_{F_0,\infty}(\tau\geq T)\neq 0,$ for all $F_0\in\mathcal{P}_0.$
\end{theorem}
\begin{proof}
From \Cref{lem:univ-threshold}, we have $\P_{F_0,t,F_1}[M_t\geq\frac{2}{\alpha }]\leq \alpha$. Hence, 
\begin{align*}
  &\P_{F_0,T,F_1}( M_T< 2/\alpha)\\
   &=\P_{F_0,T,F_1}(T\in \{t\in\N: M_t< 2/\alpha\})\\
   &\leq\P_{F_0,T,F_1}(T\in \{t\in\N: M_t< 2/\alpha \},\tau\geq T) +\mathbb \P_{F_0,T,F_1}(\tau< T)\\
    &= \P_{F_0,T,F_1}(T\in \{t\in\N: t\leq \tau, M_t< 2/\alpha\})+\mathbb \P_{F_0,\infty}(\tau< T)\\
    &\leq \P_{F_0,T,F_1}(T\in \mathcal{C})+\delta,
\end{align*}
where the last inequality follows from the fact that the PFA is less than or equal to $\delta$. Thus, we obtain
\begin{equation*}
      \P_{F_0,T,F_1}(T\in \mathcal{C})\geq 1-\alpha-\delta.
    \end{equation*} 
And
\begin{equation*}
     \P_{F_0,T,F_1}(T\in \mathcal{C}\mid \tau\geq T)=\frac{\P_{F_0,T,F_1}(T\in \mathcal{C})}{\P_{F_0,T,F_1}(\tau\geq T)}=\frac{\P_{F_0,T,F_1}(T\in \mathcal{C})}{\P_{F_0,\infty}(\tau\geq T)}\geq 1-\frac{\alpha}{1-\delta}, 
    \end{equation*} 
    completing the proof.
\end{proof}
A similar simplification can be applied to the adaptive confidence sets introduced in Section 4 for the parametric setting. By replacing the $\alpha r_t^*$ term with $\alpha$ in \Cref{eq:ci-simple,eq:conf-comp-post,eq:conf-comp}, we obtain
\begin{equation}
\label{eq:conf-simp-pfa}
\mathcal{C}_1=\left\{t\in \N:t\leq \tau, M_t\leq\operatorname{Quantile}(1-{\alpha};M_t,\{M_{t,L}^j\}_{j=1}^B)\right\},
\end{equation}
for known pre- and post-change setting. Similarly, for the known pre-change and composite post-change setting,
\begin{equation}
\label{eq:conf-comp-post-pfa}
\mathcal{C}_1=\left\{t\in \N:t\leq \tau, M_t\leq\sup_{\theta\in \mathcal{S}_{\tau-t+1}}\operatorname{Quantile}(1-{\alpha};M_t,\{M_{t,L}^j(\theta)\}_{j=1}^B)\right\},
\end{equation}
 And finally,
\begin{equation}
\label{eq:conf-comp-pfa}
\mathcal{C}_1=\left\{t\in \N:t\leq \tau, M_t\leq\sup_{(\theta, \theta^\prime)\in \mathcal{\mathbf S}^{(t)}}\operatorname{Quantile}(1-{\alpha};M_t,\{M_{t,L}^j(\theta,\theta^\prime)\}_{j=1}^B)\right\}
\end{equation}
for composite pre- and post-change setting.

For detection algorithms with PFA $\leq \delta,~\delta\in(0,1)$, these simpler forms retain the coverage guarantees --- both conditional and unconditional --- similar to what we have shown in the theorem above.

In the next subsection, we present the experimental results with the bounded PFA algorithms, along with some other experiments.

\subsection{Experiments with detection algorithms that control PFA}
\label{sec:expt-pfa}
\subsubsection{Known pre- and post-change (Setting A, PFA \texorpdfstring{$\leq 0.001$}{Lg})}
We perform experiments for both Gaussian and Poisson mean-change scenarios having the true changepoint at $T=100,500$, using a likelihood ratio detector defined 
as
\begin{equation}
\label{eq:LRT}
    \tau=\inf\left\{t\in\N: \prod_{i=1}^t\frac{f_1(X_i)}{f_0(X_i)}\geq A\right\},
\end{equation}
with thresholds at $A=1000$. Note that $\prod_{i=1}^t\frac{f_1(X_i)}{f_0(X_i)}$ is a non-negative martingale under $T=\infty$ and hence, by Ville's inequality \citep{ville1939etude}, the PFA of this detector is at most $1/A$, i.e., $0.001$ here.
After detection, we construct both the proposed confidence sets (universal  \eqref{eq:ci-nonpara-pfa} and adaptive \eqref{eq:conf-simp-pfa}), with $B=100$, $\alpha=0.1$. We report the average size and the conditional and unconditional coverage rates across $500$ independent runs in \Cref{tab:simple-pfa}. 
We also report the average delay (given no false detection) in the table.

 \begin{table}[!ht]
    \centering
    \caption{Known pre- and post-change distributions (Setting A, PFA $\leq 0.001$) 
    }
\label{tab:simple-pfa}
    \resizebox{\linewidth}{!}{
    \begin{tabular}{ccc|cccc|c}
    \toprule
    \addlinespace
$T$ & $F_0$ & $F_1$   &\multicolumn{2}{c}{Conditional Coverage} &
\multicolumn{2}{c}{Conditional Size}  & \specialcell{Delay} \\
\cmidrule(lr){4-5} \cmidrule(lr){6-7}  & &
&  Universal &  Adaptive &   Universal &  Adaptive &  (Conditional) \\
\midrule
\addlinespace 
100 & $N(0,1)$ & $N(1,1)$  &  0.95 &  0.94 & 12.65 & 9.86   & 110.91 \\
100 &  $\text{Pois}(1)$ & $\text{Pois}(2)$  &  0.95   & 0.95 &  14.22 & 11.85  & 94.33\\
500 &  $N(0,1)$ & $N(1,1)$  & 0.97  & 0.97 & 12.83 & 9.76 &  514.12\\
500 &  $\text{Pois}(1)$ & $\text{Pois}(2)$  &  0.96  & 0.96 & 14.79 & 12.13   & 417.44\\

\bottomrule
 \end{tabular}}
 \end{table}
\subsubsection{Known pre- and unknown post-change (Setting B, PFA \texorpdfstring{$\leq 0.001$}{Lg})}
We perform experiments for Gaussian mean-change scenarios having the true changepoint at $T=100,500$, using a mixture likelihood ratio detector defined 
as
\begin{equation}
\label{eq:mix-LRT}
    \tau=\inf\left\{t\in\mathbb N:\int_{\theta\in\Theta_1}\prod_{i=1}^t\frac{f_{\theta}(X_i)}{f_{\theta_0}(X_i)}dW(\theta)\geq A\right\}
\end{equation}
with thresholds at $A=1000$. The weights are chosen just as described in \Cref{setting-ii}. Similarly, the PFA of this detector is at most $1/A$, i.e., $0.001$ here.
After detection, we construct our universal confidence sets \eqref{eq:ci-nonpara-pfa} for $T$, with $\alpha=0.1$. The test statistic is the same as mentioned in Setting II of \Cref{sec:expt}. We report the average size and the conditional and unconditional coverage rates across $500$ independent runs in \Cref{tab:comp-post-pfa}.
We also report the average delay (given no false detection) in the table.

\begin{table}[!ht]
    \centering
    \caption{Known pre-change but unknown post-change (Setting B, PFA $\leq 0.001$)
    }
\label{tab:comp-post-pfa}
    \resizebox{\linewidth}{!}{
    \begin{tabular}{cc|cccc}
    \toprule
    \addlinespace
 T    & \specialcell{Composite \\post-change}  & \specialcell{Conditional\\coverage} & \specialcell{Unconditional\\coverage} & \specialcell{Size\\(Conditional)} & \specialcell{Delay\\(Conditional)} \\
    \midrule
\addlinespace 

100 &  $\{N(\mu,1):\mu> 0.75\}$   & 0.96 & 0.95 & 16.38  & 118.76 \\
100 &  $\{N(\mu,1):\mu> 0.9\}$   & 0.96 & 0.96 & 13.11  & 115.64\\
500 &  $\{N(\mu,1):\mu> 0.75\}$ & 0.97 & 0.97 & 16.92 & 530.83\\
500 &  $\{N(\mu,1):\mu> 0.9\}$ & 0.96 &0.96 & 14.85 & 528.92\\
\bottomrule
 \end{tabular}}
 \end{table}

 \subsubsection{Composite pre- and post-change (Setting C)}
We generate pre-change data from $N(0,1)$ and post-change data from $N(0,1)$, having a true changepoint at $T=100$ or $500$ --- these are all unknown to the algorithm. We assume access to a historic/training data $\{Y_1,\cdots,Y_m\}$, $m=200$. We detect change using the following change detector, as in Sections 2 and 3 of \cite{aue2024state}:
\begin{equation}
    \tau = \min \left\{ k \in \mathbb{N} : \frac{1}{\hat{\sigma}_m \sqrt{m}} w(m, k) \left| \sum_{j=1}^k (X_{j} - \bar{Y}_m) \right| > A \right\},
\end{equation}
where $\bar{Y}_m=\frac{1}{m} \sum_{i=1}^m Y_i$ is the sample mean of the training data, $\hat{\sigma}_m = \sqrt{\frac{1}{m-1} \sum_{i=1}^m (Y_i - \bar{Y}_m)^2}$ is the estimated standard deviation from the training data,  $w(m, k) = \left( 1 + \frac{k}{m} \right)^{-1}$ and $A$ is the threshold, which we vary in our experiment.
After detection, we construct our universal confidence sets \eqref{eq:ci-nonpara-pfa} for $T$, with $\alpha=0.1$, $\mathcal{P}_0=\{N(\mu,1):\mu\leq 0.25\}$ and $\mathcal{P}_1=\{N(\mu,1):\mu> 0.75\}$. The test statistic is the same as mentioned in Setting III of \Cref{sec:expt}.  We report the average size and coverage rates across $500$ independent runs in \Cref{tab:comp-pfa}. 
We also report the average delay (given no false detection). 

 \begin{table}[!ht]
    \centering
    \caption{Composite pre- and post-change (Setting C)
    }
\label{tab:comp-pfa}
    \resizebox{0.72\linewidth}{!}{
    \begin{tabular}{cc|cccc}
    \toprule
    \addlinespace
 T & A & \specialcell{Conditional\\coverage} & \specialcell{Marginal\\coverage} & \specialcell{Size\\(Conditional)} &  \specialcell{Delay\\(Conditional)}  \\
    \midrule
\addlinespace 
100 & 2  & 0.97 & 0.97 & 23.54 &50.64\\
100 &  3 & 0.97 & 0.97 & 21.82 & 82.95\\
500 &  2 & 0.97 & 0.97 & 21.37 & 122.84\\
500 & 3  & 0.98 &0.98 & 23.22  &  193.59\\
\bottomrule
\end{tabular}}
\end{table}

\subsection{Comparison of experimental results with ARL and PFA controlling detection algorithms}
As expected, algorithms that control PFA exhibit a significantly higher average detection delay compared to those controlling ARL. This delay arises because stricter false alarm constraints force the detection algorithm to accumulate stronger evidence before signaling a change.
In terms of inference, we observe that PFA-controlling algorithms often tend to yield slightly lower conditional coverage probabilities with confidence sets of comparable sizes. Also, conditional and unconditional coverage rates are identical when PFA-controlling detection algorithms are employed, due to the very low probability of false detections, while the ARL controlling algorithms lead to lower unconditional conditional coverage rates for larger $T$, ie.e., $T=500$, supporting our theory. 
Overall, the major differences between these two types lie in detection delay and unconditional coverage properties. We do not observe significant differences in the size and conditional coverage properties of the confidence sets, nor in the absolute deviation of the point estimates, across these two types of detection algorithms.

\section{Some additional experimental results}
\label{sec:add-expt}
\subsection{Experiments varying $T$ and ARL}
\label{sec:add-expt-ratio}
We have seen in the derivation of the asymptotic length bound that it increases with the ratio $T/A_T$. Now, we experimentally validate the fact by conducting the same experiment as in Setting I (known pre- and post-change) in \Cref{sec:expt} with varying $T$ thresholds $A$ of the CUSUM detector (larger threshold means large ARL). Average (of $500$ independent runs) size and coverage of the confidence sets using \Cref{eq:ci-nonpara-known-pre} are reported in \Cref{tab:varying-ratio}. The results show a clear pattern: as the ratio 
$T/A$ increases, the resulting confidence sets become wider, reflecting growing conservativeness.

 \begin{table}[!ht]
    \centering
    \caption{Setting I: Known pre- and post-change.
     Average (of $500$ independent runs) size and coverage of the confidence sets using \Cref{eq:ci-nonpara-known-pre} with $F_0=N(0,1), F_1=N(1,1), N=100$, $\alpha=0.1$, the average detection delay (given a true detection) with threshold $A$, while varying $T$ and $A$.
    }
\label{tab:varying-ratio}
    \resizebox{0.67\linewidth}{!}{
    \begin{tabular}{cc|cccc}
    \toprule
    \addlinespace
$T$ & A     & \specialcell{Conditional\\coverage} & \specialcell{Marginal\\coverage} & Size & \specialcell{Conditional\\Delay}  \\
\midrule
\addlinespace 
50 & 100  & 0.978 & 0.91 &  14.382 &  9.33\\
1000 & 100  & 0.989 & 0.21 & 18.283& 9.14\\
2000 & 100  & 0.992 & 0.07 & 19.627&  8.35\\
500 & 1000  & 0.925 & 0.97 & 12.618&  13.22\\
10000 & 1000  & 0.947 & 0.27 & 18.254&  13.25\\
20000 & 1000 & 1 & 0.07 & 22.113 & 13.22\\
\bottomrule
 \end{tabular}}
 \end{table}

\subsection{Experiments with dependent pre- and post-change data}
For the sake of illustration, we now demonstrate our methodology works with dependent data. Specifically, we consider the pre-change and post-change distributions to be Markov$(p_0, q_0)$ and Markov$(p_1, q_1)$, respectively, with parameters $p_0 = 0.75$, $p_1 = 0.25$, and $q_0 = q_1 = 0.5$. The first data point is drawn from $\text{Bern}(0.5)$. The changepoint $T$ is set to either 100 or 500.
(Here, Markov$(p, q)$ denotes the two-state Markov chain where $p$ is the probability of transitioning from state $0$ to state $1$, and $q$ is the probability of transitioning from state $1$ to state $1$.)

We use the likelihood ratio-based CUSUM detector:
\begin{equation}
    \tau=\inf\left\{n\in\N: \max_{1\leq j \leq n} \left(\frac{p_1}{p_0}\right)^{n_{10}(j:n)}\left(\frac{1-p_1}{1-p_0}\right)^{n_{11}(j:n)}\geq A\right\},
\end{equation}
where $n_{11}(j:n)$ denotes the number of transitions from 1 to 1 and $n_{10}(j:n)$ denotes the number of transitions from 0 to 1 within $X_j,\cdots,X_n$, and  $A=1000$.

After detection, we compute the adaptive confidence sets \eqref{eq:ci-simple} with the test statistic in
\eqref{eq:test-stat-nonpara} with $R^{(t)}_{n}=\frac{1}{2}\left(\frac{p_0}{p_1}\right)^{n_{10}(t:n)}\left(\frac{1-p_0}{1-p_1}\right)^{n_{11}(t:n)}$, $S^{(t)}_{n}=\frac{1}{2}\left(\frac{p_1}{p_0}\right)^{n_{10}(n:t-1)}\left(\frac{1-p_1}{1-p_0}\right)^{n_{11}(n:t-1)}$, $B=N=100$ and $\alpha=0.1$. Note that $\{S^{(t)}_{n}\}_{n< t}$ is not a valid e-process in the backward filtration. In contrast to our universal method, the adaptive procedure in \Cref{algo:1-exact} (or \Cref{eq:ci-simple}) accommodates arbitrary test statistics and does not rely on e-process structure; consequently, it remains applicable in settings with dependent data.
\Cref{tab:markov} contains the results averaged across $100$ random runs.
\begin{table}[!ht]
    \centering
    \caption{Known pre- and post-change with dependency}
\label{tab:markov}
    \resizebox{0.62\linewidth}{!}{
    \begin{tabular}{c|ccc}
    \toprule
    \addlinespace
 T    & \specialcell{Conditional coverage}  & Size  & \specialcell{Conditional delay}  \\
    \midrule
\addlinespace 
100 &    0.918 &  23.245 &  31.929\\
500 &   0.907 &  23.698 &  34.625\\
\bottomrule
\end{tabular}}
\end{table}

\begin{remark}
We emphasize again that the universal approach described in Section 3 cannot handle dependent data, but the adaptive method in Section 4 can. To elaborate, our test statistic in Section 3 requires constructing non-trivial e-processes in both forward (usual) and backward (nonstandard) filtrations; in settings with general dependence,  constructing the latter is generally infeasible.
\end{remark}

\end{document}